\documentclass{article}

\usepackage{microtype}
\usepackage{graphicx}
\usepackage{booktabs} 

\usepackage{hyperref}

\usepackage[accepted]{icml2025}

\usepackage[utf8]{inputenc} 
\usepackage[T1]{fontenc}    
\usepackage{url}            
\usepackage{amsfonts}       
\usepackage{nicefrac}       
\usepackage{colortbl}
\usepackage{xcolor}         
\usepackage{tikz}
\usetikzlibrary{shapes, arrows, positioning}

\usepackage{amsmath,amssymb,amsfonts}
\usepackage{amsthm}
\usepackage{mathtools}
\mathtoolsset{showonlyrefs}
\usepackage{bm}
\usepackage{subcaption}

\newcommand{\tabitem}{\textbullet~~}

\usepackage{siunitx} 
\newcommand{\best}{\cellcolor{green!15!white}\bf }

\newcommand{\R}{\mathbb{R}}

\newcommand{\E}{\mathbb{E}}
\newcommand{\Z}{\mathbb{Z}}
\renewcommand{\d}{\mathrm{d}}
\renewcommand{\P}{\mathcal{P}}

\newcommand{\T}{\mathrm{T}}

\newcommand{\F}{\mathcal F}

\newcommand{\zb}{\bm}

\mathtoolsset{showonlyrefs}

\DeclareMathOperator*{\argmin}{arg\,min}

\DeclareMathOperator{\dom}{dom}

\DeclareMathOperator{\prox}{prox}

\newcommand{\tT}{\mathrm{T}}
\newcommand{\opt}{\mathrm{opt}}

\theoremstyle{plain}
\newtheorem{theorem}{Theorem}[section]

\newtheorem{lemma}[theorem]{Lemma}
\newtheorem{corollary}[theorem]{Corollary}
\newtheorem{example}[theorem]{Example}
\theoremstyle{definition}
\newtheorem{definition}[theorem]{Definition}
\newtheorem{assumption}[theorem]{Assumption}
\theoremstyle{remark}
\newtheorem{remark}[theorem]{Remark}

\icmltitlerunning{Importance Corrected Neural JKO Sampling}

\begin{document}

\twocolumn[
\icmltitle{Importance Corrected Neural JKO Sampling}




\begin{icmlauthorlist}
\icmlauthor{Johannes Hertrich}{dauphine,ucl}
\icmlauthor{Robert Gruhlke}{fub}
\end{icmlauthorlist}

\icmlaffiliation{dauphine}{Université Paris Dauphine - PSL}
\icmlaffiliation{ucl}{University College London}
\icmlaffiliation{fub}{FU Berlin}

\icmlcorrespondingauthor{Johannes Hertrich}{johannes.hertrich@dauphine.psl.eu}
\icmlcorrespondingauthor{Robert Gruhlke}{r.gruhlke@fu-berlin.de}

\icmlkeywords{Sampling, Wasserstein Gradient Flows, Normalizing Flows, Rejection Sampling}

\vskip 0.3in
]



\printAffiliationsAndNotice{\icmlEqualContribution} 

\begin{abstract}
    In order to sample from an unnormalized probability density function, we propose to combine continuous normalizing flows (CNFs) with rejection-resampling steps based on importance weights. We relate the iterative training of CNFs with regularized velocity fields to a JKO scheme and prove convergence of the involved velocity fields to the velocity field of the Wasserstein gradient flow (WGF).
    The alternation of local flow steps and non-local rejection-resampling steps allows to overcome local minima or slow convergence of the WGF for multimodal distributions.
    Since the proposal of the rejection step is generated by the model itself, 
    they do not suffer from common drawbacks of classical rejection schemes.
    The arising model can be trained iteratively, reduces the reverse Kullback-Leibler (KL) loss function in each step, allows to generate \textit{iid} samples and moreover allows for evaluations of the generated underlying density.
    Numerical examples show that our method yields accurate results on various test distributions including high-dimensional multimodal targets and outperforms the state of the art in almost all cases significantly.
\end{abstract}
\section{Introduction}

We consider the problem of sampling from an unnormalized probability density function. That is, we are given an integrable function $g\colon\R^d\to\R_{>0}$ and we aim to generate samples from the probability distribution $\nu$ given by the density $q(x)=g(x)/Z_g$, where the normalizing constant $Z_g=\int_{\R^d} g(x)\d x$ is unknown.
Many classical sampling methods are based on Markov chain Monte Carlo (MCMC) methods like the overdamped Langevin sampling, see, e.g., \cite{WT2011}. The generated probability path of the underlying stochastic differential equation follows the Wasserstein-$2$ gradient flow of the reverse KL divergence $\mathcal F(\mu)=\mathrm{KL}(\mu,\nu)$.
Over the last years, generative models like normalizing flows \citep{RM2015} or diffusion models \citep{HJA2020,SSKKEP2021} became more popular for sampling, see, e.g., \cite{PDHD2024,VGD2023}. Also these methods are based on the reverse KL divergence as a loss function. 
While generative models have successfully been applied in data-driven setups, their application to the problem of sampling from arbitrary unnormalized densities is not straightforward. This difficulty arises from the significantly harder nature of the problem, even in moderate dimensions, particularly when dealing with target distributions that exhibit phenomena such as concentration effects, multimodalities, heavy tails, or other issues related to the curse of dimensionality.

In particular, the reverse KL is non-convex in the Wasserstein space as soon as the target density $\nu$ is not log-concave which is for example the case when $\nu$ consists of multiple modes. 
In this case generative models often collapse to one or a small number of modes. We observe that for continuous normalizing flows (CNFs, \citealp{CRBD2018, GCBSD2018}) this can be prevented by regularizing the $L^2$-norm of the velocity field as proposed under the name OT-flow by \cite{OFLR2021}. In particular, this regularization converts the objective functional into a convex one.
However, the minimizer of the regularized loss function is no longer given by the target measure $\nu$ but by the Wasserstein proximal mapping of the objective function applied onto the latent distribution. Considering that the Jordan-Kinderlehrer-Otto (JKO) scheme \citep{JKO1998} iteratively applies the Wasserstein proximal mapping and converges to the Wasserstein gradient flow, several papers proposed to approximate the steps of the scheme by generative models, see \cite{AHS2023,ASM2022,FZTC2022,LCBBR2022,MKLGS2021,VWTON2023,XCX2024}. We will refer to this class of method by the name \textit{neural JKO}. Even though this approximates the same gradient flow of the Langevin dynamics these approaches have usually the advantage of faster inference (once they are trained) and additional possibly allow for density evaluations.

However, already the time-continuous Wasserstein gradient flow suffers from the non-convexity of the reverse KL loss function by getting stuck in local minima or by very slow convergence times.
In particular, it is well-known that Langevin-based sampling methods often do not distribute the mass correctly onto multimodal target distributions.
As a remedy, \cite{N2001} proposed importance sampling, i.e., to reweight the sample based on quotients of the target distribution and its current approximation which leads to an unbiased estimator of the Monte Carlo integral. However, this estimator may lead to highly imbalanced weights between these samples and the resulting estimator might have a large variance. Moreover, the density of the current approximation has to be known up to a possible multiplicative constant, which is not the case for many MCMC methods like Langevin sampling or Hamiltonian Monte Carlo. In \cite{DDJ2006} the authors propose a scheme for alternating importance sampling steps with local Monte Carlo steps.
However, the corresponding importance sampling step generates non-iid samples, that de-correlate over time by construction.

\begin{figure*}
\begin{subfigure}[c]{.18\textwidth}
    \centering
    {\tiny Latent distribution}
    \includegraphics[width=\textwidth]{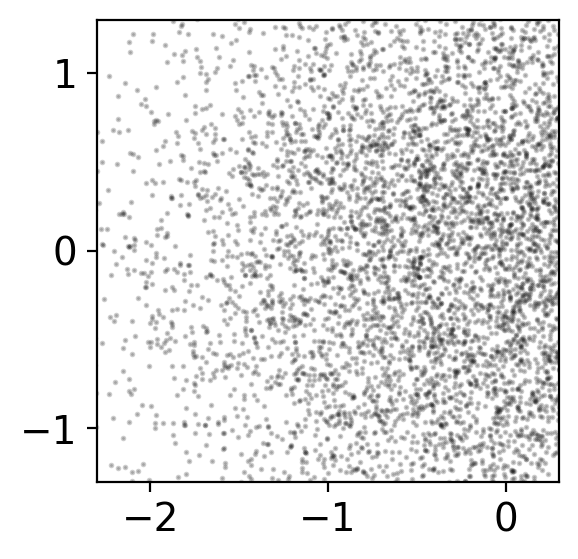}
\end{subfigure}
\begin{subfigure}[t]{.06\textwidth}
\centering
\begin{tikzpicture}[thick,scale=0.5, every node/.style={scale=0.5}]
    \draw [-stealth](0,0) -- node [text width=0.8cm,midway,above]{CNF\\[-.5em] \phantom{.}} (1,0);
\end{tikzpicture}
\end{subfigure}
\begin{subfigure}[c]{.18\textwidth}
    \centering
    {\tiny After Step 1}
    \includegraphics[width=\textwidth]{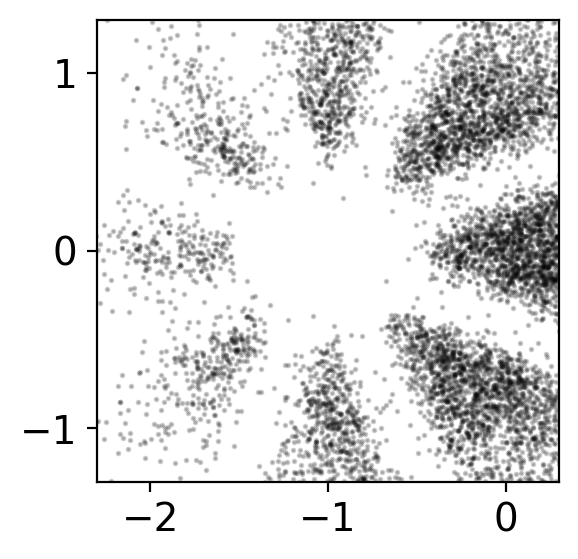}
\end{subfigure}
\begin{subfigure}[t]{.06\textwidth}
\centering
\begin{tikzpicture}[thick,scale=0.5, every node/.style={scale=0.5}]
    \draw [-stealth](0,0) -- node [text width=0.8cm,midway,above]{CNF\\[-.5em] \phantom{.}} (1,0);
\end{tikzpicture}
\end{subfigure}
\begin{subfigure}[c]{.18\textwidth}
    \centering
    {\tiny After Step 2}
    \includegraphics[width=\textwidth]{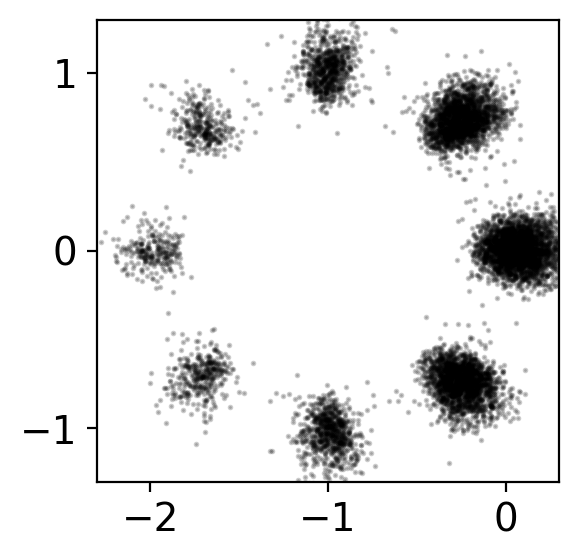}
\end{subfigure}
\begin{subfigure}[t]{.06\textwidth}
\centering
\begin{tikzpicture}[thick,scale=0.5, every node/.style={scale=0.5}]
    \draw [-stealth](0,0) -- node [text width=0.8cm,midway,above]{CNF\\[-.5em] \phantom{.}} (1,0);
\end{tikzpicture}
\end{subfigure}
\begin{subfigure}[c]{.18\textwidth}
    \centering
    {\tiny After Step 3}
    \includegraphics[width=\textwidth]{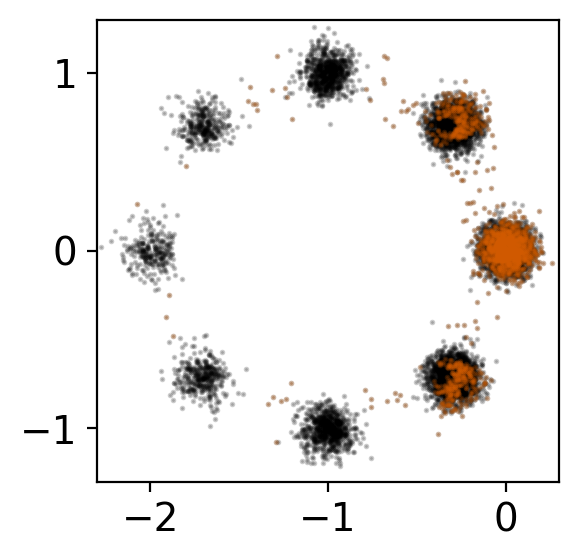}
\end{subfigure}
\begin{subfigure}[t]{.06\textwidth}
\centering
\begin{tikzpicture}[thick,scale=0.5, every node/.style={scale=0.5}]
    \draw [-stealth](0,0) -- node [text width=1.4cm,midway,above]{\centering \textcolor[RGB]{213, 94, 0}{rejection} \\ $+$ \\ \textcolor[RGB]{70, 154, 208}{resampling}\\[-.5em] \phantom{.}} (1,0);
\end{tikzpicture}
\end{subfigure}

\begin{subfigure}[c]{.18\textwidth}
    \centering
    {\tiny After Step 4}
    \includegraphics[width=\textwidth]{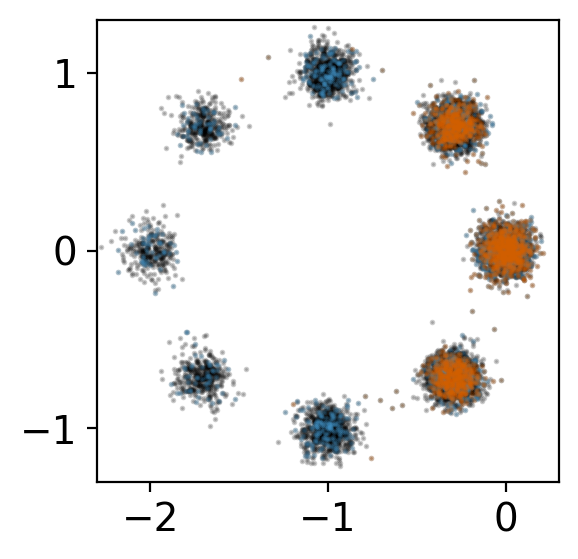}
\end{subfigure}
\begin{subfigure}[t]{.06\textwidth}
\centering
\begin{tikzpicture}[thick,scale=0.5, every node/.style={scale=0.5}]
    \draw [-stealth](0,0) -- node [text width=1.2cm,midway,above]{\centering \textcolor[RGB]{213, 94, 0}{rejection} \\ $+$ \\ \textcolor[RGB]{70, 154, 208}{resampling}\\[-.5em] \phantom{.}} (1,0);
\end{tikzpicture}
\end{subfigure}
\begin{subfigure}[c]{.18\textwidth}
    \centering
    {\tiny After Step 5}
    \includegraphics[width=\textwidth]{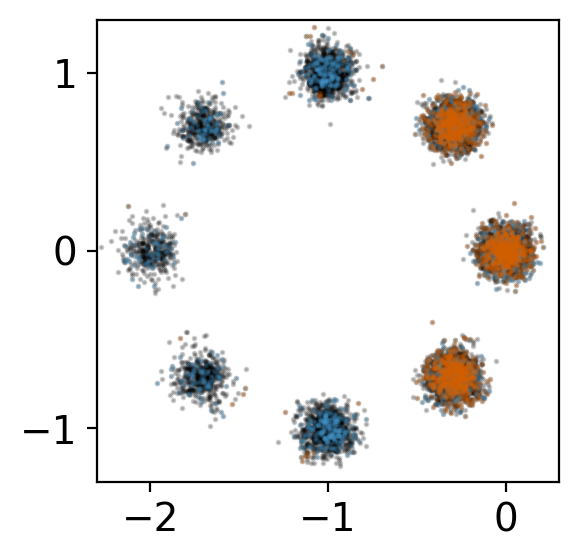}
\end{subfigure}
\begin{subfigure}[t]{.06\textwidth}
\centering
\begin{tikzpicture}[thick,scale=0.5, every node/.style={scale=0.5}]
    \draw [-stealth](0,0) -- node [text width=1.2cm,midway,above]{\centering \textcolor[RGB]{213, 94, 0}{rejection} \\ $+$ \\ \textcolor[RGB]{70, 154, 208}{resampling}\\[-.5em] \phantom{.}} (1,0);
\end{tikzpicture}
\end{subfigure}
\begin{subfigure}[c]{.18\textwidth}
    \centering
    {\tiny After Step 6}
    \includegraphics[width=\textwidth]{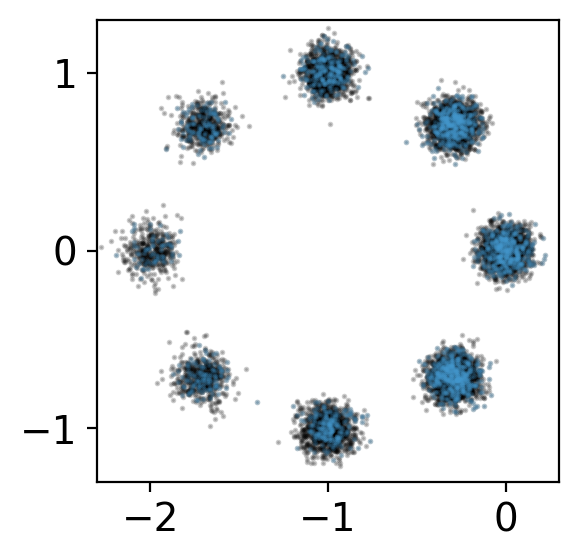}
\end{subfigure}
\begin{subfigure}[t]{.06\textwidth}
\centering
\begin{tikzpicture}[thick,scale=0.5, every node/.style={scale=0.5}]
    \draw [-stealth](0,0) -- node [text width=.6cm,midway,above]{$\cdots$ \\ \phantom{.}} (1,0);
\end{tikzpicture}
\end{subfigure}
\begin{subfigure}[c]{.18\textwidth}
    \centering
    {\tiny Final Approximation}
    \includegraphics[width=\textwidth]{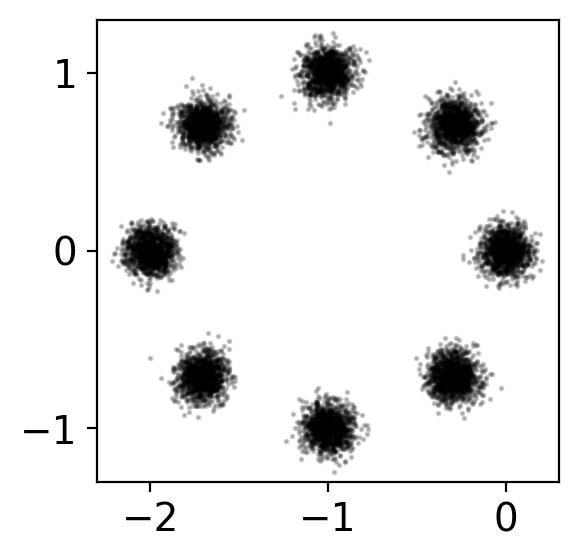}
\end{subfigure}
\begin{subfigure}[t]{.06\textwidth}
\phantom{\includegraphics[width=\textwidth]{figs_final/intro/intro_0.png}}
\end{subfigure}
\caption{Iterative application of neural JKO steps and rejection steps for a shifted mixture target distribution. The orange samples are rejected in the next following rejection step and the blue samples are the resampled points. The latter approach enables for the correction of wrong mode weights introduced by the underlying WGF. See Figure~\ref{fig:8modes_2d} for more steps.
}
\label{fig:death-birth}
\end{figure*}

\noindent
\textbf{Contributions}\,\,\,
In this paper, we propose a sampling method which combines neural JKO steps based on CNFs with importance based rejection steps.
While the CNFs adjust the position of the generated samples \textit{locally}, the rejection steps readjust the inferred distribution \textit{non-locally} based on the quotient of generated and target distribution. Then, in each rejection step we resample the rejected points based on the current constructed generative model.
We illustrate this procedure in Figure~\ref{fig:death-birth}.

Our method generates independent samples and allows to evaluate the density of the generated distribution. 
In our numerical examples, we apply our method to common test distributions up to the dimension $d=1600$ which are partially highly multimodal. We show that our \textit{importance corrected neural JKO} sampling (neural JKO IC) achieves significantly better results than the comparisons\footnote{The code is available at \url{https://github.com/johertrich/neural_JKO_ic}}.

From a theoretical side, we prove that the velocity fields from a sequence of neural JKO steps strongly converge to the velocity field of the corresponding Wasserstein gradient flow and that the reverse KL loss function decreases throughout the importance-based rejection steps.

\noindent
\textbf{Outline}\,\,\,
In Section~\ref{sec:prelim}, we recall the fundamental concepts which will be required. Afterwards, we consider neural JKO schemes more detail in Section~\ref{sec:neural_JKO}.
We introduce our importance-based rejection steps in Section~\ref{sec:rejection_steps}. Finally, we evaluate our model numerically and compare it to existing methods in Section~\ref{sec:numerics}. Conclusions are drawn in Section~\ref{sec:conclusions}. Additionally, proofs, further numerical and technical details are presented in Appendix~\ref{app:backgrounds}-~\ref{app:detailed_experiments}.

\noindent
\textbf{Related Work}\,\,\,
Common methods for sampling of unnormalized densities are often based on Markov Chain Monte Carlo (MCMC) methods, see e.g. \citep{GRS1995}. In particular first order based variants, such as the Hamiltonian Monte Carlo (HMC) 
\citep{B2017,HG2014}
and the Metropolis Adjusted Langevin Algorithm (MALA \citealp{GC2011,RDF1978,RT1996}) are heavily used in practice, see \citep{ADDJ2003} for an overview. 
The viewpoint of these samplers 
as sample space description of gradient flows defined in a metricized  probability space then allows for extensions such as interacting particle systems  \citep{CHHRS2023, EGS2024,GNR2020,WL2022}.
However, since these algorithms are based on local transformations of the samples, they are unable to distribute the mass correctly among different modes, which can partially be corrected by importance sampling \citep{N2001} and sequential Monte Carlo samplers (SMC) \citep{DDJ2006} as described above. 
In contrast to our model, SMC approximates the density of the approximation by assigning ``inverse Markov kernels'' to certain MCMC kernels, which might lead to propagating errors. Furthermore, the generation of additional samples requires to rerun the whole procedure which can be very costly. Other approaches approximate the target density by combination of transport maps and low-rank models such as tensor trains \cite{gruhlke2022low} yielding efficient access to posterior statistics. 

In the last years, generative models became very popular, including VAEs \citep{KW2013}, normalizing flows \citep{RM2015}, diffusion models \citep{HJA2020} or flow-matching \citep{LCBNL2022} which is also known as rectified flow \citep{LGL2022}. In contrast to our setting, they initially consider the \emph{modeling} task, i.e., they assume that they are given samples from the target measure instead of an unnormalized density. However, there are several papers, which adapt these algorithms for the \textit{sampling} task. For normalizing flows, this mostly amounts to changing the loss function \citep{HSDL2019,MMPS2016,QW2024}. 
Very recently, there appeared also a flow-matching variant for the sampling task \citep{WA2024} based on \cite{ARJM2024}.
For diffusion (and stochastic control) models this was done based on variational approaches \citep{BJEVN2024,PDHD2024,VGD2023,VN2023,ZC2021} or by computing the score by solving a PDE \citep{RB2023,GS2024}. These methods usually provide much faster sampling times than MCMC methods and are often used in combination with some conditioning parameter for inverse problems, where a (generative) prior is combined with a known likelihood term \citep{AKWR2019,AH2023,AFHHSS2021,10Autoren}.
Combinations of generative models with stochastic sampling steps were considered in the literature for generative modeling under the name stochastic normalizing flows \citep{HHS2023,HHS2022,NOKW2019,WKN2020} and for sampling under the name annealed flow transport Monte Carlo \citep{AMD2021,MARD2022}. \cite{GRV2022} use normalizing flows to learn proposal distributions in an Metropolis-Hastings algorithm.
These generative models can be adapted to follow a Wasserstein gradient flow, by mimicking a JKO scheme with generative models or directly following the velocity field of a kernel-based functional. Such approaches were proposed for generative modeling \citep{FZTC2022,HHABCS2024,HWAH2024,LSMDS2019,VWTON2023,XCX2024}, sampling \citep{FZTC2022,LCBBR2022,L2016,MKLGS2021} or other tasks \citep{AHS2023,AKSG2019,ASM2022}.
Gradient flows of the reverse KL divergence with respect to different metric were considered in \cite{LW2016} under the name Stein variational gradient descent. Moreover, \citet{LLN2019} study gradient flows in the Wasserstein-Fisher-Rao metric which can be implemented via birth-death processes.

After the first version of our paper appeared, several related preprints were released, including \cite{AV2024,CRBBNA2025,HDVZ2025,WX2024}.

\section{Preliminaries}\label{sec:prelim}

In this section, we provide a rough overview of the required concepts for this paper. To this end, we first revisit the basic definitions of Wasserstein gradient flows, e.g., based on \cite{AGS2005}. Afterwards we recall continuous normalizing flows with OT-regularizations.

\subsection{Curves in Wasserstein Spaces}

\noindent
\textbf{Wasserstein Distance}\,\,\,
Let $\P(\R^d)$ be the space of probability measures on $\R^d$ and denote by $\P_2(\R^d)\coloneqq \{\mu\in\P(\R^d):\int_{\R^d}\|x\|^2\d \mu(x)<\infty\}$ the subspace of probability measures with finite second moment. Let $\P_2^{\mathrm{ac}}(\R^d)$ be the subspace of absolutely continuous measures from $\P_2(\R^d)$.
Moreover, we denote for $\mu,\nu\in\P_2(\R^d)$ by $\Gamma(\mu,\nu)\coloneqq \{\zb \pi\in\P_2(\R^d\times\R^d):{P_1}_\#\zb \pi=\mu, {P_2}_\#\zb \pi=\nu\}$ the set of all transport plans with marginals $\mu$ and $\nu$, where $P_i\colon\R^d\times\R^d\to\R^d$ defined by $P_i(x_1,x_2)=x_i$ is the projection onto the $i$-th component for $i=1,2$.
Then, we equip $\P_2(\R^d)$ with the Wasserstein-2 metric defined by
$$
W_2^2(\mu,\nu)=\inf_{\zb \pi\in\Gamma(\mu,\nu)}\int_{\R^d\times\R^d} \|x-y\|^2\d \zb \pi(x,y).
$$
If $\mu \in \P_2^\mathrm{ac}(\R^d)$, the above problem has always a unique minimizer.

\noindent
\textbf{Absolutely Continuous Curves}\,\,\,
A curve $\gamma\colon I\to\P_2(\R^d)$ 
on the interval $I\subseteq \R$ is called \emph{absolutely continuous} 
if there exists a Borel velocity field $v\colon\R^d\times I\to\R^d$ with 
$\int_I \|v(\cdot,t)\|_{L_2(\gamma(t),\mathbb{R}^d)} \d t<\infty$ 
such that the continuity equation 
\begin{equation}\label{eq:continuity_equation}
\partial_t\gamma(t)+\nabla\cdot(v(\cdot,t)\gamma(t))=0
\end{equation}
is fulfilled on $I\times\R^d$ in a weak sense.
Then, any velocity field $v$ solving the continuity equation~\eqref{eq:continuity_equation} for fixed $\gamma$ characterizes $\gamma$ as
$
\gamma(t)=z(\cdot,t)_\#\gamma(t_0),
$
where $z$ is the solution of the ODE $\dot z(x,t)=v(z(x,t),t)$ with $z(x,t_0)=x$ and $t_0\in I$.
It can be shown that for an absolutely continuous curve there exists a unique solution of minimal norm which is equivalently characterized by the so-called regular tangent space $\T_{\gamma (t)}\P_2(\R^d)$, see Appendix~\ref{app:backgrounds} for details.
An absolutely continuous curve is a geodesic if there exists some $c>0$ such that $W_2(\gamma(s),\gamma(t))=c|s-t|$.

The following theorem formulates a dynamic version of the Wasserstein distance based minimal energy curves in the Wasserstein space.
\begin{theorem}[\citealp{BB2000}]\label{thm:benamou-brenier}
Assume that $\mu,\nu\in\P_2^\mathrm{ac}(\R^d)$. Then, $W_2^2(\mu,\nu)$ is equal to
\begin{align}
\inf_{\substack{v\colon\R^d\times[0,1]\to\R^d,\\\dot z(x,t)=v(z(x,t),t),\\ z(x,0)=x,\, z(\cdot,1)_\#\mu=\nu}}\int_0^1 \int_{\R^d} \|v(z(x,t),t)\|^2\d \mu(x)\d t.
\end{align}
Moreover, there exists a unique minimizing velocity field $v$ and the curve defined by $\gamma(t)=z(\cdot,t)_\#\mu$ with $t\in[0,1]$ and $\dot z(x,t)=v(z(x,t),t)$, $z(x,0)=x$ is a geodesic which fulfills the continuity equation
$
\partial_t\gamma(t)+\nabla\cdot(v(\cdot,t)\gamma(t))=0.
$
\end{theorem}

Let $\tau>0$. Then, by substitution of $t$ by $t/\tau$ and rescaling $v$ in the time variable, this is equal to
\begin{equation}
\label{eq:rescaled_BB}
\inf_{\substack{v\colon\R^d\times[0,\tau]\to\R^d,\\\dot z(x,t)=v(z(x,t),t), \\
z(x,0)=x, \, z(\cdot,
\tau)_\#\mu=\nu}}\tau\int_0^\tau \int_{\R^d} \|v(z(x,t),t)\|^2\d \mu(x)\d t.
\end{equation}

\noindent
\textbf{Wasserstein Gradient Flows}\,\,\,
An absolutely continuous curve $\gamma\colon(0,\infty)\to\P_2(\R^d)$ 
with velocity field $v_t\in \T_{\gamma (t)}\P_2(\R^d)$ is a \emph{Wasserstein gradient flow with respect to} $\F\colon\P_2(\R^d)\to(-\infty,\infty]$ 
if $v_t\in -\partial \F(\gamma(t))$, for a.e.~$t>0$,
where $\partial \F(\mu)$ denotes the reduced Fr\'echet subdiffential at $\mu$, see Appendix~\ref{app:backgrounds} for a definition.

To compute Wasserstein gradient flows numerically, we can use the generalized minimizing movements or Jordan-Kinderlehrer-Otto (JKO) scheme \citep{JKO1998}. To this end, we consider the Wasserstein proximal mapping defined as
$$
\prox_{\tau \F}(\hat \mu)=\argmin_{\mu\in\P_2(\R^d)}\left\{\tfrac12W_2^2(\mu,\hat\mu)+\tau\F(\mu)\right\}.
$$
Then, define as $\mu_\tau^k$ for $k\in\mathbb{N}$ the steps of the minimizing movements scheme, i.e.,
\begin{equation}\label{eq:JKO_discrete}
\mu_\tau^0=\mu^0,\qquad \mu_\tau^{k+1}=\prox_{\tau\F}(\mu_\tau^k).
\end{equation}
We denote the piecewise constant interpolations
$
\tilde \gamma_\tau\colon[0,\infty)\to\P_2(\R^d)
$
of the minimizing movement scheme by
\begin{equation}\label{otto_curve}
\tilde \gamma_\tau(k\tau+t\tau)=\mu_\tau^k,\quad t\in[0,1).
\end{equation}
Then, the following convergence result holds true. We recall the necessary definitions of coercivity and $\lambda$-convexity in Wasserstein spaces in Appendix~\ref{app:backgrounds}.

\begin{theorem}[{\citealp[Thm~11.2.1]{AGS2005}}]   \label{thm:existence_gflows_ggd}
    Let $\F\colon \P_2(\R^d) \to (-\infty,+\infty]$ be 
    proper, 
    lsc, 
    coercive,
    and $\lambda$-convex along generalized geodesics,
    and let $\mu^0 \in \overline{\dom \F}$.
    Then the curves $\tilde \gamma_\tau$ defined via the minimizing movement scheme \eqref{otto_curve} 
    converge for $\tau \to 0$ locally uniformly to 
    a locally Lipschitz curve $\gamma \colon (0,+\infty) \to \P_2(\R^d)$
    which is the unique Wasserstein gradient flow of $\F$ with $\gamma(0+) = \mu^0$.  
\end{theorem} 

\subsection{Continuous Normalizing Flows and OT-Flows}\label{sec:CNF}

The concept of normalizing flows first appeared in \cite{RM2015}.
It follows the basic idea to approximate a probability distribution $\nu$ by considering a simple latent distribution $\mu_0$ (usually a standard Gaussian) and to construct a diffeomorphism $\mathcal T_\theta\colon\R^d\to\R^d$ depending on some parameters $\theta$ such that $\nu\approx{\mathcal T_\theta}_\#\mu_0$.
In practice, the diffeomorphism can be approximated by coupling-based neural networks \citep{DSB2016,KD2018}, residual architectures \citep{BGCD2019,CBDJ2029,H2022} or autoregressive flows \citep{CTA2019,DBMP2019,HKLC2018,PPM2017}. 
In this paper, we mainly focus on \emph{continuous normalizing flows} proposed by \cite{CRBD2018,GCBSD2018}, see also \cite{RH2021} for an overview. Here the diffeomorphism $\mathcal T_\theta$ is parameterized as neural ODE.
To this end let $v_\theta\colon\R^d\times\R\to\R^d$ be a neural network with parameters $\theta$ and let 
$z_\theta\colon\R^d\times[0,\tau]\to\R^d$ for fixed $\tau>0$ be the solution of
$\dot z_\theta=v_\theta$, with initial condition $z_\theta(x,0)=x$.
Then, we define $\mathcal T_\theta$ via the solution $z_\theta$ as $\mathcal T_\theta(x)=z_\theta(x,\tau)$.
The density $p_\theta$ of ${\mathcal T_\theta}_\#\mu_0$ can be described by the change-of-variables formula $p_{\theta}(x)=p_0(x)/|\mathrm{det}(\nabla \mathcal T_\theta(x))|$,
where $p_0$ is the density of the latent distribution $\mu_0$. In the case of continuous normalizing flows, it can be shown that the denominator can be computed as
$\log(|\mathrm{det}(\nabla \mathcal T_\theta(x))|)=\ell_\theta(x,\tau)$ with $\partial_t\ell_\theta(x,t)=\mathrm{trace}(\nabla v_\theta(z_\theta(x,t),t))$ and $ \ell_{\theta}(\cdot,0)=0$.
In order to train a normalizing flow, one usually uses the Kullback-Leibler divergence. If $\nu$ is given by a density $q(x)=g(x)/Z_g$, where $Z_g$ is an unknown normalizing constant, then this amounts to the reverse KL loss function
\begin{align}
\mathcal L(\theta)&=\mathrm{KL}({\mathcal T_\theta}_\#\mu_0 , \nu)\\
&=\E_{x\sim\mu_0}[-\log(q(\mathcal T_\theta(x)))+\log(p_\theta(x))]\\
&=\E_{x\sim\mu_0}[-\log(g(\mathcal T_\theta(x)))-\ell_\theta(x,\tau)] + C, 
\end{align}
for some constant $C$ independent of $\theta$.

In order to stabilize and accelerate the training, \cite{OFLR2021} propose to regularize the velocity field $v_\theta$ by its expected squared norm. More precisely, they propose to add the regularizer
$
\mathcal R(\theta)=\tau\int_0^\tau \|v_\theta(z_\theta(x,t),t)\|^2\d t
$
to the loss function. 
This leads to straight trajectories in the ODE such that adaptive solvers only require very few steps to solve them.
Following Theorem~\ref{thm:benamou-brenier}, the authors of \cite{OFLR2021} note that for $\beta>0$ the functional $\mathcal L(\theta)+\beta\mathcal R(\theta)$ has the same minimizer as the functional $\mathcal{L}(\theta)+\beta W_2^2(\mu_0,{\mathcal T_\theta}_\#\mu_0)$, which relates to the JKO scheme as pointed out by \cite{VWTON2023}. 

\section{Neural JKO Scheme}\label{sec:neural_JKO}

In the following, we learn the steps \eqref{eq:JKO_discrete} of the JKO scheme by neural ODEs. While similar schemes were already suggested in several papers \citep{AHS2023,ASM2022,FZTC2022,LCBBR2022,MKLGS2021,VWTON2023,XCX2024}, we are particularly interested in the convergence properties of the corresponding velocity fields.
In Subsection~\ref{sec:geodesic_interpolations}, we introduce the general scheme and derive its properties. Afterwards, in Subsection~\ref{sec:neural_nets}, we describe the corresponding neural network approximation.
Throughout this section, we consider the following assumptions on the objective functional $\F$ and our initialization $\mu_0$. 

\begin{assumption}\label{ass:F}
Let $\F\colon\P_2(\R^d)\to\R\cup\{\infty\}$ be proper, lower semi-continuous with respect to narrow convergence, coercive, $\lambda$-convex along generalized geodesics and bounded from below. Moreover, assume that $\dom(|\partial\F|)\subseteq \P_2^\mathrm{ac}(\R^d)$ and that $\F$ has finite metric derivative $|\partial\F|(\mu_0)<\infty$ at the initialization $\mu_0\in\P_2^\mathrm{ac}(\R^d)$.
\end{assumption}

This assumption is fulfilled for many important divergences and loss functions $\F$. We list some examples in Appendix~\ref{ex:divergences}.
We will later pay particular attention to the reverse Kullback-Leibler divergence
$
\F(\mu)=\mathrm{KL}(\mu,\nu)=\int p(x)\log\left(\frac{p(x)}{q(x)}\right)\d x
$,
where $\nu$ is a fixed target measure and $p$ and $q$ are the densities of $\mu$ and $\nu$ respectively. This functional fulfills Assumption~\ref{ass:F} if $-\log(q)$ is $\lambda$-convex.

We stress the fact that the $\lambda$-convexity property is assumed for some $\lambda\in\mathbb{R}$, which explicitly includes negative values. Consequently, the theoretical results are also applicable for target densities which are not log-concave. 

Additionally, \citep[Lem 9.2.7]{AGS2005} states that the functional $\mathcal G(\mu)=\frac{1}{2\tau}W_2^2(\mu,\mu_\tau^k)+\F(\mu)$ is $(\lambda+\frac1\tau)$-convex along geodesics. In particular, for $\tau<\frac1\lambda$, the functional $\mathcal G$ is strongly convex such that we expect that optimizing it with a generative model is much easier than optimizing $\F$, see also Appendix~\ref{app:nf_collapse} for a discussion how this can prevent mode collapse.

\subsection{Piecewise Geodesic Interpolation}\label{sec:geodesic_interpolations}

In order to represent the JKO scheme by neural ODEs, we first reformulate it based on Theorem~\ref{thm:benamou-brenier}. 
To this end, we insert the dynamic formulation of the Wasserstein distance in the Wasserstein proximal mapping defining the steps in \eqref{eq:JKO_discrete}.
This leads to 
$
\mu_\tau^{k+1}=z_\tau^k(\cdot,\tau)_\#\mu_\tau^k,
$
where
\begin{equation}\label{eq:dynamic_JKO}
(v_{\tau,k},z_{\tau,k})\in\argmin_{\substack{v\colon\R^d\times[0,\tau]\to\R^d\\\dot z(x,t)=v(z(x,t),t),\,z(x,0)=x}} \mathcal{E}_{\mu_\tau^k}(z,v),
\end{equation}
where $\mathcal{E}_{\mu_\tau^k}(z,v)$ is given by the dynamic formulation \eqref{eq:rescaled_BB} as
\begin{equation}
     \frac1{2}\int_{0}^{\tau}\int_{\R^d}\|v(z(x,t),t)\|^2\d\mu_\tau^k(x)\d t\\
+\F(z(\cdot,\tau)_\#\mu_\tau^k).
\end{equation}

Finally, we concatenate the velocity fields of all steps by $v_\tau|_{(k\tau,(k+1)\tau]}=v_{\tau,k}$ and obtain the ODE
$$
\dot z_\tau(x,t)=v_\tau(z_\tau(x,t),t),\qquad z_\tau(x,0)=x.
$$
As a straightforward observation, we obtain that the curve defined by the velocity field $v_\tau$ is the geodesic interpolation between the points from JKO scheme \eqref{eq:JKO_discrete}.
We state a proof in Appendix~\ref{proof:geodesic}.
\begin{corollary}\label{cor:geodesic}
Under Ass.~\ref{ass:F} the following holds true.
\vspace{-.2cm}
\begin{enumerate}
\item[(i)] It holds that 
{\small
$$
W_2^2(\mu_\tau^k,\mu_\tau^{k+1})=\tau\!\int_0^\tau \!\!\int_{\R^d} \|v_{\tau,k}(z_{\tau,k}(x,t),t)\|^2\d \mu_\tau^k(x)\d t,
$$}
i.e., $v_{\tau,k}$ is the optimal velocity field from Theorem~\ref{thm:benamou-brenier}.
\item[(ii)] The curve $\gamma_\tau(t)\coloneqq z(\cdot,t)_\#\mu_0$ fulfills
$
\gamma_\tau(k\tau+t\tau)=((1-t)I+t\mathcal{T}_\tau^k)_\#\mu_\tau^k$ for $t\in[0,1],
$
where $\mathcal{T}_\tau^k$ is the optimal transport map between $\mu_\tau^k$ and $\mu_\tau^{k+1}$.
\item[(iii)] $v_\tau$ and $\gamma_\tau$ solve the continuity equation
$$
\partial_t\gamma_\tau(t)+\nabla\cdot(v_\tau(\cdot,t)\gamma(t))=0.
$$
\end{enumerate}
\end{corollary}
Analogously to Theorem~\ref{thm:existence_gflows_ggd} one can show that also the curves $\gamma_\tau$ are converging locally uniformly to the unique Wasserstein gradient flow (see, e.g., the proof of \citealp[Thm 11.1.6]{AGS2005}). 
In the next subsection, we will approximate the velocity fields $v_{\tau,k}$ by neural networks.
In order to retain the stability of the resulting scheme, the next theorem states that also the velocity fields $v_\tau$ converge strongly towards the velocity field from the Wasserstein gradient flow. The proof is given in Appendix~\ref{proof:convergence}.

\begin{theorem}\label{thm:convergence}
Suppose that Assumption~\ref{ass:F} is fulfilled and
let $(\tau_l)_l\subseteq(0,\infty)$ with $\tau_l\to 0$. Then, $(v_{\tau_l})_l$ converges strongly to the velocity field $\hat v\in L_2(\gamma,\R^d\times[0,T])$ of the Wasserstein gradient flow $\gamma\colon(0,\infty)\to\P_2(\R^d)$ of $\F$ starting in $\mu^0$.
\end{theorem}

In some cases the limit velocity field $v$ can be stated explicitly, even though its direct computation is intractable. For details, we refer to Appendix~\ref{ex:divergences}.

\subsection{Neural JKO Sampling}\label{sec:neural_nets}
In the following, we learn the velocity fields $v_{\tau,k}$ as neural ODEs in order to sample from a target measure $\nu$ given by the density $q(x)=\frac1{Z_g} g(x)$ with unknown normalization constant $Z_g=\int_{\R^d}g(x)\d x$. 
To this end, we consider the Wasserstein gradient flow with respect to the reverse KL loss function $\F(\mu)=\mathrm{KL}(\mu,\nu)$ which has the unique minimizer $\mu=\nu$. Note that similar derivations were done in \cite{VWTON2023,XCX2024} for the application of generative modeling instead of sampling.

Then, due to~\eqref{eq:dynamic_JKO} the loss function $\mathcal{L}(\theta)$  from the JKO steps for the training of the velocity field $v_\theta$ reads as
$$
\frac1{2}\int_{0}^{\tau}\int_{\R^d}\|v_\theta(z_\theta(x,t),t)\|^2\d\mu_\tau^k(x)\d t+\F(z_\theta(\cdot,\tau)_\#\mu_\tau^k),
$$
where $z_\theta$ is the solution of $\dot z_\theta(x,t)=v_\theta(z_\theta(x,t),t)$ with $z_\theta(x,0)=x$.
Now, following the derivations of continuous normalizing flows, cf.~Section~\ref{sec:CNF}, the second term of $\mathcal L$ can be rewritten (up to an additive constant) as
$$
\F(z_\theta(\cdot,\tau)_\#\mu_\tau^k)\propto\E_{x\sim\mu_\tau^k}[-\log(g(z_\theta(x,\tau)))-\ell_\theta(x,\tau)],
$$
where $\ell_\theta$ solves \smash{$\dot \ell_\theta(x,t)=\mathrm{trace}(\nabla v_\theta(z_\theta(x,t),t))$} with $\ell_\theta(\cdot,0)=0$.
Moreover, we can rewrite the first term of $\mathcal L$ based on 
\begin{align*}
\int_{0}^{\tau}\int_{\R^d}\|v_\theta(z_\theta(x,t),t)\|^2\d\mu_\tau^k(x)\d t=
\E_{x\sim\mu_\tau^k}[\omega_\theta(x,\tau)],
\end{align*}
where $\omega_\theta$ is the solution of $\dot \omega_\theta(x,t)=\|v_\theta(z_\theta(x,t),t)\|^2$.
Hence, we can represent $\mathcal L$ up to an additive constant as
\begin{equation}\label{eq:loss_neural_JKO}
\mathcal L(\theta)=\E_{x\sim\mu_\tau^k}[-\log(g(z_\theta(x,\tau)))-\ell_\theta(x,\tau)+\omega_\theta(x,\tau)],
\end{equation}
where $(z_\theta,\ell_\theta,\omega_\theta)$ solves the ODE system
\begin{equation}\label{eq:neuralJKO_ODE}
\left(\begin{array}{c}\dot z_\theta(x,t)\\\dot \ell_\theta(x,t)\\\dot\omega_\theta(x,t)\end{array}\right)=\left(\begin{array}{c}v_\theta(z_\theta(x,t),t)\\\mathrm{trace}(\nabla v_\theta(z_\theta(x,t),t))\\\|v_\theta(z_\theta(x,t),t)\|^2\end{array}\right),
\end{equation}
with initial conditions $z_\theta(x,0)=x$, $\ell_\theta(x,0)=0$ and $\omega_\theta(x,0)=0$.
In particular, the loss function $\mathcal L$ can be evaluated and differentiated based on samples from $\mu_\tau^k$.
Once the parameters $\theta$ are optimized, we can evaluate the JKO steps in the same way as standard continuous normalizing flows. We summarize training and evaluation of the JKO steps in Algorithm~\ref{alg:training_CNF} and \ref{alg:density_probagation_CNF} in Appendix~\ref{app:algorithms_training_eval}. In practice, the density values are computed and stored in log-space for numerical stability.
Additionally, note that the continuous normalizing flows can be replaced by other normalizing flow architectures, see Appendix~\ref{app:cnf_details} and Remark~\ref{rem:other_architectures} for details.

\section{Importance-Based Rejection Steps}\label{sec:rejection_steps}

While a large number of existing sampling methods rely on Wasserstein gradient flows with respect to some divergences, it is well known that these loss functions are non-convex. This leads to very slow convergence or only convergence to suboptimal local minima.
In particular, if the target distribution is multimodal the modes often do not have the right mass assigned.
In this section, we derive a novel rejection step that corrects such imbalanced mass assignments. Remarkably, we will see in Theorem~\ref{thm:density_propagation_rejection}, that it is possible to access the density after these rejection steps. 
This observation is crucial, since it allows for the iterative combination of several rejection and neural JKO steps, leading to the proposed \textit{importance corrected neural JKO sampling}.

\noindent
\textbf{Importance Sampling}\,\,\,
As a remedy, many sampling algorithms from the literature are based on importance weights, see, e.g., sequential Monte Carlo samplers \citep{DDJ2006} or annealed importance sampling \cite{N2001}. That is, we assign to each generated sample $x_i$ a weight $w_i=\frac{q(x_i)}{p(x_i)}$, where $p$ is some proposal density and $q$ is the density of the target distribution $\nu$. Then, for any $\nu$-integrable function $f\colon\mathbb{R}^d\to\mathbb{R}$ it holds that $\sum_{i=1}^N w_i f(x_i)$ is an unbiased estimator of $\int_{\mathbb{R}^d} f(x) d\nu(x)$.  Note that importance sampling is very sensitive with respect to the proposal $p$ which needs to be designed carefully and problem adapted.

\noindent
\textbf{Rejection Steps}\,\,\,
Inspired by importance sampling, we propose to use importance-based rejection steps.
More precisely, let $\mu$ be a proposal distribution where we can sample from with density $p(x)=f(x)/Z_f$ and denote by $\nu$ the target distribution with density $q(x)=g(x)/Z_g$. In the following, we assume that we have access to the unnormalized densities $f$ and $g$, but not to the normalization constants $Z_f$ and $Z_g$
Then, for a random variable $X\sim \mu$, we now generate a new random variable $\tilde X$ by the following procedure: First, we compute the importance based acceptance probability $\alpha(X)=\min\left\{1,\frac{q(X)}{\tilde cp(X)}\right\}=\min\left\{1,\frac{g(X)}{cf(X)}\right\}$, where $c>0$ is a positive hyperparameter and $\tilde c=c Z_f/Z_g$. Then, we set $\tilde X=X$ with probability $\alpha(X)$ and choose $\tilde X=X'$ otherwise, where $X'\sim \mu$ and $X$ are independent. 

\begin{remark}
This is a one-step approximation of the classical rejection sampling scheme \citep{V1951}, see also \citep{ADDJ2003} for an overview.
More precisely, we arrive at the classical rejection sampling scheme by choosing $\tilde c>\sup_x q(x)/p(x)$ and redo the procedure when $X$ is rejected instead of choosing $\tilde X=X'$.
\end{remark}

Similarly to importance sampling, the rejection sampling algorithm is highly sensitive towards the proposal distribution $p$. 
In particular, it suffers from the curse of dimensionality, in case of a non-tailored proposal $p$.
We will tackle this problem later in the section by choosing $p$ already close to the target density $q$.

The following theorem describes the density of the distribution $\tilde \mu$ of $\tilde X$. Moreover, it ensures that the KL divergence to the target distribution decreases. We include the proof in Appendix~\ref{app:density_propagation_rejection}

\begin{theorem}
\label{thm:density_propagation_rejection}
Let $\tilde \mu$ be the distribution of $\tilde X$. Then, the following holds true.
\begin{enumerate}
\item[(i)] $\tilde \mu$ admits the density $\tilde p$ given by
$
\tilde p(x)=p(x)(\alpha(x)+1-\E[\alpha(X)])$.
In particular, we have $\tilde p(x)=\tilde f(x)/Z_{\tilde f}$ with $\tilde f(x)=f(x)(\alpha(x)+1-\E[\alpha(X)])$.
\item[(ii)] It holds that $\mathrm{KL}(\tilde\mu,\nu)\leq \mathrm{KL}(\mu,\nu)$. 
\end{enumerate}
\end{theorem}

The application of a single \textit{importance-based rejection step} is summarized in
Algorithm~\ref{alg:density_probagation_rejection} and illustrated in Figure~\ref{fig:flowchart_rejection} in the appendix.
Note that the value $\E[\alpha(X)]$ can easily be estimated based on samples during the training phase. Indeed, given $N$ iid copies $X_1,...,X_N$ of $X$, we obtain that $\E[\alpha(X)]\approx \frac1N\sum_{i=1}^N \alpha(X_N)$ is an unbiased estimator fulfilling the error estimate from the following corollary. The proof is a direct consequence of Hoeffding's inequality and given in Appendix~\ref{proof:concentration}.
\begin{corollary}\label{cor:concentration}
Let $X_1,...,X_N$ be iid copies of $X$. Then, it holds
$$
\E\left[\left|\E[\alpha(X)]-\frac1N\sum_{i=1}^N \alpha(X_N)\right|\right]\leq \frac{\sqrt{2\pi}}{\sqrt{N}} \in O\left(\frac{1}{\sqrt{N}}\right).
$$
\end{corollary}

\begin{algorithm}
\begin{algorithmic}
\STATE \textbf{Input:} $\left\{ \begin{array}{l}\text{- samples $x_1^k,...,x_N^k$ of $\mu^k$ with density $p^k$}, \\
        \text{- hyperparameter $c$ (see Remark~\ref{rem:choice_c})}.
        \end{array}
        \right.$
\STATE
\STATE \textbf{Assume:}
$\left\{ \begin{array}{l}\text{- can draw samples from $\mu^k$},\\
\text{- can evaluate the unnormalized density $g$.}
\end{array}\right.$
\STATE
\FOR{$i=1,...,N$}
\STATE 1. Compute $\alpha_k(x_i^k)=\min\left\{1,\frac{g(x_i^k)}{cp^k(x_i^k)}\right\}$.
\STATE 2. Draw $u$ uniformly from $[0,1]$ and $x'$ from $\mu_\tau^k$.
\STATE 3. Set $x_i^{k+1}=
\begin{cases}x_i^k, &\text{if }u\leq\alpha(x_i^k),\\
                  x'&\text{if }u>\alpha(x_i^k).\end{cases}$
\STATE 4. Compute $\alpha_k(x_i^{k+1})=\min\left\{1,\frac{g(x_i^{k+1})}{cp^k(x_i^{k+1})}\right\}$.
\STATE 5. Define the density value
\small{
$$p^{k+1}(x_i^{k+1})=p^k(x_i^{k+1})(\alpha_k(x_i^{k+1})+1-\E[\alpha_k(X_k)]).$$
}
\ENDFOR
\STATE
\STATE \textbf{Output:}
$\left\{ \begin{array}{l}\text{- Samples $x_1^{k+1},...,x_N^{k+1}$ of $\mu_\tau^{k+1}$.}\\
\text{- Density values  $\left\{p_\tau^{k+1}(x_i^{k+1})\right\}_{i=1}^N$.}
\end{array}\right.$
\end{algorithmic}
\caption{Sampling and density propagation for importance-based rejection steps}
\label{alg:density_probagation_rejection}
\end{algorithm}

\begin{remark}[Choice of $c$]\label{rem:choice_c}
We choose the hyperparameter $c$ such that a constant ratio $r>0$ of the samples will be resampled, i.e., that $\E[\alpha(X)]\approx 1-r$. To this end, we assume that we are given samples $x_1,...,x_N$ from $X$ and approximate
$$
\E[\alpha(X)]\approx\frac1N\sum_{i=1}^N\alpha(x_i)=\frac1N\sum_{i=1}^N\min\left\{1,\frac{g(x_i)}{cp(x_i)}\right\}.
$$
Note that the right side of this formula depends monotonically on $c$ such that we can find $c>0$ such that $\E[\alpha(X)]=1-r$ by a bisection search. In our numerical experiments, we set $r=0.2$.
We summarize the choice of $c$ in Algorithm~\ref{alg:training_rejection} in Appendix~\ref{app:algorithms_training_eval}.
\end{remark}

\begin{algorithm}
	\caption{
Importance corrected neural JKO sampling
}
\begin{algorithmic}
 \STATE \textbf{Input:}
 {
    $\left\{\begin{array}{ll}
         \text{\tabitem unnormalized target density $g$},   \\[0.2em] 
         \text{\tabitem initial measure $\mu^0$ with density $p_0$,}  \\[0.2em]
         \text{\tabitem Number $N\in\mathbb{N}$ of samples,} \\[0.2em]
         \text{\tabitem Number $K\in\mathbb{N}$ of total steps.}
    \end{array}
    \right. $
 }
  \vspace{0.1cm}
 \STATE \textbf{Output:} Sample generator $\{x^{i}\}_{i=1}^N \sim \hat{\nu}\approx \nu=g\cdot\lambda$.
 \vspace{0.1cm}
 \STATE Let $x_1^0,\ldots, x_N^0\sim\mu^0$.
 \FOR{$k=1,\ldots,K$}
    \STATE Define $\mu^k$ with density $p^k$ and draw samples $x_1^{k},\ldots,x_N^{k} \sim \mu^k$ either by
       \STATE \tabitem 
       \parbox[t]{200pt}{%
        \textit{neural JKO step} by solving the ODE \eqref{eq:neuralJKO_ODE} or \strut}
       \STATE \tabitem 
       \parbox[t]{200pt}{
 \textit{importance based rejection} by Alg.~\ref{alg:density_probagation_rejection}.
       }
    \STATE (details in Alg.~\ref{alg:training_CNF}, \ref{alg:density_probagation_CNF} resp. Alg.~\ref{alg:density_probagation_rejection}, \ref{alg:training_rejection})
 \ENDFOR
 \STATE Set $\hat{\nu} = \mu^K$ with density $p^K$.
    \label{alg:importance_corrected_JKO}
\end{algorithmic}
\end{algorithm}

\noindent
\textbf{Neural JKO Sampling with Importance Correction}\,\,\, Finally, we combine the neural JKO scheme from the previous section with our rejection steps to obtain a sampling algorithm.
More precisely, we start with a simple latent distribution $\mu_0$ with known density and which we can sample from.
In our numerical experiments this will be a standard Gaussian distribution.
Now, we iteratively generate distributions $\mu_k$, $k=1,...,K$ by applying either neural JKO steps as described in Algorithm~\ref{alg:training_CNF} and \ref{alg:density_probagation_CNF} or importance-based rejection steps as described in Algorithm~\ref{alg:training_rejection} and \ref{alg:density_probagation_rejection}.
We call the resulting model an importance corrected neural JKO model, which we summarize in Algorithm~\ref{alg:importance_corrected_JKO}.
During the sampling process we can maintain the density values $p^k(x)$ of the density $p^k$ of $\mu_k$ for the generated samples by the Algorithms~\ref{alg:density_probagation_CNF} and \ref{alg:density_probagation_rejection}.
Moreover, we can also use Theorem~\ref{thm:density_propagation_rejection} to evaluate the density $p^k$ at some arbitrary point $x\in\R^d$. We outline this density evaluation process in Appendix~\ref{app:density_evaluation}.

\begin{table*}
\caption{Energy distance. We run each method $5$ times and state the average value and corresponding standard deviations. The rightmost column shows the reference sampling error, i.e., the lower bound magnitude of the average energy distance between two sets of different samples drawn from the ground truth. A smaller energy distance indicates a better result.}
\centering
\scalebox{.55}{
\begin{tabular}{ccrrrrrrcc}
\toprule
  & \phantom{.}   &  \multicolumn{6}{c}{Sampler}  &  \phantom{.} & \\
  \cmidrule{3-8}  
Distribution& &\multicolumn{1}{c}{MALA}&\multicolumn{1}{c}{HMC}&\multicolumn{1}{c}{DDS}&\multicolumn{1}{c}{CRAFT}&\multicolumn{1}{c}{Neural JKO}&\multicolumn{1}{c}{Neural JKO IC (\textbf{ours})}& &Sampling Error\\
\midrule
Mustache&&$\num{4.6e-2}\pm\num{1.6e-3}$&$\num{1.7e-2}\pm\num{4.3e-4}$&$\num{6.9e-2}\pm\num{1.8e-3}$&$\num{9.2e-2}\pm\num{9.9e-3}$&$\num{1.8e-2}\pm \num{2.0e-3}$&$ \best  \num{2.9e-3}\pm \num{4.4e-4}$&&$\num{8.6e-5}$\\
shifted 8 Modes&&$\num{5.3e-3}\pm\num{4.9e-4}$&$\num{4.1e-5}\pm\num{3.3e-5}$&$\num{1.2e-2}\pm\num{4.1e-3}$&$\num{5.2e-2}\pm\num{1.1e-2}$&$\num{1.3e-1}\pm\num{3.8e-3}$&$\best  \num{1.2e-5} \pm \num{5.1e-6} $&&$\num{2.6e-5}$\\
shifted 8 Peaky&&$\num{1.3e-1}\pm\num{3.2e-3}$&$\num{1.2e-1}\pm\num{2.4e-3}$&$\num{1.1e-2}\pm\num{3.4e-3}$&$\num{5.2e-2}\pm\num{2.2e-2}$&$\num{1.3e-1}\pm\num{2.2e-3}$&$\best  \num{3.4e-5} \pm \num{8.2e-6} $&&$\num{2.4e-5}$\\
Funnel&&$\num{1.2e-1}\pm\num{3.1e-3}$&$\best  \num{3.1e-3} \pm \num{3.2e-4} $&$\num{2.6e-1}\pm\num{2.6e-2}$&$\num{7.4e-2}\pm\num{2.8e-3}$&$\num{4.6e-2}\pm\num{1.6e-3}$&$\num{1.4e-2}\pm\num{8.2e-4}$&&$\num{3.4e-4}$\\
GMM-10&&$\num{1.2e-2}\pm\num{5.5e-3}$&$\num{1.2e-2}\pm\num{5.2e-3}$&$\num{3.7e-3}\pm\num{1.6e-3}$&$\num{1.8e-1}\pm\num{6.6e-2}$&$\num{1.1e-2}\pm\num{5.6e-3}$&$\best  \num{5.3e-5} \pm \num{1.7e-5} $&&$\num{4.6e-5}$\\
GMM-20&&$\num{9.1e-3}\pm\num{2.8e-3}$&$\num{9.1e-3}\pm\num{2.7e-3}$&$\num{5.0e-3}\pm\num{1.5e-3}$&$\num{5.4e-1}\pm\num{1.4e-1}$&$\num{1.0e-2}\pm\num{2.8e-3}$&$\best  \num{1.1e-4} \pm \num{3.4e-5} $&&$\num{6.4e-5}$\\
GMM-50&&$\num{2.4e-2}\pm\num{7.5e-3}$&$\num{2.4e-2}\pm\num{7.5e-3}$&$\num{2.3e-2}\pm\num{1.1e-2}$&$\num[retain-zero-exponent]{1.8e+0}\pm\num{1.7e-1}$&$\num{2.7e-2}\pm\num{7.8e-3}$&$\best  \num{1.0e-4} \pm \num{4.6e-5} $&&$\num{1.1e-4}$\\
GMM-100&&$\num{3.6e-2}\pm\num{1.6e-2}$&$\num{3.7e-2}\pm\num{1.7e-2}$&$\num{3.9e-2}\pm\num{2.1e-2}$&$\num{2.8e+1}\pm \num{1.0e-1}$&$\num{4.7e-2}\pm\num{2.2e-2}$&$\best  \num{6.0e-4} \pm \num{3.4e-4} $&&$\num{1.5e-4}$\\
GMM-200&&$\num{6.4e-2}\pm\num{2.1e-2}$&$\num{6.6e-2}\pm\num{1.9e-2}$&$\num{9.8e-2}\pm\num{3.1e-2}$&$\num[retain-zero-exponent]{3.9e+0}\pm\num{1.6e-1}$&$\num{8.9e-2}\pm\num{2.7e-2}$&$\best  \num{3.3e-3} \pm \num{1.9e-3} $&&$\num{2.0e-4}$\\
\bottomrule
\end{tabular}
}
\label{tab:energy_distance}
\end{table*}

\begin{table*}
\centering
\caption{Estimated $\log(Z)$. We run each method $5$ times and state the average value and corresponding standard deviations. The DDS values for LGCP are from \cite{VGD2023}. Higher values of $\log(Z)$ estimates correspond to better results.
}
\scalebox{.55}{
\begin{tabular}{ccrrrrcc}
\toprule
  & \phantom{.}   &  \multicolumn{4}{c}{Sampler}  &  \phantom{.} & \\
  \cmidrule{3-6}  
Distribution& &\multicolumn{1}{c}{DDS}&\multicolumn{1}{c}{CRAFT}&\multicolumn{1}{c}{Neural JKO}&\multicolumn{1}{c}{Neural JKO IC (\textbf{ours})}& &Ground Truth\\
\midrule
Mustache&&$\num{-1.5e-1}\pm\num{2.7e-2}$&$\num{-6.5e-2}\pm\num{5.5e-2}$&$\num{-3.0e-2}\pm\num{2.6e-3}$&$\best \num{-7.3e-3 }\pm  \num{8.2e-4 }$&&$0$\\
shifted 8 Modes&&$-\num{5.7e-2}\pm\num{2.0e-2}$&$-\num{1.2e-2}\pm\num{1.4e-3}$&$-\num{3.4e-1}\pm\num{3.1e-3}$&$\best +\num{5.1e-6 }\pm \num{2.4e-3 }$&&$0$\\
shifted 8 Peaky&&$-\num{1.2e-1}\pm\num{2.2e-2}$&$-\num{1.8e-3}\pm\num{2.6e-3}$&$-\num{3.5e-1}\pm\num{3.1e-3}$&$\best -\num{2.1e-3 }\pm \num{3.2e-3 }$&&$0$\\
Funnel&&$-\num{1.8e-1}\pm\num{6.8e-2}$&$-\num{1.2e-1}\pm\num{7.9e-3}$&$-\num{1.4e-1}\pm\num{1.6e-3}$&$\best- \num{7.1e-3 }\pm \num{1.9e-3 }$&&$0$\\
GMM-10&&$\num{-2.3e-1}\pm\num{1.0e-1}$&$\num{-8.5e-1}\pm\num{1.7e-1}$&$\num{-4.3e-1}\pm\num{5.1e-2}$&$\best + \num{3.5e-3 }\pm \num{2.0e-3 }$&&$0$\\
GMM-20&&$\num{-5.1e-1}\pm\num{6.0e-2}$&$-\num{1.5e+0}\pm\num{1.7e-1}$&$-\num{6.3e-1}\pm\num{2.7e-2}$&$\best + \num{6.4e-3 }\pm \num{3.8e-3 }$&&$0$\\
GMM-50&&$-\num{1.3e+0}\pm\num{3.3e-1}$&$-\num{2.3e+0}\pm\num{1.5e-3}$&$-\num{9.3e-1}\pm\num{4.6e-2}$&$\best +\num{1.1e-2 }\pm \num{3.9e-3 }$&&$0$\\
GMM-100&&$-\num{3.0e+0}\pm\num{7.3e-1}$&$-\num{2.3e+0}\pm\num{8.6e-2}$&$-\num{1.8e+0}\pm\num{9.6e-2}$&$\best -\num{3.9e-2 }\pm \num{7.8e-3 }$&&$0$\\
GMM-200&&$-\num{9.4e+0}\pm\num{7.2e-1}$&$-\num{6.3e+0}\pm\num{1.5e-1}$&$-\num{5.2e+0}\pm\num{2.5e-1}$&$\best -\num{5.6e-2 }\pm \num{1.3e-2 }$&&$0$\\
LGCP&&$503.0\pm\num{7.7e-1}$&$507.6\pm\num{3.2e-1}$&$499.9\pm\num{1.7e-1}$&$\best 508.2\pm \num{1.0e-1 }$&&not available\\
\bottomrule
\end{tabular}
}
\label{tab:logZ}
\end{table*}

\begin{remark}[Runtime Limitations]\label{rem:time_constraint}
The sampling time of our importance corrected neural JKO sampling depends exponentially on the number of rejection steps since in each rejection step we resample a constant fraction of the samples. However, due to the moderate exponential base of $1+r$ we will see in the numerical part that we are able to perform a significant number of rejection steps in a tractable time.
\end{remark}

In our numerics, we build our importance corrected neural JKO model by first applying $n_1\in\mathbb{N}$ neural JKO steps followed by $n_2\in\mathbb{N}$ blocks consisting out of one neural JKO step and three importance-based rejection steps, where $n_1$ and $n_2$ are hyperparameters given for each model separately in Table~\ref{tab:hyper}. For the neural JKO steps, we choose an initial step size $\tau_0>0$ as a hyper-parameter and then increase the step size exponentially by $\tau_{k+1}=4\tau_k$. Note that one could alternatively use adaptive step sizes similar to \cite{XCX2024}. However, for our setting, we found that the simple step size rule is sufficient.

\section{Numerical Results}\label{sec:numerics}

We compare our method with classical Monte Carlo samplers like a Metropolis adjusted Langevin sampling (MALA) and Hamiltonian Monte Carlo (HMC), see e.g., \cite{B2017,RT1996}.
Additionally, we compare with two recent deep-learning based sampling algorithms, namely denoising diffusion samplers (DDS, \citealp{VGD2023}) and continual repeated annealed flow transport Monte Carlo (CRAFT, \citealp{MARD2022}). We evaluate all methods on a set of common test distributions which is described in detail in Appendix~\ref{sec:test_distributions}.  Moreover, we report the error measures for our \textit{importance corrected neural JKO} sampler (neural JKO IC), see Appendix~\ref{app:implementation_details} for implementation details. Additionally, we emphasize the importance of the rejection steps by reporting values for the same a neural JKO scheme without rejection steps (neural JKO).

For evaluating the quality of our results, we use two different metrics. First, we evaluate the energy distance \citep{S2002}. It is a kernel metric which can be evaluated purely based on two sets of samples from the model and the ground truth. Moreover it encodes the geometry of the space such that a slight perturbation of the samples only leads to a slight change in the energy distance.
Second, we estimate the log-normalizing constant which is equivalent to approximating the reverse KL loss of the model. A higher estimate of the log-normalizing constant corresponds to a smaller KL divergence between generated and target distribution and therefore to a higher similarity of the two measures. Since this requires the density of the model, this approach is not applicable for MALA and HMC.

The results are given in Table~\ref{tab:energy_distance} and \ref{tab:logZ}. Our importance corrected neural JKO sampling significantly outperforms the comparison for all test distributions. In particular, we observe that for shifted 8 modes, shifted 8 peaky and GMM-$d$ the energy distance between neural JKO IC and ground truth samples is in the same order of magnitude as the energy distance between to different sets of ground truth samples. This implies that the distribution generated by neural JKO IC is indistinguishable from the target distribution in the energy distance. For these examples, the $\log(Z)$ esitmate is sometimes slightly larger than the ground truth, which can be explained by numerical effects, see Remark~\ref{rem:biasa} for a detailed discussion.
We point out to a precise description of the metrics, the implementation details and additional experiments and figures in Appendix~\ref{app:detailed_experiments}.

\section{Conclusions}\label{sec:conclusions}
\noindent
\textbf{Methodology}\,\,\,
We proposed a novel and expressive generative method that enables the efficient and accurate sampling from a prescribed unnormalized target density which is empirically confirmed in numerical examples.
To this end, we combine \textit{local sampling steps}, relying on piecewise geodesic interpolations of the JKO scheme realized by CNFs, and \textit{non-local rejection and resampling steps} based on importance weights. 
Since the proposal of the rejection step is generated by the model itself, 
they do not suffer from the curse of dimensionality as opposed to classical variants of rejection sampling.
The proposed approach provides the  advantage that we can draw \textit{independent} samples while correcting imbalanced mode weights, iteratively refine the current approximation and evaluate the density of the generated distribution.
 This is a consequence of the \textit{density value propagation through CNFs} and Theorem~\ref{thm:density_propagation_rejection} and enables possible further post-processing steps that require density evaluations of the approximated sample process.
 
\noindent
\textbf{Outlook}\,\,\,
  Our method allows for the pointwise access to the approximated target density and the log normalization constant. These quantities can be used for the error monitoring during the training and hence provide guidelines for the adaptive design of the emulator in terms of CNF -or rejection/resampling steps and  provide a straightforward stopping criterion. 
  The importance-based rejection steps can also be used in other domains like fine-tuning of score-based diffusion models for downstream tasks, see \citet{DPVH2025}.

\noindent
\textbf{Limitations}\,\,\,
In the situation, when the emulator is realized through a stack of underlying rejection/resampling steps, the sample generation process time is negatively affected, see Remark~\ref{rem:time_constraint}. In order to resolve the drawback we plan to utilize diffusion models for the sample generation. This is part of ongoing and future work by the authors.
Finally, the use of continuous normalizing flows comes with computational challenges, which we discuss in detail in Appendix~\ref{app:cnf_details}.

\section*{Impact Statement}

This paper studies a general problem from the field of statistics and machine learning. There are many potential societal consequences of our work, none which we feel must be specifically highlighted here.

\section*{Acknowledgements}
JH acknowledges funding by the German Research Foundation (DFG) within the Walter-Benjamin Programme with project number 530824055 and by the EPSRC programme grant \textit{The Mathematics of Deep Learning} with reference EP/V026259/1.
RG acknowledges support by the DFG MATH+ project AA5-5 (was EF1-25) -
\textit{Wasserstein Gradient Flows for Generalised Transport in Bayesian Inversion}.

\bibliography{references}
\bibliographystyle{icml2025}

\newpage
\appendix
\onecolumn
\section{Background on Wasserstein Spaces}\label{app:backgrounds}

We give some theoretical background on Wasserstein gradient flows extending Section~\ref{sec:prelim}. In what follows we refer to $L_2(\mu):=L_2(\mu,\mathbb{R}^d)$ as the set of square integrable measurable functions on $\mathbb{R}^d$ with respect to a given measure $\mu\in\mathcal{P}(\mathbb{R}^d)$.

For an absolutely continuous curve $\gamma$, there exists a unique minimal norm solution $v_t$ of the continuity equation \eqref{eq:continuity_equation} in the sense that any solution $\tilde v$ of \eqref{eq:continuity_equation} fulfills $\|\tilde v(\cdot,t)\|_{L_2(\gamma(t))}\geq \|v(\cdot,t)\|_{L_2(\gamma(t))}$ for almost every $t$. This is the unique solution of \eqref{eq:continuity_equation} such that $v(\cdot,t)$ is contained in the regular tangent space
\begin{align} \label{tan_reg}
    {\T}_{\mu}\mathcal P_2(\R^d)
    &\coloneqq 
      \overline{
      \left\{ \lambda (T- \text{Id}): (\text{Id} ,T)_{\#} \mu \in \Gamma^{\text{opt}} (\mu , T_{\#} \mu), \; \lambda >0
      \right\} }^{L_2(\mu)}.
\end{align}
An absolutely continuous curve $\gamma\colon(0,\infty)\to\P_2(\R^d)$ 
with velocity field $v_t\in \T_{\gamma (t)}\P_2(\R^d)$ is a \emph{Wasserstein gradient flow with respect to} $\F\colon\P_2(\R^d)\to(-\infty,\infty]$ 
if 
\begin{equation}
v_t\in -\partial \F(\gamma(t)),\quad \text{for a.e. } t>0,
\end{equation}
where $\partial \F(\mu)$ denotes the reduced Fr\'echet subdiffential at $\mu$ defined as
{\small
\begin{equation}\label{eq:subdiff}
 \partial \F(\mu) \coloneqq \left\{ \xi \in L_2(\mu):   \F(\nu) - \F(\mu)
    \ge 
    \inf_{ \pi \in \Gamma^{\opt}(\mu,\nu)}
    \int\limits_{\R^{d}\times\mathbb{R}^d}
    \langle \xi(x), y - x \rangle
    \, \d \pi (x, y)
    + o(W_2(\mu,\nu)) \; \forall \nu \in \P_2(\R^d) \right\}.
\end{equation}
}
The norm $\|v_t\|_{L_2(\gamma(t))}$ of the velocity field of a Wasserstein gradient flow coincides for almost every $t$ with the metric derivative
$$
|\partial\F|(\mu)=\inf_{\nu\to\mu}\frac{\F(\mu)-\F(\nu)}{W_2(\mu,\nu)}.
$$

For the convergence result from Theorem~\ref{thm:existence_gflows_ggd}, we need two more definitions.
First, we need some convexity assumption.
For $\lambda \in \mathbb R$, $\F\colon\P_2(\R^d)\to\R\cup\{+\infty\}$ is called 
\emph{$\lambda$-convex along geodesics} if, for every 
$\mu, \nu \in \dom \F \coloneqq \{\mu \in  \P_2(\R^d): \F(\mu) < \infty\}$,
there exists at least one geodesics $\gamma \colon [0, 1] \to \P_2(\R^d)$ 
between $\mu$ and $\nu$ such that
\begin{equation*}
    \F(\gamma(t)) 
    \le
    (1-t) \, \F(\mu) + t \, \F(\nu) 
    - \tfrac{\lambda}{2} \, t (1-t) \,  W_2^2(\mu, \nu), 
    \qquad  t \in [0,1].
\end{equation*}
To ensure uniqueness and convergence of the JKO scheme, a slightly stronger condition, namely being
\emph{$\lambda$-convex along generalized geodesics} 
will be in general needed.
Based on the set of three-plans with base 
$\sigma \in \P_2(\R^d)$ given by
\begin{equation*}
    \Gamma_\sigma(\mu,\nu)
    \coloneqq
    \bigl\{ 
    \zb{\alpha} \in \P_2(\R^d \times \R^d \times \R^d)
    :
    (\pi_1)_\# \zb \alpha = \sigma,
    (\pi_2)_\# \zb \alpha = \mu,
    (\pi_3)_\# \zb \alpha = \nu
    \bigr\},
\end{equation*}
the so-called
\emph{generalized geodesics} $\gamma \colon [0, \epsilon] \to \P_2(\R^d)$
joining $\mu$ and $\nu$ (with base $\sigma$) is defined as
\begin{equation} \label{ggd}
  \gamma(t) 
  \coloneqq 
  \bigl(
  (1-\tfrac t \epsilon) \pi_2 
  + \tfrac t \epsilon \pi_3
  \bigr)_\# \zb \alpha,
  \qquad
  t \in [0, \epsilon],
\end{equation}
where $\zb \alpha \in \Gamma_\sigma (\mu,\nu)$
with
$(\pi_{1,2})_\# \zb{\alpha} \in \Gamma^{\opt}(\sigma,\mu)$
and
$(\pi_{1,3})_\# \zb{\alpha} \in \Gamma^{\opt}(\sigma,\nu)$, see Definition~9.2.2 in \cite{AGS2005}.
The plan $\zb{\alpha}$ may be interpreted as transport from $\mu$ to $\nu$ via $\sigma$.
Then
a function $\F\colon \mathcal P_2(\R^d) \to (-\infty,\infty]$ is called
\emph{$\lambda$-convex along generalized geodesics} (see \citealp[Definition~9.2.4]{AGS2005}), 
if for every $\sigma,\mu,\nu \in \dom \F$,
there exists at least one generalized geodesics $\gamma \colon [0,1] \to \P_2(\R^d)$ 
related to some $\zb{\alpha}$ in \eqref{ggd} such that
\begin{equation} \label{eq:gg}
  \F(\gamma(t))
  \le 
  (1-t) \, \F(\mu) 
  + t \, \F(\nu) 
  - \tfrac\lambda2 \, t(1-t) \, W_{\zb \alpha}^2(\mu,\nu), 
  \qquad t \in [0,1],
\end{equation}
where 
\begin{equation*}
W_{\zb{\alpha}}^2 (\mu, \nu)
\coloneqq \int_{\R^d \times \R^d \times \R^d} \|y - z\|_2^2 \, \d \zb{\alpha}(x,y,z).
\end{equation*}
Every function being $\lambda$-convex along generalized geodesics 
is also $\lambda$-convex along geodesics
since generalized geodesics with base $\sigma = \mu$ are actual geodesics.
Second, a $\lambda$-convex functional $\F\colon\P_2(\R^d)\to\R\cup\{+\infty\}$ is called coercive, if there exists some $r>0$ such that
$$
\inf\{\F(\mu):\mu\in\P_2(\R^d), \int_{\R^d}\|x\|^2\d\mu(x)\leq r\} > -\infty,
$$
see \citep[eq. (11.2.1b)]{AGS2005}.
In particular, any functional which is bounded from below is coercive.

If $\F$ is proper, lower semicontinuous, coercive and $\lambda$-convex along generalized geodesics, one can show that the $\prox_{\tau \F}(\mu)$ is non-empty and unique for $\tau$ small enough (see \citealp[page~295]{AGS2005}).

\section{Proofs and Examples from Section~\ref{sec:neural_JKO}}

\subsection{Examples fulfilling Assumption~\ref{ass:F}}\label{ex:divergences}
Assumption~\ref{ass:F} is fulfilled for many important divergences and loss functions $\F$. We list some examples below. While it is straightforward to check that they are proper, lower semicontinuous, coercive and bounded from below, the convexity is non-trivial. However, the conditions under which these functionals are $\lambda$-convex along generalized geodesics are well investigated in \citep[Section 9.3]{AGS2005}.
In the following we denote the Lebesgue measure as $\mu_{\mathrm{Leb}}$. For a measure $\mu\in\mathcal{P}^{\mathrm{ac}}(\mathbb{R}^d)$, we denote by $\mathrm{d}\mu/\mathrm{d}\mu_{\mathrm{Leb}}$ the Lebesgue density of $\mu$ if it exists.

\begin{itemize}
\item[-] Let $\nu\in\P_2^\mathrm{ac}(\R^d)$ with Lebesgue density $q$ and define the \emph{forward Kullback-Leibler (KL) loss function}
$$
\F(\mu)=\mathrm{KL}(\nu, \mu)\coloneqq
\begin{cases}
\int_{\R^d} q(x)\log\left(\frac{q(x)}{p(x)}\right)\d x,&\text{if }\exists \mathrm{d}\mu/\mathrm{d}\mu_{\mathrm{Leb}} =p \text{ and } \mathrm{d}\nu/\mathrm{d}\mu_{\mathrm{Leb}} = q,\\
+\infty&\text{otherwise.}
\end{cases}
$$
By \citep[Proposition 9.3.9]{AGS2005}, we obtain that $\F$ fulfills Assumption~\ref{ass:F}. 
\item[-] We can also derive a functional, be reversing the arguments in the KL divergence. Then, we arrive at the \emph{reverse KL loss function} given by
$$
\F(\mu)=\mathrm{KL}(\mu , \nu)\coloneqq
\begin{cases}
\int_{\R^d} p(x)\log\left(\frac{p(x)}{q(x)}\right)\d x,&\text{if }\exists \mathrm{d}\mu/\mathrm{d}\mu_{\mathrm{Leb}} =p \text{ and } \mathrm{d}\nu/\mathrm{d}\mu_{\mathrm{Leb}} = q,\\
+\infty&\text{otherwise.}
\end{cases}
$$
Given that $-\log(q)$ is $\lambda$-convex, we obtain that $\F$ fulfills Assumption~\ref{ass:F}, see \citep[Proposition 9.3.2]{AGS2005}.
\item[-] Finally, we can define $\F$ based on the Jensen-Shannon divergence. This results into the function
$$
\F(\mu)=\mathrm{JS}(\mu,\nu)\coloneqq \frac12 \left[\mathrm{KL}\left(\mu,\tfrac12(\mu+\nu)\right)+\mathrm{KL}\left(\nu,\tfrac12(\mu+\nu)\right) \right].
$$
Assume $\mu$ and $\nu$ admit Lebesgue densities $p$ and $q$ respectively. Then, combining the two previous statements, this fulfills Assumption~\ref{ass:F} whenever $-\log(p)$ and $-\log(q)$ are $\lambda$-convex.
\end{itemize}

All of these functionals are integrals of a smooth Lagrangian functional, i.e., there exists some  smooth $F\colon\R^d\times\R\times\R^d\to\R$ such that
\begin{equation}\label{eq:lagrangian}
\F(\mu)=
\begin{cases}
\int_{\R^d}F(x,p(x),\nabla p(x))\d x,&\text{if }\exists \mathrm{d}\mu/\mathrm{d}\mu_{\mathrm{Leb}} =p,\\
\infty&\text{otherwise.}
\end{cases}
\end{equation}
In this case, the limit velocity field from Theorem~\ref{thm:existence_gflows_ggd} (which appears as a limit in Theorem~\ref{thm:convergence}) can be expressed analytically as the gradient of the so-called variational derivative of $\F$, which is given by
$$
\frac{\delta \F}{\delta\gamma(t)}(x)=-\nabla\cdot \partial_3 F(x,p(x),\nabla p(x))+\partial_2 F(x,p(x),\nabla p(x))
$$
where $\partial_i F$ is the derivative of $F$ with respect to the $i$-th argument and $\gamma$ is the Wasserstein gradient flow, see \citep[Example 11.1.2]{AGS2005}.
For the above divergence functionals, computing these terms lead to a (weighted) difference of the Stein scores of the input measure $\gamma(t)$ and the target measure $\nu$ which is a nice link to score-based methods. More precisely, denoting the density of $\gamma(t)$ by $p_t$, we obtain the following limiting velocity fields.
\begin{itemize}
    \item[-] For the forward KL loss function we have that $F(x,y,z)=q(x)\log\left(\frac{q(x)}{y}\right)$.
Thus, we have that $\frac{\delta \F}{\delta\mu}(x)=-\frac{q(x)}{p(x)}$.
Hence, the velocity field $v(\cdot,t)=\nabla \frac{\delta \F}{\delta\gamma(t)}$ is given by
$$
v(x,t)=\frac{q(x)}{p_t(x)}\frac{\nabla p_t(x)}{p_t(x)}-\frac{q(x)}{p_t(x)}\frac{\nabla q(x)}{q(x)}=\frac{q(x)}{p_t(x)}\left(\nabla\log (p_t(x))-\nabla \log (q(x))\right).
$$
\item[-] For the reverse KL loss function the Lagrangian is given by $F(x,y,z)=y\log\left(\frac{y}{q(x)}\right)$.
Thus, we have that
$
\frac{\delta \F}{\delta\mu}(x)=\log\left(\frac{p(x)}{q(x)}\right)+1=\log(p(x))-\log(q(x))+1
$.
Hence, the velocity field $v(\cdot,t)=\nabla \frac{\delta \F}{\delta\gamma(t)}$ is given by
$$
v(x,t)=\nabla(\log ( p_t)(x))-\nabla(\log ( q)(x)).
$$
\item[-] For the Jensen-Shannon divergence the Lagrangian is given by $F(x,y,z)=\frac12\left(y\log\left(\frac{2y}{y+q(x)}\right)+q(x)\left(\frac{2q(x)}{y+q(x)}\right)\right)$.
Thus, we have that
$
\frac{\delta\F}{\delta\mu}(x)=\frac12\log\left(\frac{p(x)}{p(x)+q(x)}\right).
$
Hence, the velocity field $v(\cdot,t)=\nabla \frac{\delta \F}{\delta\gamma(t)}$ is given by
\begin{align}
v(x,t)&=\frac12\frac{p_t(x)+q(x)}{p_t(x)}\frac{\nabla p_t(x) (p_t(x)+q(x))-p_t(x)(\nabla p_t(x)+\nabla q(x))}{(p_t(x)+q(x))^2}\\
&=\frac12\frac{\nabla p_t(x)q(x)-p_t(x)\nabla q(x)}{(p_t(x)+q(x))p_t(x)} \\
&=\frac{q(x)}{2(p_t(x)+q(x))}\left[\frac{\nabla p_t(x)}{p_t(x)}-\frac{\nabla q(x)}{q(x)}\right]\\
&=\frac{q(x)}{2(p_t(x)+q(x))}\left[\nabla(\log ( p_t)(x))-\nabla (\log ( q)(x))\right].
\end{align}
\end{itemize}
Note that computing the score $\nabla \log (p_t)$ of the current approximation is usually intractable, such that these limits cannot be inserted into a neural ODE directly.

\subsection{Proof of Corollary~\ref{cor:geodesic}}\label{proof:geodesic}
Using similar arguments as in \cite{AHS2023,MKLGS2021,OFLR2021,XCX2024}, let $v\colon\R^d\times[0,\tau]\to\R^d$ such that the solution of $$\dot z(x,t)=v(z(x,t),t),\quad  z(x,0)=x,$$ fulfills $z(\cdot,\tau)_\#\mu_\tau^k=\mu_\tau^{k+1}$.
Since $(v_{\tau,k},z_{\tau,k})$ is a minimizer of \eqref{eq:dynamic_JKO}, we obtain that
\begin{align}
\phantom{.}&\phantom{=}\,\,\,\frac12\int_0^\tau \int_{\R^d} \|v_{\tau,k}(z_{\tau,k}(x,t),t)\|^2\d \mu_\tau^k(x)\d t+\F(z_{\tau,k}(\cdot,\tau)_\#\mu_\tau^k)\\
&\leq\frac12\int_0^\tau \int_{\R^d} \|v(z(x,t),t)\|^2\d \mu_\tau^k(x)\d t+\F(z(\cdot,\tau)_\#\mu_\tau^k).
\end{align}
Observing that $\F(z_{\tau,k}(\cdot,\tau)_\#\mu_\tau^k)=\F(z(\cdot,\tau)_\#\mu_\tau^k)=\F(\mu_{\tau}^{k+1})$, we obtain that
$$
\tau\int_0^\tau \int_{\R^d} \|v_{\tau,k}(z_{\tau,k}(x,t),t)\|^2\d \mu_\tau^k(x)\leq\tau\int_0^\tau \int_{\R^d} \|v(z(x,t),t)\|^2\d \mu_\tau^k(x)\d t.
$$
Since $v$ was chose arbitrary, we obtain that $v_{\tau,k}$ is the optimal velocity field from the theorem of Benamou-Brenier, which now directly implies part (ii) and (iii).\hfill$\Box$

\subsection{Proof of Theorem~\ref{thm:convergence}}\label{proof:convergence}
In order to prove convergence of the velocity fields $v_\tau$, we first introduce some notations. To this end, let $\mathcal{T}_\tau^k$ be the optimal transport maps between $\mu_\tau^k$ and $\mu_\tau^{k+1}$ and define by $v_\tau^k=(\mathcal{T}_\tau^k-I)/\tau$ the corresponding discrete velocity fields. 
Then, the velocity fields $v_\tau$ can be expressed as
\begin{equation}\label{eq:interpolated_velocity_field}
v_\tau(x,k\tau+t\tau)=v_\tau^k(((1-t)I+t\mathcal{T}_\tau^k)^{-1}(x)).
\end{equation}
Further, we denote the piece-wise constant concatenation of the discrete velocity fields by
$$
\tilde v_\tau(x,k\tau+t\tau)=v_\tau^k,\quad t\in(0,1).
$$

Note that for any $\tau$ it holds that $v_\tau\in L_2(\gamma_\tau,\R^d\times[0,T])$ and $\tilde v_\tau\in L_2(\tilde \gamma_\tau,\R^d\times[0,T])$.
To derive limits of these velocity fields, we recall the notion of convergence from \citep[Definition 5.4.3]{AGS2005} allowing that the iterates are not defined on the same space. In this paper, we stick to square integrable measurable functions defined on finite dimensional domains, which slightly simplifies the definition.
\begin{definition}\label{def:conv}
Let $\Omega\subseteq\R^d$ be a measurable domain, assume that $\mu_k\in\P_2(\Omega)$ converges weakly to $\mu\in\P_2(\Omega)$ and let $f_k\in L_2(\mu_k,\Omega)$ and $f\in L_2(\mu,\Omega)$. Then, we say that $f_k$ converges weakly to $f$, if
$$
\int_\Omega \langle \phi(x), f_k(x)\rangle \d\mu_k(x)\to \int_\Omega \langle \phi(x), f(x)\rangle \d\mu(x),\quad\text{as }k\to\infty
$$
for all test functions $\phi\in C_c^\infty(\Omega)$.
We say that $f_k$ converges strongly to $f$ if
\begin{equation}\label{eq:norm_inequality}
\limsup\limits_{k\to\infty}\|f_k\|_{L_2(\mu_k,\Omega)}\leq \|f\|_{L_2(\mu,\Omega)}.
\end{equation}
\end{definition}
Note that \citep[Theorem 5.4.4 (iii)]{AGS2005} implies that formula \eqref{eq:norm_inequality} is fulfilled with equality for any strongly convergent sequence $f_k$ to $f$.
Moreover it is known from the literature that subsequences of the piece-wise constant velocity admit weak limits.

\begin{theorem}[{\citealp[Theorem~11.1.6]{AGS2005}}]\label{thm:relaxed_GF}
Suppose that Assumption~\ref{ass:F} is fulfilled for $\mathcal{F}$. Then, for any $\mu^0\in \dom(\F)$ and any sequence $(\tau_l)_l\subset (0,\infty)$, there exists a subsequence (again denoted by $(\tau_l)_l$) such that
\begin{itemize}
\item[-] The piece-wise constant curve $\tilde \gamma_{\tau_l}(t)$ narrowly converges to some limit curve $\hat \gamma(t)$ for all $t\in[0,\infty)$.
\item[-] The velocity field $\tilde v_{\tau_l}\in L_2(\tilde \gamma_{\tau_l}(t),\R^d\times[0,T])$ weakly converges to some limit $\hat v\in L_2(\hat\gamma(t),\R^d\times[0,T])$ according to Definition~\ref{def:conv} for any $T>0$.
\item[-] The limit $\hat v$ fulfills the continuity equation with respect to $\hat \gamma$, i.e.,
$$
\partial_t\hat\gamma(t)+\nabla\cdot(\hat v_{\tau_l}(\cdot,t)\hat \gamma(t))=0.
$$
\end{itemize}
\end{theorem}

In order to show the desired result, there remain the following questions, which we answer in the rest of this section:
\begin{itemize}
\item[-] Does Theorem~\ref{thm:relaxed_GF} also hold for the velocity fields $v_{\tau_l}$ defined in \eqref{eq:interpolated_velocity_field} belonging to the geodesic interpolations?
\item[-] Can we show strong convergence for the whole sequence $v_{\tau_l}$?
\item[-] Does the limit $\hat v$ have the norm-minimizing property that $\|\hat v(\cdot,t)\|_{L_2(\hat \gamma(t))}=|\partial\F|(\hat \gamma(t))$?
\end{itemize}

To address the first of these questions, we show that weak limits of $\tilde v_\tau$ and $v_\tau$ coincide. The proof is a straightforward computation. A similar statement in a slightly different setting was proven in \citep[Section 8.3]{S2015}.

\begin{lemma}\label{lem:weak_limits_coincide}
Suppose that Assumption~\ref{ass:F} is fulfilled and let $(\tau_l)_l\subseteq (0,\infty)$ be a sequence with $\tau_l\to 0$ as $l\to\infty$ such that $\tilde v_{\tau_l}\in L_2(\tilde \gamma_{\tau_l},\R^d\times[0,T])$ converges weakly to some $\hat v\in L_2(\hat \gamma,\R^d\times[0,T])$. Then, also $v_{\tau_l}\in L_2(\gamma_{\tau_l},\R^d\times[0,T])$ converges weakly to $\hat v$.
\end{lemma}
\begin{proof}
Let $\phi\in C_c^\infty(\R^d\times[0,T])$. In particular, $\phi$ is Lipschitz continuous with some Lipschitz constant $L<\infty$.
Then, it holds that
\begin{align*}
\phantom{.}&\phantom{\leq}\hspace{0.38em}\left|\int_{\R^d\times[0,T]} \langle\phi,v_{\tau_l}\rangle \d\gamma_{\tau_l}-\int_{\R^d\times[0,T]} \langle\phi,v\rangle \d\hat\gamma\right|\\
&\leq \left|\int_{\R^d\times[0,T]} \langle\phi,v_{\tau_l}\rangle \d\gamma_{\tau_l}-\int_{\R^d\times[0,T]} \langle\phi,\tilde v_{\tau_l}\rangle \d\tilde\gamma_{\tau_l}\right|+\left|\int_{\R^d\times[0,T]} \langle\phi,\tilde v_{\tau_l}\rangle \d\tilde \gamma_{\tau_l}-\int_{\R^d\times[0,T]} \langle\phi,v\rangle \d\hat\gamma\right|.
\end{align*}
Since $\tilde v_{\tau_l}$ converges weakly to $\hat v$, the second term converges to zero as $l\to\infty$. Thus in order to show that also $v_{\tau_l}$ converges weakly to $\hat v$, it remains to show that also the first term converges to zero.
By using the notations $T_l=\min\{k\tau_l:k\tau_l\geq T,k\in\Z_{\geq0}\}$ and $K_l=\lceil\tfrac{T}{\tau_l}\rceil=\tfrac{T_l}{\tau_l}$, we can estimate
\begin{align*}
&\phantom{=} \hspace{0.38em}\left|\int_{\R^d\times[0,T]} \langle\phi,v_{\tau_l}\rangle \d\gamma_{\tau_l}-\int_{\R^d\times[0,T]} \langle\phi,\tilde v_{\tau_l}\rangle \d\tilde\gamma_{\tau_l}\right|\\
&=\left|\int_0^T \left(\int_{\R^d} \langle\phi(x,t),v_{\tau_l}(x,t)\rangle \d\gamma_{\tau_l}(t)(x)-\int_{\R^d} \langle\phi(x,t),\tilde v_{\tau_l}(x,t)\rangle \d\tilde\gamma_{\tau_l}(t)(x)\right)\d t\right|\\
&
\leq \int_0^{T_l}\left|\int_{\R^d} \langle\phi(x,t),v_{\tau_l}(x,t)\rangle \d\gamma_{\tau_l}(t)(x)-\int_{\R^d} \langle\phi(x,t),\tilde v_{\tau_l}(x,t)\rangle \d\tilde\gamma_{\tau_l}(t)(x)\right|\d t 
\\
&\leq \sum_{k=0}^{K_l-1}\int_{k\tau_l}^{(k+1)\tau_l}
\left|
\int_{\R^d} \langle\phi(x,t),v_{\tau_l}(x,t)\rangle \d\gamma_{\tau_l}(t)(x)-\int_{\R^d} \langle\phi(x,t),\tilde v_{\tau_l}(x,t)\rangle \d\tilde\gamma_{\tau_l}(t)(x)\right|\d t\\
&=\tau_l\sum_{k=0}^{K_l-1}\int_{0}^{1}
\left|\int_{\R^d} \langle\phi(x,k\tau_l+t\tau_l),v_{\tau_l}(x,k\tau_l+t\tau_l)\rangle \d\gamma_{\tau_l}(k\tau_l+t\tau_l)(x)\right.\\
&\qquad \qquad \quad \left. -\int_{\R^d} \langle\phi(x,k\tau_l+t\tau_l),\tilde v_{\tau_l}(x,k\tau_l+t\tau_l)\rangle \d\tilde\gamma_{\tau_l}(k\tau_l+t\tau_l)(x)\right|\d t.
\end{align*}
By denoting with $\mathcal{T}_{\tau_l}^k$ the optimal transport map from $\mu_{\tau_l}^k$ to $\mu_{\tau_l}^{k+1}$, we have for $t\in(0,1)$ that 
$$\gamma_{\tau_l}(k\tau_l+t\tau_l)=((1-t)I+t\mathcal{T}_{\tau_l}^k)_\#\mu_{\tau_l}^k,\quad  \tilde \gamma_{\tau_l}(k\tau_l+t\tau_l)=\mu_{\tau_l}^k,$$
and the velocity fields satisfy 
$$v_{\tau_l}(x,k\tau_l+t\tau_l)=v_{\tau_l}^k(((1-t)I+t\mathcal{T}_{\tau_l}^k)^{-1}(x)), \quad \tilde v_{\tau_l}(x,k\tau_l+t\tau_l)=v_{\tau_l}^k(x).$$ 
Then, the above term becomes
\begin{align*}
&\quad\tau_l\sum_{k=0}^{K_l-1}\int_{0}^{1}\left|\int_{\R^d} \langle\phi(x,k\tau_l+t\tau_l),v_{\tau_l}^k(((1-t)I+t\mathcal{T}_{\tau_l}^k)^{-1}(x))\rangle \d((1-t)I+t\mathcal{T}_{\tau_l}^k)_\#\mu_{\tau_l}^k(x)\right.\\
&\qquad \qquad \quad \left.-\int_{\R^d} \langle\phi(x,k\tau_l+t\tau_l),v_{\tau_l}^k(x)\rangle \d\mu_{\tau_l}^k(x)\right|\d t\\
&=\tau_l\sum_{k=0}^{K_l-1}\int_{0}^{1}\left|\int_{\R^d} \langle\phi((1-t)x+t\mathcal{T}_{\tau_l}^k(x)),k\tau_l+t\tau_l),v_{\tau_l}^k(x)\rangle \d\mu_{\tau_l}^k(x)\right.\\
&\qquad \qquad \quad \left.-\int_{\R^d} \langle\phi(x,k\tau_l+t\tau_l),v_{\tau_l}^k(x)\rangle \d\mu_{\tau_l}^k(x)\right|\d t\\
&\leq\tau_l\sum_{k=0}^{K_l-1}\int_{0}^{1}\int_{\R^d} \left|\langle\phi((1-t)x+t\mathcal{T}_{\tau_l}^k(x)),k\tau_l+t\tau_l)-\phi(x,k\tau_l+t\tau_l),v_{\tau_l}^k(x)\rangle\right| \d\mu_{\tau_l}^k(x)\d t.
\end{align*}
Using Hölders' inequality and  we obtain that this is smaller or equal than
{\scriptsize
\begin{align*}
\tau_l\sum_{k=0}^{K_l-1}\int_{0}^{1}\left(\int_{\R^d} \|\phi((1-t)x+t\mathcal{T}_{\tau_l}^k(x)),k\tau_l+t\tau_l)-\phi(x,k\tau_l+t\tau_l)\|^2\d \mu_{\tau_l}^k(x) \int_{\R^d}\|v_{\tau_l}^k(x)\|^2 \d\mu_{\tau_l}^k(x)\right)^{1/2}\d t
\end{align*}}
By the Lipschitz continuity of $\phi$ and inserting the definition of $v_{\tau_l}^k=\frac{\mathcal{T}_{\tau_l}^k-I}{\tau_l}$ this is smaller or equal than
\begin{align}
\phantom{.}&\phantom{=}\hspace{0.38em}\tau_l\sum_{k=0}^{K_l-1}\int_{0}^1\left(t^2\frac{L^2}{\tau_l^2}\int_{\R^d} \|\mathcal{T}_{\tau_l}^k(x)-x\|^2\d \mu_{\tau_l}^k(x) \int_{\R^d}\|\mathcal{T}_{\tau_l}^k(x)-x\|^2 \d\mu_{\tau_l}^k(x)\right)^{1/2}\d t\\
&=L\sum_{k=0}^{K_l-1}\left(\int_0^1t\d t \right)\left(\int_{\R^d}\|\mathcal{T}_{\tau_l}^k(x)-x\|^2 \d\mu_{\tau_l}^k(x)\right)\\
&=\frac{L}{2}\sum_{k=0}^{K_l-1}W_2^2(\mu_{\tau_l}^k,\mu_{\tau_l}^{k+1}) \\
&\leq \frac{L}{2}\sum_{k=0}^{\infty}W_2^2(\mu_{\tau_l}^k,\mu_{\tau_l}^{k+1})\label{eq:estimate}
\end{align}
Finally, we have by the definition of the minimizing movements scheme that
$$
\frac{1}{2\tau_l}W_2^2(\mu_{\tau_l}^k,\mu_{\tau_l}^{k+1})\leq \F(\mu_{\tau_l}^k)-\F(\mu_{\tau_l}^{k+1}).
$$
Summing up for $k=0,1,...$ we finally arrive at the bound
$$
\left|\int_{\R^d\times[0,T]} \langle\phi,v_{\tau_l}\rangle \d\gamma_{\tau_l}-\int_{\R^d\times[0,T]} \langle\phi,v\rangle \d\hat\gamma\right| \leq 
\tau_lL\left(\F(\mu^0)-\inf_{\mu\in\P_2(\R^d)}\F(\mu)\right).
$$
Since $\F$ is bounded from below the upper bound converges to zero as $l \to \infty$.
This concludes the proof.
\end{proof}

Finally, we employ the previous results to show Theorem~\ref{thm:convergence} from the main part of the paper. That is, we show that for any $(\tau_l)_l\subseteq(0,\infty)$ with $\tau_l\to 0$ the whole  $v_{\tau_l}$ converges strongly to $v$.

\begin{theorem}[Theorem~\ref{thm:convergence}]
Suppose that Assumption~\ref{ass:F} is fulfilled and
let $(\tau_l)_l\subseteq(0,\infty)$ with $\tau_l\to 0$. Then, $(v_{\tau_l})_l$ converges strongly to the velocity field $\hat v\in L_2(\gamma,\R^d\times[0,T])$ of the Wasserstein gradient flow $\gamma\colon(0,\infty)\to\P_2(\R^d)$ of $\F$ starting in $\mu^0$.
\end{theorem}

\begin{proof}
We show that any subsequence of $\tau_l$ admits a subsequence converging strongly to $\hat v$. Using the sub-subsequence criterion this yields the claim.

By Theorem~\ref{thm:relaxed_GF} and Lemma~\ref{lem:weak_limits_coincide} we know that any subsequence of $\tau_l$ admits a weakly convergent subsequence. In an abuse of notations, we denote it again by $\tau_l$ and its limit by $\tilde v$. Then, we prove that the convergence is indeed strong and that $\tilde v=\hat v$.

\noindent
\textbf{Step 1: Bounding $\limsup\limits_{l\to\infty}\int_0^T\|v_{\tau_l}(\cdot,t)\|^2_{L_2(\gamma_{\tau_l}(t),\R^d)}\d t$ from above.}
Since $\lambda$-convexity with $\lambda\geq 0$ implies $\lambda$-convexity with $\lambda=-1$, we can assume without loss of generality that $\lambda<0$.
Then, by \citep[Lemma 9.2.7, Theorem 4.0.9]{AGS2005}, we know that for any $\tau>0$ it holds
$$
W_2(\tilde \gamma_{\tau}(t),\gamma(t))\leq \tau C(\tau,t),\qquad C(\tau,t)=\frac{(1+2|\lambda|t_{\tau})|\partial\F|(\mu_0)}{\sqrt{2}(1+\lambda\tau)}\exp\left(-\frac{\log(1+\lambda\tau)}{\tau}t\right),
$$
where $t_\tau=\min\{k\tau: k\tau\geq t, k\in\mathbb{N}_0\}$.
For simplicity, we use the notations $\tau_{\max}=\max\{\tau_l:l\in\mathbb{N}_0\}$, $K_l=\lceil\frac{T}{\tau_l}\rceil$ and $T_{\max}=T+\tau_{\max}$.
Since $\lambda<0$, we have that $C(\tau,t)\leq C(\tau,T_{\max})\eqqcolon C(\tau)$ for all $t\in[0,T_{\max}]$. Moreover, we have that $C(\tau)\to \frac{(1+2|\lambda|T_{\mathrm{max}}|\partial\F|(\mu_0)}{\sqrt{2}}\exp(-\lambda T_{\max})$ as $\tau\to0$, such that the sequence $(C(\tau_l))_l$ is bounded. In particular, there exists a $C>0$ such that
$$
W_2(\tilde \gamma_{\tau_l}(t),\gamma(t))\leq \tau_l C,\quad\text{for all}\quad t\in[0,T_{\max}].
$$
Inserting $t=k\tau_l$ for $k=0,...,K_l$ gives
\begin{equation}\label{eq:W2_estimate}
W_2(\mu_{\tau_l}^k,\gamma(k\tau_l))\leq \tau_l C,\quad\text{for all}\quad k=0,...,K_l.
\end{equation}
Now, we can conclude by the theorem of Benamou-Brenier that
\begin{align}
\int_0^T\|v_{\tau_l}(\cdot,t)\|^2_{L_2(\gamma_{\tau_l}(t),\R^d)}\d t &\leq \int_0^{K_l\tau_l}\|v_{\tau_l}(\cdot,t)\|^2_{L_2(\gamma_{\tau_l}(t),\R^d)}\d t\\
&=\sum_{k=0}^{K_l-1}\int_{k\tau_l}^{(k+1)\tau_l}\|v_{\tau_l}(\cdot,t)\|^2_{L_2(\gamma_{\tau_l}(t),\R^d)}\d t \\
&=\sum_{k=1}^{K_l-1} W_2^2(\mu_{\tau_l}^k,\mu_{\tau_l}^{k+1}).
\end{align}
Now applying the triangular inequality for any $k=1,\ldots, K_l-1$ 
$$
W_2^2(\mu_{\tau_l}^k,\mu_{\tau_l}^{k+1}) \leq \left(W_2(\mu_{\tau_l}^k,\gamma(\tau_lk))+W_2(\gamma(\tau_lk),\gamma(\tau_l(k+1))+W_2(\gamma(\tau_l(k+1)),\mu_{\tau_l}^{k+1})\right)^2
$$
and the estimate from \eqref{eq:W2_estimate} yields 
\begin{align*}
    \int_0^T\|v_{\tau_l}(\cdot,t)\|^2_{L_2(\gamma_{\tau_l}(t),\R^d)}\d t &\leq 
    \sum_{k=0}^{K_l-1}\left(2\tau C +W_2(\gamma(\tau_lk),\gamma(\tau_l(k+1))\right)^2
\end{align*}
By Jensens' inequality the right hand side can be bounded from above by
$$
4K_l\tau^2C^2+2\tau K_l C \left(\sum_{k=0}^{K_l-1}\frac{1}{K_l}W_2^2(\gamma(\tau_lk),\gamma(\tau_l(k+1)))\right)^{1/2}+\sum_{k=0}^{K_l-1}W_2^2(\gamma(\tau_lk),\gamma(\tau_l(k+1)))
$$
Finally, again Benamou-Brenier gives that $W_2^2(\gamma(\tau_lk),\gamma(\tau_l(k+1)))\leq \int_{k\tau_l}^{(k+1)\tau_l}\|\hat v(\cdot,t)\|^2_{L_2(\gamma(t))}\d t$ such that the above formula is smaller or equal than
$$
4K_l\tau_l^2C^2+2\tau_l \sqrt{K_l} C \left(\int_{0}^{K_l\tau_l}\|\hat v(\cdot,t)\|^2_{L_2(\gamma(t))}\d t\right)^{1/2}+\int_{0}^{K_l\tau_l}\|\hat v(\cdot,t)\|^2_{L_2(\gamma(t))}\d t.
$$

Since $K_l\tau_l\to T$ and $\tau_l\to 0$ as $l\to\infty$ this converges to $\int_{0}^{T}\|\hat v(\cdot,t)\|^2_{L_2(\gamma(t))}\d t$ such that we can conclude
$$
\limsup\limits_{l\to\infty} \int_0^T\|v_{\tau_l}(\cdot,t)\|^2_{L_2(\gamma_{\tau_l}(t),\R^d)}\d t\leq \int_{0}^{T}\|\hat v(\cdot,t)\|^2_{L_2(\gamma(t))}\d t.
$$

\noindent
\textbf{Step 2: Strong convergence.}
By \citep[Theorem 8.3.1, Proposition 8.4.5]{AGS2005} we know that for any $v$ fulfilling the continuity equation 
$$
\partial_t\gamma(t)+\nabla\cdot(v(\cdot,t)\gamma(t))=0,
$$
it holds that $\|v(\cdot,t)\|_{L_2(\gamma(t),\R^d}\geq \|\hat v(\cdot,t)\|_{L_2(\gamma(t),\R^d}$.
Since $\tilde v$ fulfills the continuity equation by Theorem~\ref{thm:relaxed_GF}, this implies that
$$
\int_0^T\|\tilde v(\cdot,t)\|^2_{L_2(\gamma(t),\R^d}\d t \geq \int_0^T\|\hat v(\cdot,t)\|^2_{L_2(\gamma(t),\R^d}\d t=\limsup\limits_{l\to\infty} \int_0^T\|v_{\tau_l}(\cdot,t)\|^2_{L_2(\gamma_{\tau_l}(t),\R^d)}\d t.
$$
In particular, $v_{\tau_l}\in L_2(\gamma_{\tau_l},\R^d\times[0,T])$ converges strongly to $\tilde v$ such that by \citep[Theorem 5.4.4 (iii)]{AGS2005} it holds equality in the above equation, i.e.,
$$
\int_0^T\|\tilde v(\cdot,t)\|^2_{L_2(\gamma(t),\R^d}\d t = \int_0^T\|\hat v(\cdot,t)\|^2_{L_2(\gamma(t),\R^d}\d t
$$
Using again \citep[Theorem 8.3.1, Proposition 8.4.5]{AGS2005} this implies that $\tilde v=\hat v$.
\end{proof}

\section{Proofs from Section~\ref{sec:rejection_steps}}

\subsection{Proof of Theorem~\ref{thm:density_propagation_rejection}}\label{app:density_propagation_rejection}
\begin{enumerate}
\item[i)] Let $A \subset \mathbb{R}^d$ Lebesgue-measurable set. We denote by accept and reject the measurable set of accepted and rejected draws of $X$ respectively.
By the law of total probability, we have
$$
\mathbb{P}(\tilde X \in A) = \mathbb{P}(\tilde{X}\in A \text{ and } X \text{ was accepted}) + \mathbb{P}(\tilde{X} \in A \text{ and } X \text{ was rejected}).
$$
Since it holds $\tilde X=X$ if $X$ is accepted and $\tilde X=X'$ if $X$ is rejected this can be reformulated as
$$
\mathbb{P}(\tilde{X} \in A) = \mathbb{P}(X \in A, \text{accept}) + \mathbb{P}(X' \in A, \text{reject}) 
$$
Now it holds 
\begin{align*}
\mathbb{P}(X \in A, \text{accept}) &= \int_A \alpha(x) p(x) \, dx \\
\mathbb{P}(X' \in A,\text{reject}) &= \mathbb{P}(X' \in A) \cdot \mathbb{P}(\text{reject}) 
= \left(\int_A p(x) \, dx\right) \left(1 - \mathbb{E}[\alpha(X)]\right),
\end{align*}
where we used the fact that $X,X'\sim\mu$ are independent. As a consequence, we obtain

\begin{align*}
\mathbb{P}(\tilde X \in A)&= \int_A \alpha(x) p(x) dx + \left( \int_A p(x) dx \right) \left( 1 - \mathbb{E}[\alpha(X)] \right) \\
&= \int_A p(x) \left( \alpha(x) + 1 - \mathbb{E}[\alpha(X)] \right) dx.
\end{align*}
Thus, $\tilde{p}(x) = p(x)(\alpha(x) + 1 - \mathbb{E}[\alpha(X)])$ is the density of $\tilde{X}$.

\item[(ii)]
Note, since $X\sim p$, we have 
$$
\int_{\mathbb{R}^d} \frac{p(x)\alpha(x)}{\E[\alpha(X)]} \mathrm{d}x = \int_{\mathbb{R}^d} \frac{p(x)\alpha(x)}{\int_{\R^d} p(y)\alpha(y)\d y} \mathrm{d}x  = 1, 
$$
in particular $\frac{p(x)\alpha(x)}{\E[\alpha(X)]}$ defines a density. Now let $\eta\in \P_2(\R^d)$ be the corresponding probability measure.
Then, it holds by part (i) that $\tilde \mu=\E[\alpha(X)]\eta+(1-\E[\alpha(X)])\mu$.
We will show that $\mathrm{KL}(\eta,\nu)\leq \mathrm{KL}(\mu,\nu)$.
Due to the convexity of the KL divergence in the linear space of measures this implies the claim.

We denote $Z=\E[\alpha(X)]=\int_{\mathbb{R}^d} \min(p(x),\tfrac{q(x)}{\tilde c})\d x$. Then it holds
\begin{align*}
\mathrm{KL}(\eta,\nu)&=\int_{\mathbb{R}^d}\frac{p(x)\alpha(x)}{Z}\log\left(\frac{p(x)\alpha(x)}{Zq(x)}\right)\d x\\
&=\int_{\mathbb{R}^d}\frac{\min(p(x),\frac{q(x)}{\tilde c})}{Z}\log\left(\frac{\tilde c p(x)\alpha(x)}{q(x)}\right)\d x-\log(Z\tilde c)\int_{\mathbb{R}^d}\frac{\min(p(x),\frac{q(x)}{\tilde c})}{Z}\d x\\
&=\int_{\mathbb{R}^d}\frac{\min(p(x),\frac{q(x)}{\tilde c})}{Z}\log\left(\min\left(\frac{\tilde c p(x)}{q(x)},1\right)\right)\d x-\log(Z)-\log(\tilde c)\\
&=\int_{\mathbb{R}^d}\min\left(\frac{\min(p(x),\frac{q(x)}{\tilde c})}{Z}\log\left(\frac{\tilde c p(x)}{q(x)}\right),0\right)\d x-\log(Z)-\log(\tilde c).\\
\end{align*}
Since $\log\left(\frac{\tilde c p(x)}{q(x)}\right)\leq 0$ if and only if $p(x)\leq q(x)/\tilde c$ this can be reformulated as
\begin{align}
&\quad\int_{\mathbb{R}^d} \min\left(\frac{p(x)}{Z}\log\left(\frac{\tilde cp(x)}{q(x)}\right),0\right)\d x-\log(Z)-\log(\tilde c)\\
&\leq\int_{\mathbb{R}^d} \min\left(p(x)\log\left(\frac{\tilde cp(x)}{q(x)}\right),0\right)\d x-\log(Z)-\log(\tilde c),\label{eq:asdf}
\end{align}
where the inequality comes from the fact that $Z=\E[\alpha(X)]\in[0,1]$.
Moreover, it holds by Jensen's inequality that
\begin{align}
-\log(Z)=-\log(\E[\alpha(X)])&=-\log\left(\int_{\mathbb{R}^d}  p(x)\min\left(\frac{q(x)}{\tilde c p(x)},1\right)\d x\right)\\
&\leq-\int_{\mathbb{R}^d}  p(x)\log\left(\min\left(\frac{q(x)}{\tilde c p(x)},1\right)\right)\d x \\
&=-\int_{\mathbb{R}^d}  p(x)\min\left(\log\left(\frac{q(x)}{\tilde c p(x)}\right),0\right)\d x\\
&=\phantom{-}\int_{\mathbb{R}^d}  p(x)\max\left(\log\left(\frac{\tilde c p(x)}{q(x)}\right),0\right)\d x \\
&=\phantom{-}\int_{\mathbb{R}^d}  \max\left(p(x)\log\left(\frac{\tilde c p(x)}{q(x)}\right),0\right)\d x.
\end{align}
Thus, we obtain that \eqref{eq:asdf} can be bounded from above by
$$
\int_{\mathbb{R}^d} \min\left(p(x)\log\left(\frac{\tilde c p(x)}{q(x)}\right),0\right)\d x+\int_{\mathbb{R}^d}  \max\left( p(x)\log\left(\frac{\tilde c p(x)}{q(x)}\right),0\right)\d x-\log(\tilde c)
$$
which equals
$$\int_{\mathbb{R}^d}  p(x)\log\left(\frac{\tilde cp(x)}{q(x)}\right)\d x-\log(\tilde c)
=\int_{\mathbb{R}^d} p(x)\log\left(\frac{p(x)}{q(x)}\right)\d x=\mathrm{KL}(\mu,\nu).
$$

In summary, we have $\mathrm{KL}(\eta,\nu)\leq\mathrm{KL}(\mu,\nu)$, which implies the assertion.
\end{enumerate}
\phantom{.}\hfill$\Box$

\subsection{Proof of Corollary~\ref{cor:concentration}}\label{proof:concentration}
Let $Y=\Big|\E[\alpha(X)]-\frac1N\sum_{i=1}^N \alpha(X_N)\Big|$ denote the random variable representing the error.
Since $\alpha(X_N)\in[0,1]$, Hoeffding's inequality \cite{H1963} yields that
$
P(Y>t)\leq2\exp(-\frac{Nt^2}{2})
$.
Consequently, we have that
$$
\E[Y]=\int_0^\infty P(Y>t)\d t\leq 2\int_0^\infty \exp\left(-\frac{Nt^2}{2}\right)\d t=\frac{\sqrt{2\pi}}{\sqrt{N}}.
$$
\phantom{.}\hfill$\Box$

\section{Algorithms}

\begin{figure}
\begin{center}
\begin{tikzpicture}[
    node distance=1.5cm,
    block/.style={rectangle, draw, text width=8em, text centered, rounded corners, minimum height=4em},
    decision/.style={diamond, draw, text width=5em, text centered, aspect=2, inner sep=0pt},
    line/.style={draw, -latex'},
    cloud/.style={ellipse, draw, text width=3em, text centered, minimum height=2em}
]

\node [cloud] (input) {$x_i^{k}$};
\node [decision, right=1cm of input] (decision) {accepted?};

\node [cloud, below right=1cm and 1cm of decision] (resample) {resample {\scriptsize{$x_i^k\sim\mu^k$}}};

\node [cloud, right=2cm of decision] (output) {$x_i^{k+1}$};

\node [rectangle, left=0.5cm of input] (prevlayer) {\textbf{previous layer}};
\node [below=0.15cm of prevlayer] {\color{gray!50!white}CNF/ rejection};
\node [rectangle, right=1cm of output] (nextlayer) {\textbf{next layer}};
\node [below= 0.15cm of nextlayer] {\color{gray!50!white}CNF/ rejection};

\path [line] (input) -- (decision);
\path [line] (decision) -- node[above] {\color{green!50!black}yes} (output);
\path [line] (decision) -- node[left] {\color{red!75!black}no} (resample);
\path [line] (resample) -- node[right] {\color{green!50!black}yields} (output);
\path [line] (prevlayer) -- (input);
\path [line] (output) -- (nextlayer);

\node [above=2cm of resample] {\textbf{rejection} at $k$-th layer};


\end{tikzpicture}
    \caption{Illustration of a single sample within a rejection layer. The independent sample $x_i^k$ is either accepted directly or else resampled once, and then used to define $x_i^{k+1}$ as input for the next layer.}
    \label{fig:flowchart_rejection}
\end{center}
\end{figure}

\subsection{Training and Evaluation Algorithms}\label{app:algorithms_training_eval}

We summarize the training and evaluation procedures for the neural JKO steps in Algorithm~\ref{alg:training_CNF} and \ref{alg:density_probagation_CNF} and the parameter selection for the importance-based rejection steps in Algorithm~\ref{alg:training_rejection}. The evaluation of the importance-based rejection steps is summarized in Algorithm~\ref{alg:density_probagation_rejection} in the main text.

\begin{figure}
\begin{algorithm}[H]
\begin{algorithmic}
\STATE \textbf{Input:} Samples $x_1^k,...,x_N^k$ of $\mu^k$.
\STATE Minimize the loss function $\theta\mapsto\mathcal{L}(\theta)$ from~ \eqref{eq:loss_neural_JKO} using the Adam optimizer.
\STATE \textbf{Output:} Parameters $\theta$.
\end{algorithmic}
\caption{Training of neural JKO steps}
\label{alg:training_CNF}
\end{algorithm}

\begin{algorithm}[H]
\begin{algorithmic}
\STATE \textbf{Input:} $\left\{ \begin{array}{l}
        \text{- Samples } x_1^k,...,x_N^k \text{ of } \mu^k, \\
        \text{- Density values }p^k(x_1^k),...,p^k(x_N^k).
        \end{array}
        \right.$
\STATE

\FOR{$i=1,...,N$}
\STATE 1. Solve the ODE \eqref{eq:neuralJKO_ODE} for $x=x_i^k$.
\STATE 2. Set $x_i^{k+1}=z_\theta(x_i^k,\tau)$.
\STATE 3. Set $p^{k+1}(x_i^{k+1})=\frac{p^{k}(x_i^k)}{\exp(\ell_\theta(x_i^k,\tau))}$.
\ENDFOR
\STATE
\STATE \textbf{Output:}
$\left\{ \begin{array}{l}
        \text{- Samples } x_1^k,...,x_N^k \text{ of } \mu^{k+1}, \\
        \text{- Density values }p^{k+1}(x_1^{k+1}),...,p^{k+1}(x_N^{k+1}).
        \end{array}
        \right.$
\end{algorithmic}
\caption{Sampling and density propagation for neural JKO steps}
\label{alg:density_probagation_CNF}
\end{algorithm}

\begin{algorithm}[H]
\begin{algorithmic}
\STATE \textbf{Input:} 
$\left\{ \begin{array}{l}\text{- Samples $x_1^k,...,x_N^k$ of $\mu^k$ with corresponding densities}\\
\text{- desired rejection rate $r$, unnormalized target density $g$}
\end{array}\right.$
\STATE
\STATE Choose $c$ by bisection search such that
$$
1-r=\frac1N\sum_{i=1}^N\alpha_k(x_i^k),\qquad \alpha_k(x)=\min\left\{1,\frac{g(x)}{cp^k(x)}\right\}.
$$

\STATE \textbf{Output:} Rejection parameter $c$ and estimate of $\E[\alpha(X_\tau^k)]\approx 1-r$.
\end{algorithmic}
\caption{Parameter selection for importance-based rejection steps}
\label{alg:training_rejection}
\end{algorithm}

\begin{algorithm}[H]
    \begin{algorithmic}
        \STATE \textbf{Input:} $x\in\R^d$, model $(X_0,...,X_K)$
        \IF{$K=0$}
            \STATE \textbf{Return} latent density $p^0(x)$.
        \ELSIF{the last step is a rejection step}
            \STATE 1. Evaluate $p^{K-1}(x)$ by applying this algorithm for $(X_0,...,X_{K-1})$.
            \STATE 2. \textbf{Return} $p^{K}(x)=p^{K-1}(\alpha_k(x)+1-\E[\alpha_k(X_{k-1})])$
        \ELSE 
            \STATE \hfill \textit{last step is a neural JKO step}
            \STATE 1. Solve the ODE system (with $v_\theta$ from the neural JKO step)
            $$
            \left(\begin{array}{c}\dot z_\theta(x,t)\\\dot \ell_\theta(x,t)\end{array}\right)=\left(\begin{array}{c}v_\theta(z_\theta(x,t),t)\\\mathrm{trace}(\nabla v_\theta(z_\theta(x,t),t))\end{array}\right),\qquad \left(\begin{array}{c}z_\theta(x,\tau)\\\ell_\theta(x,\tau)\end{array}\right)=\left(\begin{array}{c}x\\0\end{array}\right).
            $$
            \STATE 2. Set $\tilde x=z_\theta(x,0)$.
            \STATE 3. Evaluate $p^{K-1}(\tilde x)$ by applying this algorithm for $(X_0,...,X_{K-1})$.
            \STATE 4. \textbf{Return} $p^{K}(x)=p^{K-1}(\tilde x)\exp(\ell_\theta(x,0))$.
        \ENDIF
    \STATE \textbf{Output:} Density $p^K(x)$
    \end{algorithmic}
    \caption{Density Evaluation of Importance Corrected Neural JKO Models}
    \label{alg:density_evaluation}
\end{algorithm}
\end{figure}

\subsection{Density Evaluation of Importance Corrected Neural JKO Models}\label{app:density_evaluation}

Let $(X_0,...,X_K)$ be an importance corrected neural JKO model. We aim to evaluate the density $p^K$ of $X_K$ at some given point $x\in\R^d$. Using Theorem~\ref{thm:density_propagation_rejection}, this can be done recursively by Algorithm~\ref{alg:density_evaluation}. Note that the algorithm always terminates since $K$ is reduced by one in each call.

\section{Additional Numerical Results and Implementation Details}\label{app:detailed_experiments}

\subsection{Test Distributions}\label{sec:test_distributions}

We evaluate our method on the following test distributions.
\begin{itemize}
    \item[-] \textbf{Mustache:} 
    The two-dimensional log-density is given as $\log\mathcal{N}(0,\Sigma) \circ T$ with $\Sigma= [1, \sigma ; \sigma,1]$ and $T(x_1,x_2) = (x_1, (x_2-(x_1^2-1)^2))$. Note that $\mathrm{det}(\nabla T(x))=1$ for all $x$. In particular, we obtain directly by the transformation formula that the normalization constant is one.
    Depending on $\sigma\in[0,1)$ close to $1$, this probability distribution has very long and narrow tails making it hard for classical MCMC methods to sample them. In our experiments we use $\sigma=0.9$.
    \item[-] \textbf{Shifted 8 Modes:} A two-dimensional Gaussian mixture model with 8 equal weighted modes and covariance matrix $\num{1e-2} I$. The modes are placed in a circle with radius 1 and center $(-1,0)$. Due to the shifted center classical MCMC methods have difficulties to distribute the mass correctly onto the modes.
    \item[-] \textbf{Shifted 8 Peaky:} This is the same distribution as the shifted 8 Modes with the difference that we reduce the width of the modes to the covariance matrix $\num{5e-3} I$. Since the modes are disconnected, it becomes harder to sample from them.
    \item[-] \textbf{Funnel:} We consider the (normalized) probability density function given by $q(x)=\mathcal N(x_1|0,\sigma_f^2)\mathcal N(x_{2:10}|0,\exp(x_1)I)$, where $\sigma_f^2=9$. This example was introduced by \cite{N2003}. Similarly to the mustache example, this distribution has a narrow funnel for small values $x_1$ which can be hard to sample.
    \item[-] \textbf{GMM-$d$:} A $d$-dimensional Gaussian mixture model with 10 equal weighted modes with covariance matrix $\num{1e-2} I$ and means drawn randomly from a uniform distribution on $[-1,1]^d$. This leads to a peaky high-dimensional and multimodal probability distribution which is hard to sample from.
    \item[-] \textbf{GMM40-50D:} Another Gaussian mixture model, which was used as an example in \citep{BJEVN2024, CRBBNA2025} based on \citep{MSSSH2023}. It is a 50 dimensional GMM with 40 equally weighted modes with the identity as covariance matrix, where the means are drawn uniformly from $[-40,40]^d$. We use the same seed for drawing the modes as in \cite{CRBBNA2025}.
    \item[-] \textbf{LGCP:} This is a high dimensional standard example taken from \cite{AMD2021,MARD2022,VGD2023}. It describes a Log-Gaussian Cox process on a $40\times40$ grid as arising from spatial statistics \cite{MSW1998}. This leads to a $1600$-dimensional probability distribution with the unnormalized density function 
    $q(x)\propto \mathcal N(x|\mu,K)\prod_{i\in\{1,...,40\}^2} \exp(x_iy_i-a\exp(x_i))$, where $\mu$ and $K$ are a fixed mean and covariance kernel, $y_i$ is some data and $a$ is a hyperparameter. For details on this example and the specific parameter choice we refer to \cite{MARD2022}.
\end{itemize}

\subsection{Error Measures} 
We use the following error measures.
\begin{itemize}
    \item[-] The \textbf{energy distance} was proposed by \cite{S2002} and is defined by
    $$
    D(\mu,\nu)=-\frac12\int_{\R^d}\int_{\R^d}\|x-y\|\d(\mu-\nu)(x)\d(\mu-\nu)(y).
    $$
    It is the maximum mean discrepancy with the negative distance kernel $K(x,y)=-\|x-y\|$ (see \citealp{SBGF2013}) and can be estimated from below and above by the Wasserstein-1 distance (see \citealp{HWAH2024}).
    It is a metric on the space of probability measures. Consequently, a smaller energy distance indicates a higher similarity of the input distributions.
    By discretizing the integrals it can be easily evaluated based on $N\in\mathbb{N}$ samples $\bm{x}=(x_i)_i\sim \mu^{\otimes N}$, $\bm{y}=(y_i)_i\sim \nu^{\otimes N}$ as
    $$
    D(\zb x,\zb y)=\sum_{i,j=1}^N\|x_i-y_j\|-\frac12\sum_{i,j=1}^N\|x_i-x_j\|-\frac12\sum_{i,j=1}^N\|y_i-y_j\|.
    $$
    We use $N=50000$ samples in Table~\ref{tab:energy_distance}.
    \item[-] We also evaluate the \text{squared Wasserstein-2 distance} which is defined in Section~\ref{sec:prelim}. To this end, we use the Python Optimal Transport package (POT, \citealp{FCGA2021}). Note that computing the Wasserstein distance has complexity $O(n^3)$ where $n$ is the number of points. Hence, we evaluate the Wasserstein distance based on fewers samples compared to the case of other metrics. We use $N=5000$ samples in Table~\ref{tab:w2_results}. In addition, we want to highlight that 
    the expected Wasserstein distance evaluated on empirical measures instead of its continuous counterpart suffers from the curse of dimensionality. In particular, its  sample complexity scales as $O(n^{-1/d})$ \citep[Chapter 8.4.1]{PC2019}. Consequently, the sample-based Wasserstein distance in high dimensions can differ significantly from the true Wasserstein distance of the continuous distributions. Indeed, we can see in Table~\ref{tab:w2_results} that the sampling error has often the same order of magnitude as the reported errors. Overall we can draw similar conclusions from this evaluation as for the energy distance in Table~\ref{tab:energy_distance}. 
    \item[-] We \textbf{estimate the log normalizing constant} (short $\log(Z)$ estimation) which is used as a benchmark standard in various references (e.g., in \citealp{AMD2021,MARD2022,PDHD2024,VGD2023}). More precisely, for the generated distribution $\mu$ with normalized density $p$ and target measure $\nu$ with density $q(x)=g(x)/Z_g$ we evaluate the term
    $$
    \E_{x\sim\mu}\left[\log\left(\frac{g(x)}{p(x)}\right)\right]=\log(Z_g)-\E_{x\sim\mu}\left[\log\left(\frac{p(x)}{q(x)}\right)\right]=\log(Z_g)-\mathrm{KL}(\mu,\nu).
    $$
    Due to the properties of the KL divergence, a higher $\log(Z)$ estimate implies a lower KL divergence between $\mu$ and $\nu$ and therefore a higher similarity of generated and target distribution. In our experiments we compute the $\log(Z)$ estimate based on $N=50000$ samples. The results are given in Table~\ref{tab:logZ}.
    \item[-] To quantify how well the mass is distributed on different modes for the mixture model examples (shifted 8 Modes, shifted 8 Peaky, GMM-$d$), we compute the mode weights. That is, we generate $N=50000$ samples and assign each generated samples to the closest mode of the GMM. Afterwards, we compute for each mode the fraction of samples which is assigned to each mode. To evaluate this distribution quantitatively, we compute the mean square error (MSE) between the mode weights of the generated samples and the ground truth weights from the GMM. We call this error metric the \textbf{mode MSE}, give the results are in Table~\ref{tab:mode_MSE}.
\end{itemize}

\begin{table}
\caption{We report the expected squared empirical Wasserstein-$2$ distance ($W_2^2$) and its standard deviation between generated and ground truth samples of size $N=5000$ for the different methods and for all examples where the ground truth model is known. A smaller value of $W_2^2$ indicates a better result. Note that the curse of dimensionality present in the sample complexity, might limit the reliability of the results for the high-dimensional examples. In particular for the funnel distribution, we observe that the expected empirical Wasserstein-$2$ distance between two independent sets of ground truth samples is higher than the observed $W_2^2$ values.}
\centering
\scalebox{.55}{
\begin{tabular}{ccllllllll}
\toprule
  & \phantom{.}   &  \multicolumn{6}{c}{Sampler}  &  \phantom{.} & \\
  \cmidrule{3-8}  
Distribution& &\multicolumn{1}{c}{MALA}&\multicolumn{1}{c}{HMC}&\multicolumn{1}{c}{DDS}&\multicolumn{1}{c}{CRAFT}&\multicolumn{1}{c}{Neural JKO}&\multicolumn{1}{c}{Neural JKO IC (\textbf{ours})}& &Sampling Error\\
\midrule
Mustache&&$\num{4.7e+1}\pm\num{1.3e+1}$&$\num{2.8e+1}\pm\num{4.2e0}$&$\num{5.2e+1}\pm\num{2.0e+0}$&$\num{5.4e+1}\pm\num{1.2e+1}$&$\num{3.0e+1}\pm\num{1.3e+1}$&\best$\num{1.7e+1}\pm\num{6.0e+0}$&&$\num{1.2e1}$\\
shifted 8 Modes&&$\num{5.5e-2}\pm\num{6.9e-3}$&$\num{4.7e-3}\pm\num{3.6e-3}$&$\num{8.7e-2}\pm\num{3.1e-2}$&$\num{2.4e-1}\pm\num{2.4e-3}$&$\num{5.6e-1}\pm\num{1.6e-2}$&\best$\num{6.5e-3}\pm\num{2.1e-3}$&&$\num{7.4e-3}$\\
shifted 8 Peaky&&$\num{5.8e-2}\pm\num{2.4e-2}$&$\num{5.6e-1}\pm\num{1.4e-2}$&$\num{1.0e-1}\pm\num{2.5e-2}$&$\num{2.5e-1}\pm\num{1.1e-1}$&$\num{5.9e-1}\pm\num{1.7e-2}$&$\best\num{7.2e-3}\pm\num{1.3e-3}$&&$\num{5.6e-3}$\\
Funnel&&$\num{5.5e+2}\pm\num{1.2e+2}$&$\num{7.4e+2}\pm\num{5.6e+2}$&\best$\num{5.3e+2}\pm\num{7.5e+1}$&$\num{7.9e+2}\pm\num{4.4e+2}$&$\num{9.4e+2}\pm\num{9.3e+2}$&$\num{8.5e+2}\pm\num{3.9e+2}$&&$\num{1.0e3}$\\
GMM-10&&$\num{3.8e+0}\phantom{^-}\pm\num{4.2e-1}$&$\num{3.8e+0}\phantom{^-}\pm\num{3.9e-1}$&$\num{3.8e-1}\pm\num{1.1e-1}$&$\num{2.4e+0}\phantom{^-}\pm\num{1.0e+0}$&$\num{6.3e-1}\pm\num{1.4e-1}$&\best$\num{1.4e-1}\pm\num{2.3e-2}$&&$\num{1.4e-1}$\\
GMM-20&&$\num{8.9e+0}\phantom{^-}\pm\num{4.0e-1}$&$\num{9.0e+0}\phantom{^-}\pm\num{3.9e-1}$&$\num{8.8e-1}\pm\num{1.0e-1}$&$\num{7.9e+0}\phantom{^-}\pm\num{1.5e+0}$&$\num{1.2e-0}\phantom{^-}\pm\num{2.0e-1}$&\best$\num{3.7e-1}\pm\num{5.0e-2}$&&$\num{3.7e-1}$\\
GMM-50&&$\num{2.7e+1}\pm\num{1.0e+0}$&$\num{2.7e+1}\pm\num{9.8e-1}$&$\num{3.6e+0}\phantom{^-}\pm\num{8.9e-1}$&$\num{2.6e+1}\pm\num{2.6e+0}$&$\num{4.4e-0}\phantom{^-}\pm\num{7.9e-1}$&\best$\num{1.2e0}\phantom{^-}\pm\num{1.3e-1}$&&$\num{1.2e0}$\\
GMM-100&&$\num{5.7e+1}\pm\num{1.2e+0}$&$\num{5.7e+1}\pm\num{1.3e+0}$&$\num{9.9e+0}\phantom{^-}\pm\num{2.9e+0}$&$\num{5.6e+1}\pm\num{1.1e0}$&$\num{1.1e+1}\pm\num{2.7e0}$&\best$\num{2.9e0}\phantom{^-}\pm\num{4.7e-1}$&&$\num{2.8e0}$\\
GMM-200&&$\num{1.2e+2}\pm\num{2.8e+0}$&$\num{1.2e+2}\pm\num{2.8e+0}$&$\num{2.4e+1}\pm\num{4.0e+1}$&$\num{1.1e+2}\pm\num{3.0e+0}$&$\num{2.4e1}\pm\num{3.9e0}$&\best$\num{7.8e0}\phantom{^-}\pm\num{6.9e-1}$&&$\num{5.9e-0}$\\
\bottomrule
\end{tabular}
}
\label{tab:w2_results}
\end{table}

\begin{remark}[Bias in $\log(Z)$ Computation]\label{rem:biasa}
In the cases, where the importance corrected neural JKO sampling fits the target distribution almost perfectly, we sometimes report in Table~\ref{tab:logZ} $\log(Z)$ estimates which are slightly larger than the ground truth. This can be explained by the fact that the density evaluation of the continuous normalizing flows uses the Hutchinson trace estimator for evaluating the divergence and a numerical method for solving the ODE. Therefore, we have a small error in the density propagation of the neural JKO steps. This error is amplified by the rejection steps since samples with underestimated density are more likely to be rejected than samples with overestimated density.

We would like to emphasize that this effect only appears for examples, where the energy distance between generated and ground truth samples is in the same order of magnitude like the average distance between two different sets of ground truth samples (see Table~\ref{tab:energy_distance}). This means that in the terms of the energy distance the generated and ground truth distribution are indistinguishable. At the same time the bias in the $\log(Z)$ estimate appears at the third or fourth relevant digit meaning that it is likely to be negligible.
\end{remark}

\begin{table}
\caption{We report the mode MSEs for the different methods for all examples which can be represented as mixture model. A smaller mode MSE indicates a better result.}
\centering
\scalebox{.58}{
\begin{tabular}{ccrrrrrr}
\toprule
  & \phantom{.}   &  \multicolumn{6}{c}{Sampler}  \\
  \cmidrule{3-8}  
Distribution& &\multicolumn{1}{c}{MALA}&\multicolumn{1}{c}{HMC}&\multicolumn{1}{c}{DDS}&\multicolumn{1}{c}{CRAFT}&\multicolumn{1}{c}{Neural JKO}&\multicolumn{1}{c}{Neural JKO IC (\textbf{ours})}\\
\midrule
shifted 8 Modes&&$\num{3.2e-3}\pm\num{1.3e-4}$&$\num{1.9e-5}\pm\num{1.1e-5}$&$\num{8.8e-3}\pm\num{2.2e-3}$&$\num{3.2e-2}\pm\num{6.1e-3}$&$\num{8.3e-2}\pm\num{1.0e-3}$&$\best\num{1.3e-5}\pm\num{4.4e-6}$\\
shifted 8 Peaky&&$\num{8.3e-2}\pm\num{3.7e-4}$&$\num{7.8e-2}\pm\num{9.9e-4}$&$\num{7.9e-3}\pm\num{2.2e-3}$&$\num{3.3e-2}\pm\num{1.3e-2}$&$\num{8.4e-2}\pm\num{6.7e-4}$&$\best\num{1.5e-5}\pm\num{3.2e-6}$\\
GMM-10&&$\num{1.2e-2}\pm\num{5.5e-3}$&$\num{1.0e-2}\pm\num{5.8e-3}$&$\num{3.6e-3}\pm\num{1.9e-3}$&$\num{1.4e-1}\pm\num{4.7e-2}$&$\num{1.1e-2}\pm\num{5.6e-3}$&$\best\num{2.1e-5}\pm\num{3.0e-6}$\\
GMM-20&&$\num{6.6e-3}\pm\num{2.2e-3}$&$\num{6.6e-3}\pm\num{2.4e-3}$&$\num{3.6e-3}\pm\num{1.5e-3}$&$\num{3.8e-1}\pm\num{3.9e-2}$&$\num{7.3e-3}\pm\num{2.8e-3}$&$\best\num{2.4e-5}\pm\num{9.8e-6}$\\
GMM-50&&$\num{1.0e-2}\pm\num{3.8e-3}$&$\num{9.8e-3}\pm\num{3.6e-3}$&$\num{9.6e-3}\pm\num{4.9e-3}$&$\num{9.0e-1}\pm\num{6.4e-5}$&$\num{1.2e-2}\pm\num{4.4e-3}$&$\best\num{2.6e-5}\pm\num{6.0e-6}$\\
GMM-100&&$\num{1.1e-2}\pm\num{5.4e-3}$&$\num{1.1e-2}\pm\num{5.4e-3}$&$\num{1.1e-2}\pm\num{5.8e-3}$&$\num{9.0e-1}\pm\num{4.8e-8}$&$\num{1.4e-2}\pm\num{7.0e-3}$&$\best\num{1.5e-4}\pm\num{8.5e-5}$\\
GMM-200&&$\num{1.4e-2}\pm\num{5.1e-3}$&$\num{1.4e-2}\pm\num{4.6e-3}$&$\num{2.2e-2}\pm\num{8.6e-3}$&$\num{9.0e-1}\pm\num{5.8e-8}$&$\num{1.9e-2}\pm\num{6.3e-3}$&$\best\num{6.8e-4}\pm\num{4.0e-4}$\\
\bottomrule
\end{tabular}
}
\label{tab:mode_MSE}
\end{table}

\subsection{Additional Comparison}

We additionally compare our importance corrected neural JKO model with several other methods on the example of a $50$ dimensional Gaussian mixture model as used in \citep{MSSSH2023}. We follow the setting of Table 3 in \citet{CRBBNA2025} and evaluate the Sinkhorn distance \cite{C2013} with entropic regularization parameter $10^{-3}$ based on $2000$ generated and ground truth samples. Moreover we average the results over $5$ training runs with $25$ evaluations from each run. We compare our results with sequential Monte Carlo (SMC, \citealp{DDJ2006}), denoising diffusion sampler (DDS, \citealp{VGD2023}), continual repeated annealed flow transport Monte Carloe (CRAFT, \citealp{MARD2022}), path integral sampler (PIS, \citealp{ZC2021}), controlled Monte Carlo Diffusion (CMCD, \citealp{VN2023}) and sequential controlled Langevin diffusion (SCLD, \citet{CRBBNA2025}). For CMCD, we consider the log-variance (LV) variant considered in \cite{CRBBNA2025} since it gives better results. The hyperparameters used for our importance corrected neural JKO scheme are given in Table~\ref{tab:hyper}. Similar as most comparison methods, we use a normal distribution with a larger variance as initial distribution $\mu_0$ and use a gradient flow on the negative log-density prior to the learned CNF and rejection steps, see Remark~\ref{rem:gf_init_and_scaling} for details.

The results are given in Table~\ref{tab:lorenz}, where the values for the comparison are taken from \citep{CRBBNA2025}. We observe that neural JKO IC significantly outperforms the comparisons on this example.

\begin{table}
\caption{We report the mode Sinkhorn distance for the different methods for the GMM40-50D example. A smaller mode Sinkhorn distance indicates a better result. The comparison values are taken from \citep{CRBBNA2025}.}
\centering
\scalebox{.58}{
\begin{tabular}{ccccccccc}
\toprule
  & \phantom{.}   &  \multicolumn{6}{c}{Sampler}  \\
  \cmidrule{3-9}  
& &\multicolumn{1}{c}{SMC}&\multicolumn{1}{c}{DDS}&\multicolumn{1}{c}{CRAFT}&\multicolumn{1}{c}{PIS}&\multicolumn{1}{c}{CMCD-LV}&\multicolumn{1}{c}{SCLD}&\multicolumn{1}{c}{Neural JKO IC (\textbf{ours})}\\
\midrule
Sinkhorn distance&&$46370.34\pm137.79$&$5435.18\pm 172.20$&$28960.70\pm354.89$&$10405.75\pm69.41$&$4258.57\pm737.15$&$3787.73\pm249.75$&$\best3025.61\pm418.10$\\
\bottomrule
\end{tabular}
}
\label{tab:lorenz}
\end{table}

\subsection{Additional Figures and Evaluations}

Additionally to the results from the main part of the paper, we provide the following evaluations.

\paragraph{Visualization of the Rejection Steps} In Figure~\ref{fig:death-birth} and with involved steps in Figure~\ref{fig:8modes_2d}, we visualize the different steps of our importance corrected neural JKO model on the shifted 8 Peaky example. Due to the shift of the modes, the modes on the right attract initially more mass than the ones on the left. In the rejection layers, we can see that samples are mainly rejected in oversampled regions such that the mode weights are equalized over time. This can also be seen in Figure~\ref{fig:weight_conv_8modes_2d}, where we plot the weights of the different modes over time. We observe, that these weights are quite imbalanced for the latent distribution but are equalized over the rejection steps until they reach the ground truth value.

\paragraph{(Marginal) Plots of Generated Samples} Despite the quantitative comparison of the methods in the Tables~\ref{tab:energy_distance}, \ref{tab:logZ} and \ref{tab:mode_MSE}, we also plot the first to coordinates of generated samples for the different test distributions for all methods for a visual comparison.

In Figure~\ref{fig:samples_8mixtures}, we plot samples of the shifted 8 Modes example. We can see that all methods roughly cover the ground truth distribution even though CRAFT and DDS have slight and the uncorrected neural JKO scheme has a severe imbalance in the assigned mass for the different modes.

For the shifted 8 Peaky example in Figure~\ref{fig:samples_8peaky}, we see that this imbalance increases drastically for CRAFT, HMC, MALA and the uncorrected neural JKO. Also DDS has has a slight imbalance, while the importance corrected neural JKO scheme fits the ground truth almost perfectly.

The Funnel distribution in Figure~\ref{fig:samples_funnel} has two difficult parts, namely the narrow but high-density funnel on the left and the wide moderate density fan on the right. We can see that DDS does cover none of them very well. Also MALA, CRAFT and the uncorrected neural JKO do not cover the end of the funnel correctly and have also difficulties to cover the fan. HMC covers the fan well, but not the funnel. Only our importance corrected neural JKO scheme covers both parts in a reasonable way.

For the Mustache distribution in Figure~\ref{fig:samples_schnauzbart}, the critical parts are the two heavy but narrow tails. We observe that CRAFT and DDS are not able to cover them at all, while HMC and MALA and uncorrected neural JKO only cover them only partially. The importance corrected neural JKO covers them well.

Finally, for the GMM-$d$ example we consider the dimensions $d=10$ and $d=200$ in the Figures~\ref{fig:samples_mixtures10} and \ref{fig:samples_mixtures200}. We can see that CRAFT mode collapses, i.e., for $d=10$ it already finds some of the modes and for $d=200$ it only finds one mode. DDS, HMC, MALA and the uncorrected neural JKO find all modes but do not distribute the mass correctly onto all modes. While this already appears for $d=10$ it is more severe for $d=200$. The importance corrected neural JKO sampler finds all modes and distributes the mass accurately.

Note that the distributions generated by MALA and the neural JKO are always similar. This is not surprising since both simulate a Wasserstein gradient flow with respect to the KL divergence and only differ by the time-discretization and the Metropolis correction in MALA.

\paragraph{Development of Error Measures over the Steps}
We plot how the quantities of interest decrease over the application of the steps of our model. The results are given in the Figure~\ref{fig:qoi}.
It can be observed that the different steps may optimize different metrics. While the rejection steps improve the $\log(Z)$ estimate more significantly, the JKO steps focus more on the minimization of the energy distance.
Overall, we see that in all figures the errors decrease over time.

\subsection{Implementation Details}\label{app:implementation_details}

To build our importance corrected neural JKO model, we first apply $n_1\in\mathbb{N}$ JKO steps followed by $n_2\in\mathbb{N}$ blocks consisting out of a JKO step and three importance-based rejection steps. The velocity fields of the normalizing flows are parameterized by a dense three-layer feed-forward neural network. For the JKO steps, we choose an initial step size $\tau_0>0$ and then increase the step size exponentially by $\tau_{k+1}=4\tau_k$. The choices of $n_1$, $n_2$, $\tau_0$ and the number of hidden neurons from the networks is given in Table~\ref{tab:hyper} together with the execution times for training and sampling. For evaluating the density propagation through the CNFs, we use the Hutchinson trace estimator with $5$ Rademacher distributed random vectors whenever $d>5$ and the exact trace evaluation otherwise. Note that we redraw the Rademacher vectors in every step of the ODE solver such that errors average out throughout the trajectory. We observed that the estimator of the divergence already has a sufficiently small variance, particularly, since the trajectories are almost straight due to the norm penalty of the velocity field, see also Appendix~\ref{app:cnf_details}. For implementing the CNFs, we use the code from Ffjord \citep{GCBSD2018} and the \texttt{torchdiffeq} library by \cite{C2018}. We provide the code online at \url{https://github.com/johertrich/neural_JKO_ic}.

For the MALA and HMC we run an independent chain for each generated sample and perform $50000$ steps of the algorithm. For HMC we use $5$ momentum steps and set the step size to $0.1$ for 8Modes and Funnel and step size $0.01$ for mustache. To stabilize the first iterations of MALA and HMC we run the first iterations with smaller step sizes ($0.01$ times the final step size in the first $1000$ iterations and $0.1$ times the final step size for the second $1000$ iterations). 
For MALA we use step size $0.001$ for all examples. For DDS and CRAFT we use the official implementations. In particular for the test distributions such as Funnel and LGCP we use the hyperparameters suggested by the original authors. For the other examples we optimized them via grid search.

\begin{remark}[Initializations]\label{rem:gf_init_and_scaling}
To deal with target distributions living on different scales, we consider initial distributions $\mu_0=\mathcal N(0,c^2\mathrm{Id})$, where $c^2$ is a scaling parameter. The choice of $c$ for each target distribution is givein in Table~\ref{tab:hyper} in the row ``latent scaling''.
Moreover, it is beneficial in some cases to choose the velocity field in the first CNF not by learning it, but by fixing it to $v_t(x)=\nabla_x \log(g(x))=\nabla_x \log(q(x))$. We report the final time $T$ of this first gradient flow step in Table~\ref{tab:hyper} in the row ``gradient flow initialization'' with value - if we do not use a gradient flow step. This initialization is closely related to Langevin preconditioning, which is used in most neural sampling methods, see \cite{HDVZ2025} for a detailed study.
\end{remark}

During the training of $v_{\theta}^{k+1}$, we do not draw in each training step a new batch from $\mu_k$. Instead, we maintain a dataset of $N=50000$ samples of $\mu_k$ and draw batches from it. After the training is completed, we generate a new dataset of $N$ samples by applying the learned CNF onto $N$ new drawn samples from $\mu_k$. To avoid overfitting effects, we redraw our dataset after each CNF step.

\begin{table}
\centering
\caption{Parameters, training and sampling times for the different examples. For the sampling time we draw $N=50000$ samples once the method is trained. The execution times are measured on a single NVIDIA RTX $4090$ GPU with $24$ GB memory.}
\scalebox{.63}{
\begin{tabular}{ccccccccccccccc}
\toprule
  & \phantom{.}   &  \multicolumn{10}{c}{Distribution}  \\
  \cmidrule{3-13} 
Parameter&&Mustache&shifted 8 Modes&shifted 8 Peaky&Funnel&GMM-10&GMM-20&GMM-50&GMM-100&GMM-200&GMM40-50D&LGCP\\
\midrule
Dimension&&2&2&2&10&10&20&50&100&200&50&1600\\
Number $n_1$ of flow steps&&6&2&2&6&4&4&4&4&5&4&3\\
Number $n_2$ of rejection blocks&&6&4&4&6&6&6&7&8&8&7&6\\
Initial step size $\tau_0$ &&0.05&0.01&0.01&5&0.0025&0.0025&0.0025&0.0025&0.001&1&5\\
Latent scaling $c$ &&1&1&1&1&1&1&1&1&1&40&1\\
Gradient Flow Initialization&&-&-&-&-&-&-&-&-&-&3&-\\
Hidden neurons&&54&54&54&256&70&90&150&250&512&256&1024\\
batch size&&5000&5000&5000&5000&5000&5000&5000&5000&2000&5000&500\\
Required GPU memory&&5 GB&5 GB&5 GB&5 GB&5 GB&5 GB&5 GB&5 GB&6 GB&6 GB&11 GB\\
Training time (min)&&38&20&21&107&33&34&44&53&80&34&163\\
Sampling time (sec)&&15&2&3&79&22&23&72&129&473&550&535\\
\bottomrule
\end{tabular}
}

\label{tab:hyper}
\end{table}

\section{Computational Aspects of Normalizing Flows}

In this appendix, we give some more details  about the computational aspects of the normalizing flows in our model.
First, we discuss the relation between multimodal target distributions, mode collapse and non-convex loss functions of normalizing flows.
Afterwards, we discuss some computational aspects of continuous normalizing flows like ODE discretizations and density evaluations with trace estimators.
Finally, we run some standard normalizing flow networks on our numerical example distributions and compare them with our importance corrected neural JKO sampling.

\subsection{Multimodalities, Mode Collapse and Non-Convex Loss Functions}\label{app:nf_collapse}

For the sampling application, normalizing flows are usually trained with the reverse KL loss function $\F(\mu)=\mathrm{KL}(\mu,\nu)$, see, e.g., \cite{MMPS2016}. More precisely, a normalizing flow aims to learn the parameters $\theta$ of a family of diffeomorphisms $\mathcal T_\theta$ such that ${\mathcal T_\theta}_\#\mu_0\approx \nu$ by minimizing $\F({\mathcal T_\theta}_\#\mu_0)$.
In the case that $\nu$ is multimodal it can be observed that this training mode collapses. That is, the approximation $\mathcal T_\#\mu_0$ covers not all of the modes of $\nu$ but instead neglects some of them. Examples of this phenomenon can be seen in the numerical comparison in Appendix~\ref{app:nf_numerics}.
One reason for this effect is that the functional $\F$ is convex if and only if $\nu$ is log-concave and therefore unimodal, see \citep[Prop. 9.3.2]{AGS2005}. In particular, for multimodal $\nu$, the functional $\F$ is non-convex. 
In the latter case of, then the mode collapses can correspond to the convergence to a local minimum, as the following example shows.

\begin{example}\label{ex:nonconvexity}
We consider the case $d=1$ and the target distribution 
$$\nu=\tfrac12\mathcal{N}(-\tfrac12,0.05^2)+\tfrac12\mathcal{N}(\tfrac12,0.05^2).$$
As latent distribution $\mu_0$ we choose a standard Gaussian. Then, we parametrize the normalizing flow $\mathcal T_\theta=(1-\theta)\mathcal \mathcal T_0+\theta\mathcal T_1$ for $\theta\in[0,1]$, where $\mathcal T_0$ is the optimal transport map between $\mu_0$ and $\nu$ and $\mathcal T_1$ is the optimal transport map between $\mu_0$ and $\mathcal N(\frac12,0.05^2)$. In particular, ${\mathcal T_\theta}_\#\mu_0$ perfectly recovers the target distribution for $\theta=0$ and produces a mode collapsed version of it for $\theta=1$.
Now, we plot the reverse KL loss function $\mathcal L(\theta)=\mathrm{KL}({\mathcal T_\theta}_\#\mu_0,\nu)$ and the densities of the generated distributions ${\mathcal T_\theta}_\#\mu_0$ in Figure~\ref{fig:nonconvexity}. We observe that it has two local minima for $\theta=0$ and $\theta=1$ (for $\theta$ outside of $[0,1]$, $\mathcal T$ is no longer invertible), where $\theta=0$ is the perfectly learned parameter and $\theta=1$ is the mode collapsed version.
At the same time, we note that the curve $\theta\mapsto {\mathcal T_\theta}_\#\mu_0$ is a geodesic in the Wasserstein space. In particular, the non-convexity of $\mathcal L$ is a direct consequence of the fact that $\F(\mu)=\mathrm{KL}(\mu,\nu)$ is geodesically non-convex.
\end{example}

\begin{remark}[Motivation of Wasserstein Regularization]
    This connection between the non-convexity of the loss function $\F$ and mode collapse also motivates the Wasserstein regularization from Section~\ref{sec:neural_JKO}.
    By considering the loss function $\mathcal G(\mu)=\F(\mu)+\frac1{2\tau}W_2^2(\mu, \mu_\tau^k)$ for $\tau>0$ small enough instead of $\F$, we obtain a loss function which is convex in the Wasserstein space, see \citep[Lem. 9.2.7]{AGS2005}.
    Even though this does not imply that the map $\theta\mapsto \mathcal G({\mathcal T_\theta}_\#\mu_0)$ is convex, we can expect that the training does not get stuck in local minima as long as the architecture of $\mathcal T_\theta$ is expressive enough.

    Theoretically, we need for $\lambda$-convex $\F$ that $\smash{\tau\leq\frac1\lambda}$ to ensure that $\mathcal G$ is geodesically convex. However, if $\mu^k$ is already close to a minimum of $\F$, then it can be sufficient that the functional $\mathcal G$ is convex locally around $\mu^k$. In this case, the distribution generated by the CNF will stay in this neighborhood. Note that in practice the constant $\lambda$ is unknown such that we start with a small step size $\tau_0$ and increase it over time as outlined at the end of Section~\ref{sec:rejection_steps} and in Appendix~\ref{app:implementation_details}.
\end{remark}

\subsection{Computational Limitations of (Continuous) Normalizing Flows}\label{app:cnf_details}

Similar to the literature on neural JKO schemes \citep{VWTON2023,XCX2024}, our implementation of the neural JKO scheme relies on continuous normalizing flows \citep{CRBD2018}. This comes with some limitations and challenges, which were extensively discussed in \citep{CRBD2018}. Since they are relevant for our method, we give a synopsis below.

\paragraph{Derivatives of ODE Solutions}

For the training phase we need the derivative of the solution of an ODE. For this, we use the \texttt{torchdiffeq} package \citep{C2018}. In particular, this package does not rely on backpropagation through the steps from the forward solver. Instead, it overwrites the backward pass of the ODE by solving an \textit{adjoint ODE}. This avoids expensive tracing in the automatic differentiation process within the forward pass and keeps the memory consumption low.
Moreover, the quadratic regularization of the velocity field leads to straight paths such that the adaptive solvers only require a few steps for solving the ODE, see \citep{OFLR2021} for a detailed discussion. Indeed, we use in our numerical examples the \texttt{dopri5} solver from \texttt{torchdiffeq}, which is an adaptive Runge-Kutta method.
However, we still have to solve two ODEs during the training time. This still can be computational costly, in particular when the underlying model is large.

\paragraph{Trace Estimation for Density Computations}

In order to evaluate the (log-)density of our model, we have to compute the divergence $\operatorname{div}v_{\theta}=\mathrm{trace}(\nabla v_\theta(z_\theta(x,t),t))$ of the learned vector field $v_\theta$ (and integrate it over time), see \eqref{eq:neuralJKO_ODE}. Computing the trace of the Jacobian of $v_\theta$ becomes computationally costly in high dimensions. As a remedy, we consider the Hutchinson trace estimator \citep{H1989}, which states that for any matrix $A$ and a random vector $z$ with mean zero and identity covariance matrix it holds that $\E[z^\tT Az]=\mathrm{trace}(A)$. Applying this estimator to the divergence, we obtain that the divergence coincides with $\E[z^\tT\nabla v_\theta(z_\theta(x,t),t)z]$. The integrand is now again a Jacobian-vector product which can be computed efficiently.
Finally, we estimate the trace by empirically discretizing the expectation by finitely many realizations of $z$.

In our numerics, we choose $z$ to be Rademacher random vectors, i.e. each entry is with probability $\frac12$ either $-1$ or $1$. For the training of the continuous normalizing flow, an unbiased low-precision estimator is sufficient, such that we discretize the expectation  with one realization of $z$. During evaluation, we require a higher precision and use $5$ realizations instead.

\paragraph{Initialization}
In order to find a stable initialization of the model, we initialize the last layer of the velocity field $v_\theta$ with zeros such that it holds $v_\theta(x,t)=0$ for all $x\in\R^d$, $t\in[0,\tau]$. In this case, the solution $z_\theta$ of the ODE $\dot z_\theta=v_\theta$ with initial condition $z_\theta(x,0)=x$ is given by $z_\theta(x,t)=x$ for all $t$. In particular, we have that the transport map $\mathcal T_\theta=z_\theta(x,\tau)$ is the identity such that the generated distribution ${\mathcal T_\theta}_\#\mu_0$ coincides with the latent distribution for the initial parameters.

\begin{remark}[Other Normalizing Flow Architectures]\label{rem:other_architectures}
In general any architecture $\mathcal T_\theta$ of normalizing flows can be considered in the neural JKO scheme. To this end, we can replace the velocity field regularization in \eqref{eq:dynamic_JKO} by the penalizer $\int_{\R^d}\|\mathcal T_\theta(x)-x\|^2\d \mu_\tau^k$. By Brenier’s theorem \citep{Brenier1987} this is again equivalent to minimizing the Wasserstein distance.
However, while discrete-time normalizing flows like coupling-based networks \citep{DSB2016,KD2018} and autoregressive flows \citep{CTA2019,DBMP2019,HKLC2018,PPM2017}  are often faster than CNFs, this is not generally true in our setting, since their evaluation time does not benefit from the OT-regularization. 
Additionally, we observed numerically that the expressiveness of discrete-time architectures scales much worse to high dimensions and that these architectures are less stable to train.
Nevertheless, their evaluation can be cheaper since these architectures do not have to deal with trace estimators within the density evaluation.
Residual architectures \citep{BGCD2019,CBDJ2029} are at least similar to continuous normalizing flows, but they are very expensive to train and evaluate, and additionally rely on trace estimators for computing the density of the generated samples.
\end{remark}

\subsection{Numerical Comparison}\label{app:nf_numerics}

\begin{table}
\caption{Comparison of neural JKO IC with normalizing flows trained with the reverse KL loss function. We evaluate the energy distance (smaller values are better), the $\log(Z)$-estimation (larger values are better) and the expected squared empirical Wasserstein distance (smaller values are better).}
\centering
\scalebox{.66}{
\begin{tabular}{cclllclllclll}
\toprule
  & \phantom{.}   &  \multicolumn{3}{c}{Energy distance}& \phantom{.}   &  \multicolumn{3}{c}{$\log(Z)$-estimation}& \phantom{.}   &  \multicolumn{3}{c}{squared Wasserstein-$2$}\\
  \cmidrule{3-5}  \cmidrule{7-9} \cmidrule{11-13}    
Distribution& &\multicolumn{1}{c}{continuous NF}&\multicolumn{1}{c}{coupling NF}&\multicolumn{1}{c}{neural JKO IC}& &\multicolumn{1}{c}{continuous NF}&\multicolumn{1}{c}{coupling NF}&\multicolumn{1}{c}{neural JKO IC}& &\multicolumn{1}{c}{continuous NF}&\multicolumn{1}{c}{coupling NF}&\multicolumn{1}{c}{neural JKO IC}\\
\midrule
Mustache&&$\num{2.1e-2}$&$\num{4.9e-3}$&\best$\num{2.9e-3}$&&$\num{-7.1e-2}$&$-\num{2.2e-2}$&\best$-\num{7.3e-3}$&&$\num{3.8e1}$&$\num{3.1e1}$&\best$\num{1.7e1}$\\
shifted 8 Modes&&$\num{4.1e-1}$&$$\num{4.1e-2}$$&\best$\num{1.2e-5}$&&$\num{-1.4e-0}$&$-\num{4.3e-1}$&\best$+\num{5.1e-6}$&&$\num{1.4e0}$&$\num{2.2e-1}$&\best$\num{6.5e-3}$\\
shifted 8 Peaky&&$\num{5.7e-1}$&$\num{5.8e-1}$&\best$\num{3.4e-5}$&&$\num{-2.1e-0}$&$-\num{2.1e-0}$&\best$-\num{2.1e-3}$&&$\num{1.8e0}$&$\num{1.8e0}$&\best$\num{7.2e-3}$\\
Funnel&&$\num{1.9e-1}$&\best$\num{2.2e-3}$&$\num{1.4e-2}$&&$\num{-9.4e-1}$&$-\num{2.8e-2}$&\best$-\num{7.1e-3}$&&\best$\num{3.8e2}$&$\num{8.1e2}$&$\num{8.5e2}$\\
GMM-10&&$\num{5.4e-1}$&$\num{5.3e-1}$&\best$\num{5.2e-5}$&&$\num{-2.3e-0}$&$-\num{2.3e-0}$&\best$+\num{3.5e-3}$&&$\num{3.4e0}$&$\num{3.5e0}$&\best$\num{1.4e-1}$\\
GMM-20&&$\num{1.1e-0}$&$\num{1.1e-0}$&\best$\num{1.1e-4}$&&$\num{-2.4e-0}$&$-\num{2.3e-0}$&\best$+\num{6.4e-3}$&&$\num{9.5e0}$&$\num{9.4e0}$&\best$\num{3.7e-1}$\\
\bottomrule
\end{tabular}
}
\label{tab:nf_comparison}
\end{table}

Finally, we compare our neural JKO IC scheme with standard normalizing flows trained with the reverse KL loss function as proposed in \cite{MMPS2016}. In particular, we aim to demonstrate the benefits of our neural JKO IC scheme to avoid mode collapse.

To this end, we train two architectures of normalizing flows for our examples using the reverse KL loss function as proposed in \cite{MMPS2016}. First, we use a continuous normalizing flow with the same architecture as used in the neural JKO (IC). That is, we parameterize the velocity field $v_\theta$ by a dense neural network with three hidden layers and the same hidden dimensions as in Table~\ref{tab:hyper}. Second, we compare with a coupling network with 5 Glow-coupling blocks \citep{KD2018}, where the coupling blocks have two hidden layers with three times the hidden dimension as in Table~\ref{tab:hyper}.
We found numerically that choosing larger architectures does not bring significant advantages.

We plot some generated distributions of the continuous and coupling normalizing flows as well as the neural JKO IC scheme in Figure~\ref{fig:NF_comparison}. As we have already seen in the previous examples, the neural JKO IC scheme is able to recover multimodal distributions almost perfectly. On the other hand, the normalizing flow architectures always collapse to one or a small number of modes.
We additionally report the error measures in Table~\ref{tab:nf_comparison}. We can see that neural JKO IC performs always better than the normalizing flow architectures (note that for the Funnel distribution all $W_2^2$-values are below the sampling error reported in Table~\ref{tab:w2_results}). For multimodal distributions, the normalizing flow approximations are by several orders of magnitude worse than neural JKO IC, while they work quite well for unimodal distributions (Mustache and Funnel).

\begin{figure}
\begin{figure}[H]
\begin{center}
    \includegraphics[width = 0.8\linewidth]{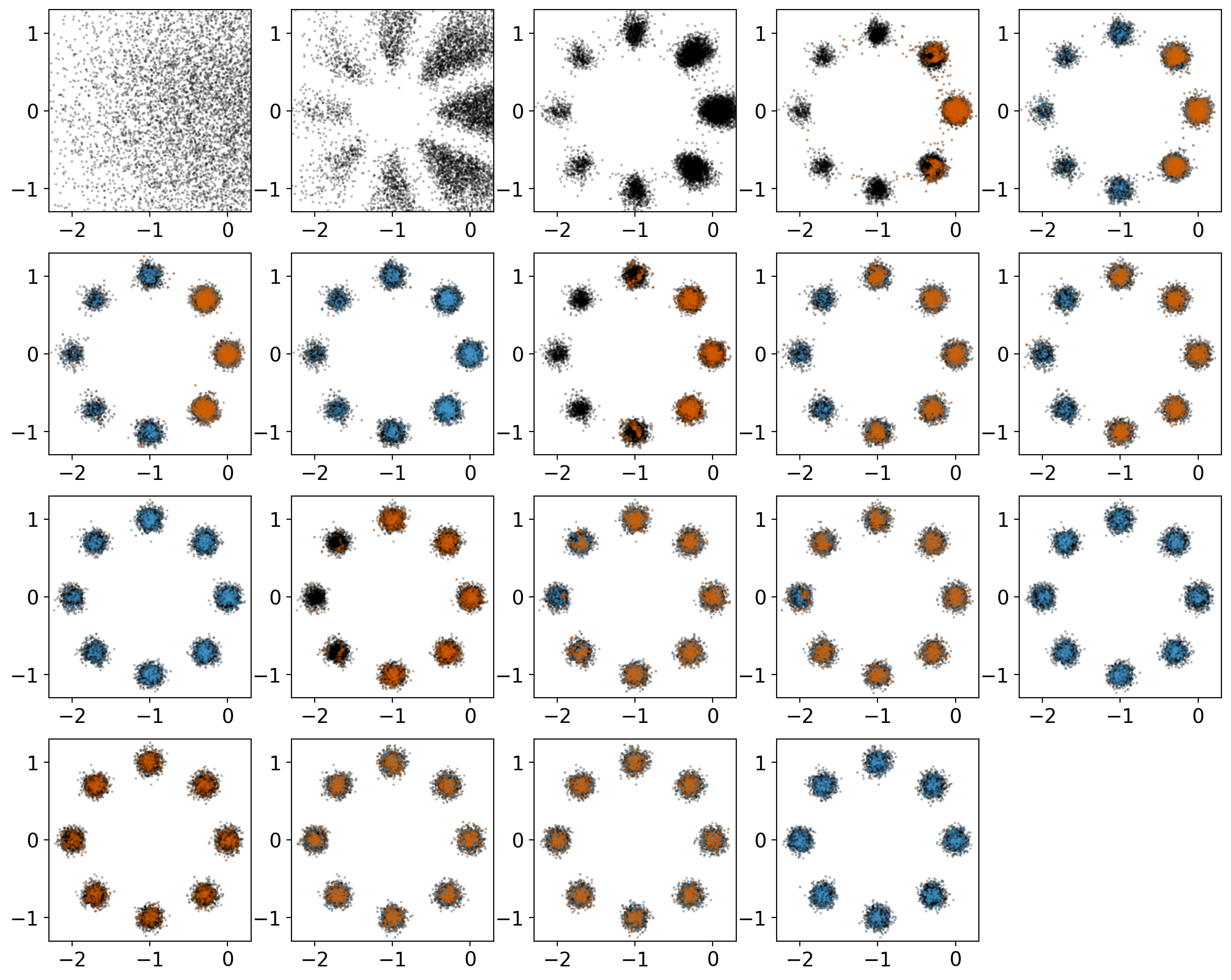}
\caption{
Visualization of the steps of the importance corrected neural JKO model for the shifted 8 Peaky example. We start at the top left with the latent distribution and apply in each image one step from the model. The orange color indicates samples which are rejected in the next step and the blue color marks the resampled points.
\label{fig:8modes_2d}
}
\end{center}
\end{figure}

\begin{figure}[H]
\begin{center}
    \includegraphics[width = .7\linewidth]{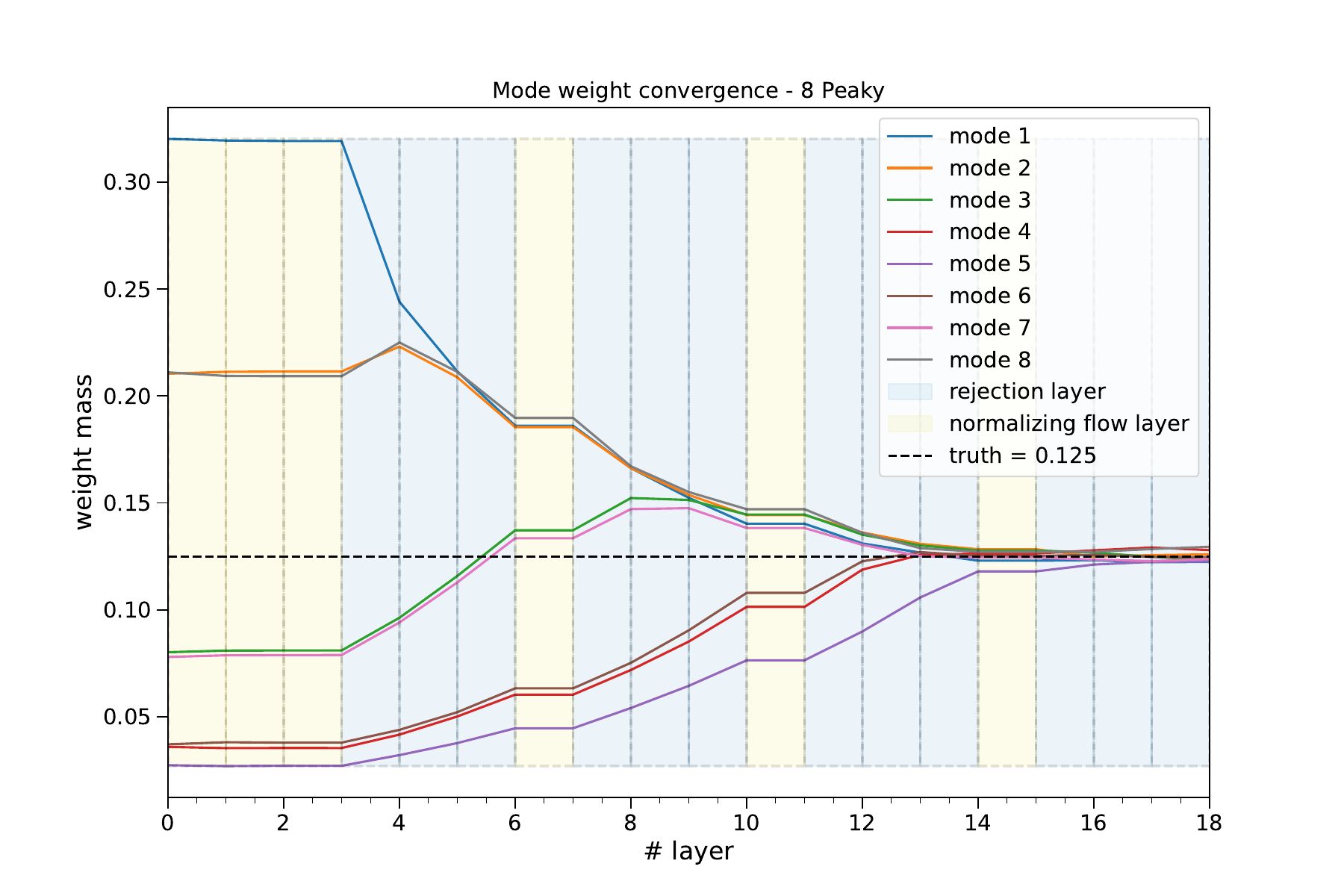}
\caption{
Plot of the mode weights for the 8 Peaky example over the different layers of the importance corrected neural JKO model. We observe that the weights are mainly changed by the rejection steps and not by the neural JKO steps. 
\label{fig:weight_conv_8modes_2d}
}
\end{center}
\end{figure}
\end{figure}

\begin{figure}
\begin{figure}[H]
\begin{center}
    \includegraphics[width = .8\linewidth]{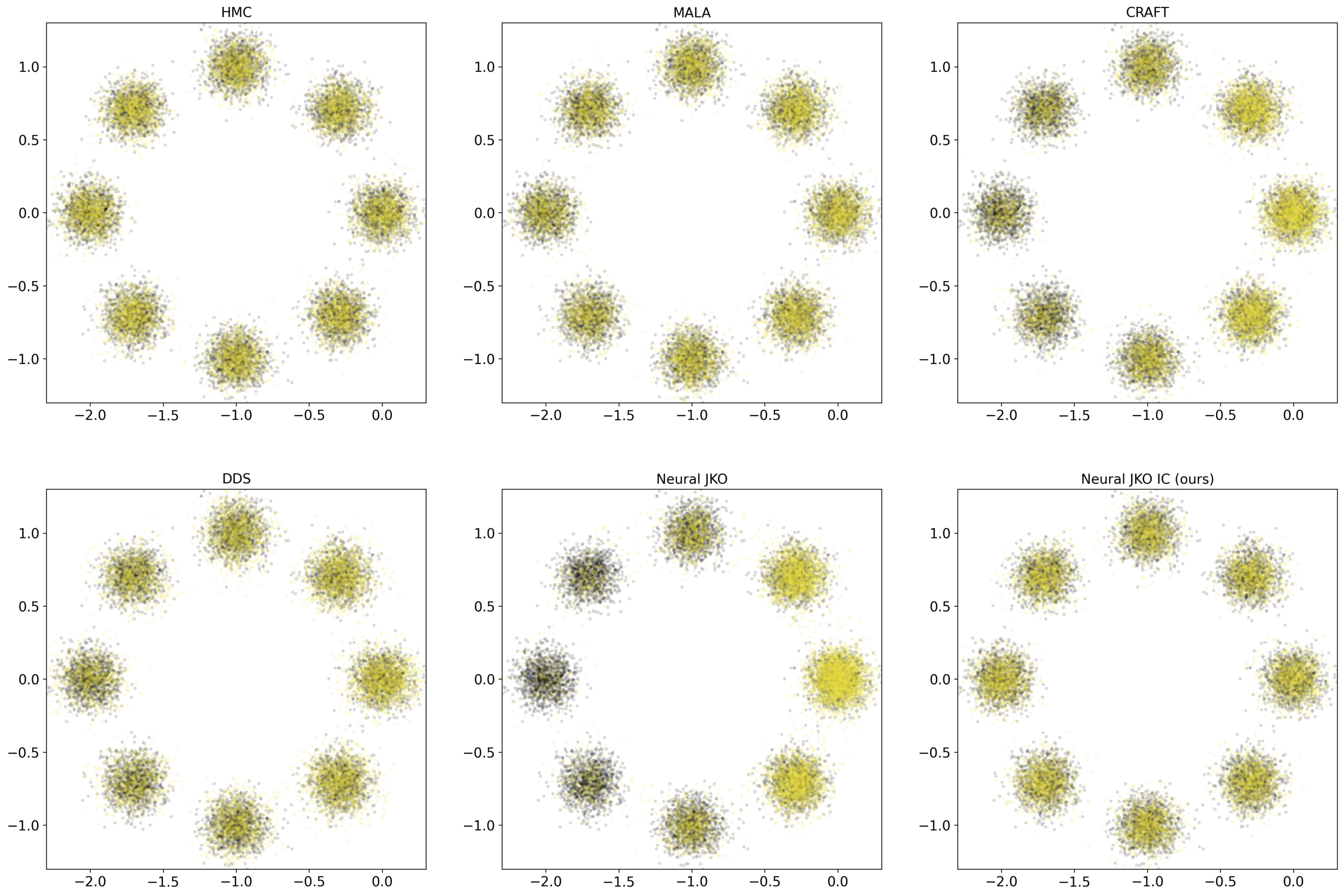}
\caption{
Sample generation with various methods for shifted $8$ Modes example with \textit{ground truth} samples and \textcolor[RGB]{240, 228, 66}{\textit{generated}} samples for each associated method.
While most methods recover the distribution well, we can see a imbalance in the modes for the uncorrected neural JKO, CRAFT and DDS.
\label{fig:samples_8mixtures}
}
\end{center}
\end{figure}

\begin{figure}[H]
\begin{center}
    \includegraphics[width = .8\linewidth]{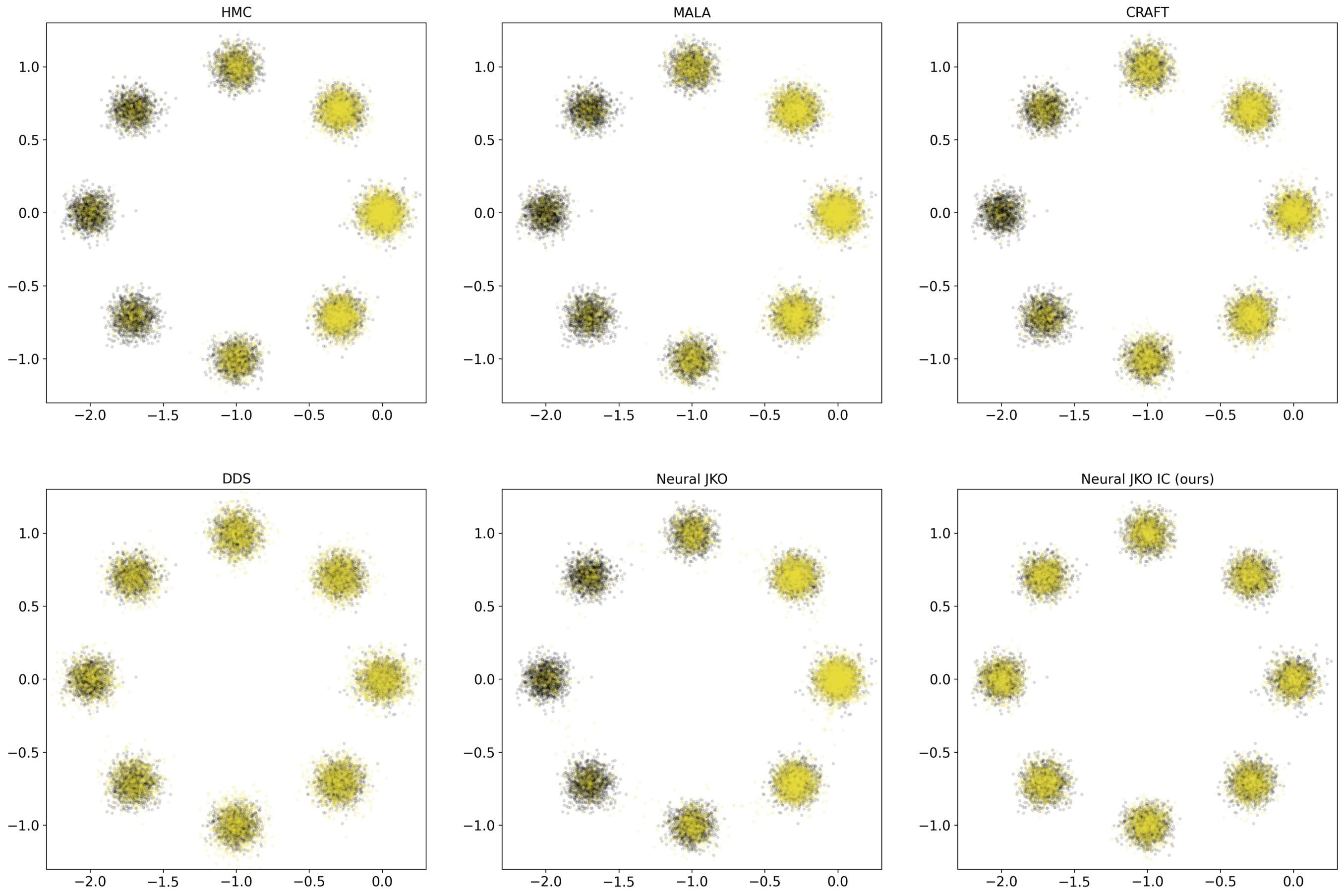}
\caption{
Sample generation with various methods for shifted $8$ Peaky mixtures with \textit{ground truth} samples and \textcolor[RGB]{240, 228, 66}{\textit{generated}} samples for each associated method. We can see a severe imbalance among the modes for HMC, MALA, CRAFT and neural JKO. Also DDS has a slight imbalance between the modes.
\label{fig:samples_8peaky}
}
\end{center}
\end{figure}
\end{figure}

\begin{figure}
\begin{center}
    \includegraphics[width = 1.\linewidth]{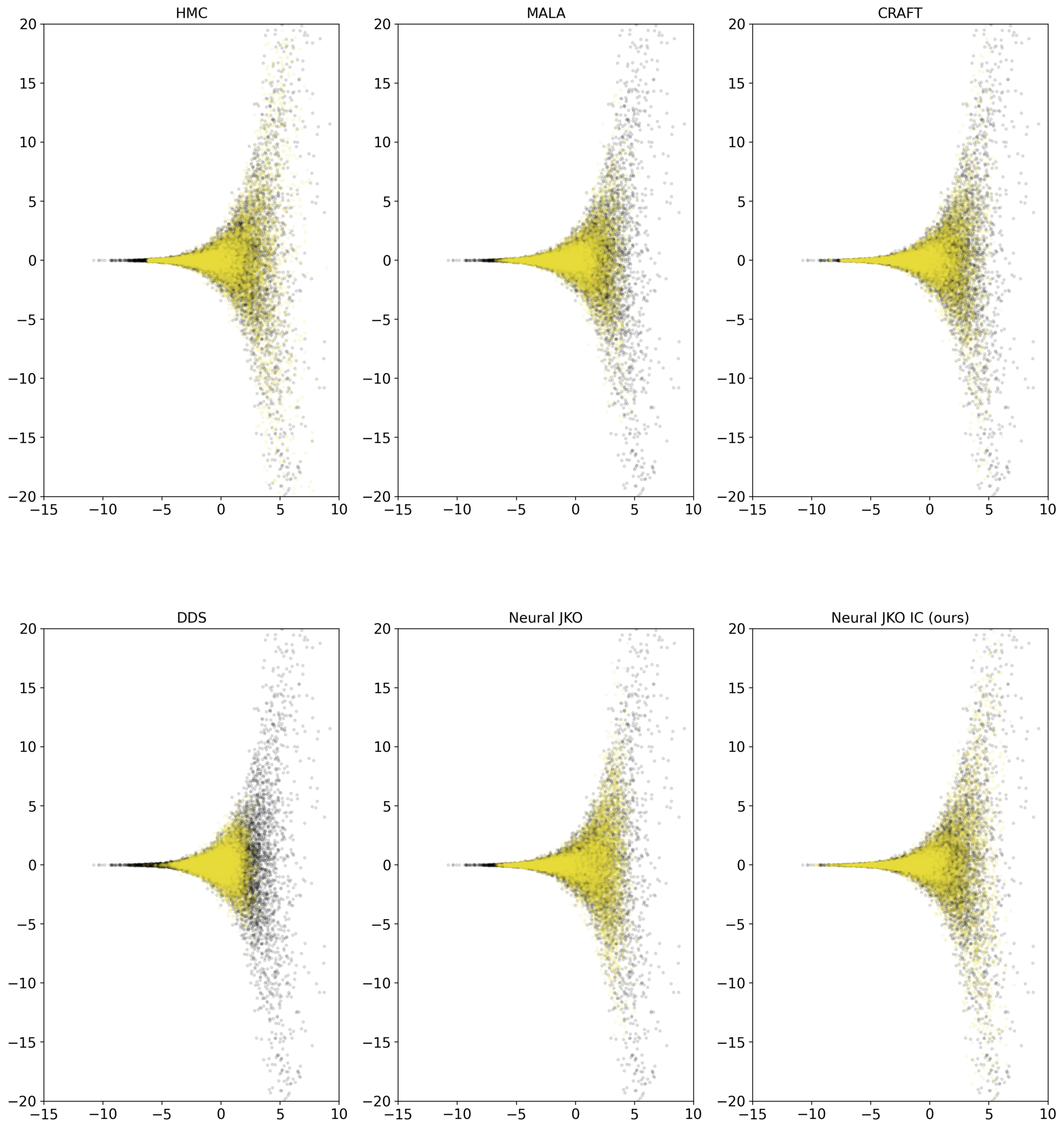}
\caption{
Marginalized sample generation with various methods for the $d=10$ funnel distribution with \textit{ground truth} samples and \textcolor[RGB]{240, 228, 66}{\textit{generated}} samples for each associated method.
We observe that only the importance corrected neural JKO covers the thin part of the funnel well.
\label{fig:samples_funnel}
}
\end{center}
\end{figure}

\begin{figure}
\begin{center}
    \includegraphics[width = 1\linewidth]{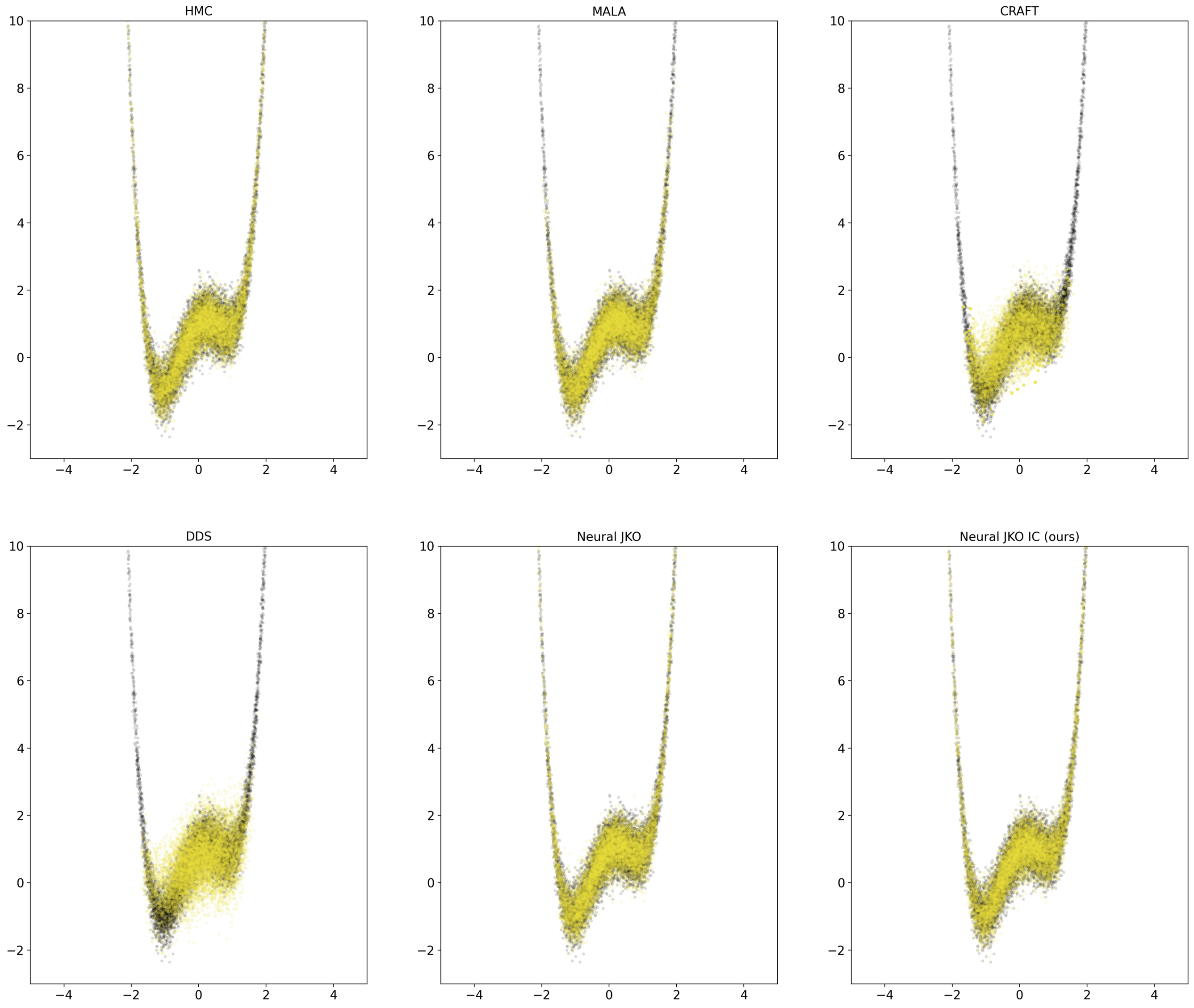}
\caption{
Sample generation with various methods for the $d=2$ mustache distribution with \textit{ground truth} samples and \textcolor[RGB]{240, 228, 66}{\textit{generated}} samples for each associated method.
We can see that MALA, CRAFT and DDS have difficulties to model the long tails of the distribution properly.
\label{fig:samples_schnauzbart}
}
\end{center}
\end{figure}

\begin{figure}
\begin{figure}[H]
\begin{center}
    \includegraphics[width = .8\linewidth]{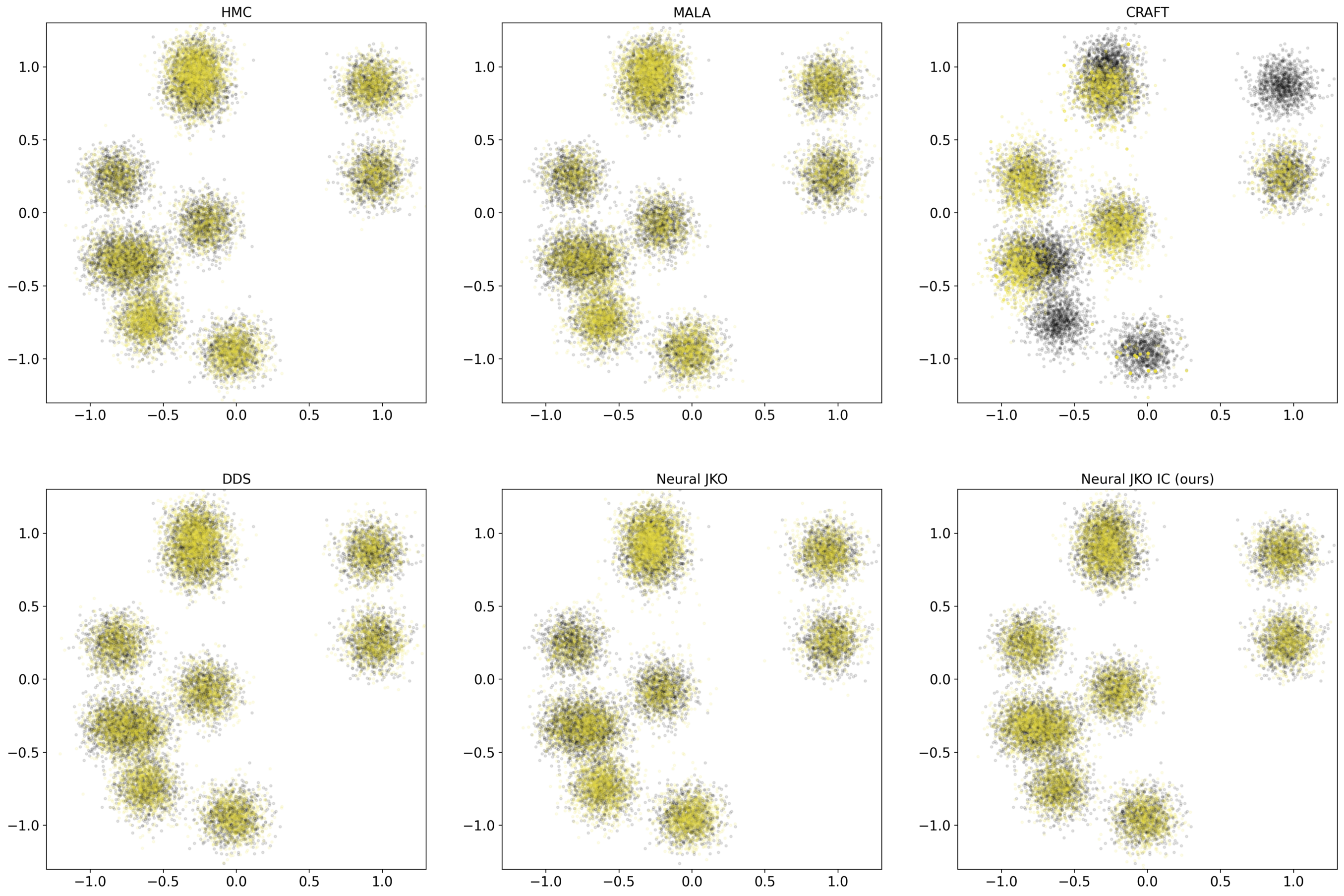}
\caption{
Marginalized sample generation with various methods for the GMM-$10$ distribution with \textit{ground truth} samples and \textcolor[RGB]{240, 228, 66}{\textit{generated}} samples for each associated method. We observe that CRAFT mode collapses and that only the importance corrected neural JKO model distributes the mass correctly onto the modes.
\label{fig:samples_mixtures10}
}
\end{center}
\end{figure}

\begin{figure}[H]
\begin{center}
    \includegraphics[width = .8\linewidth]{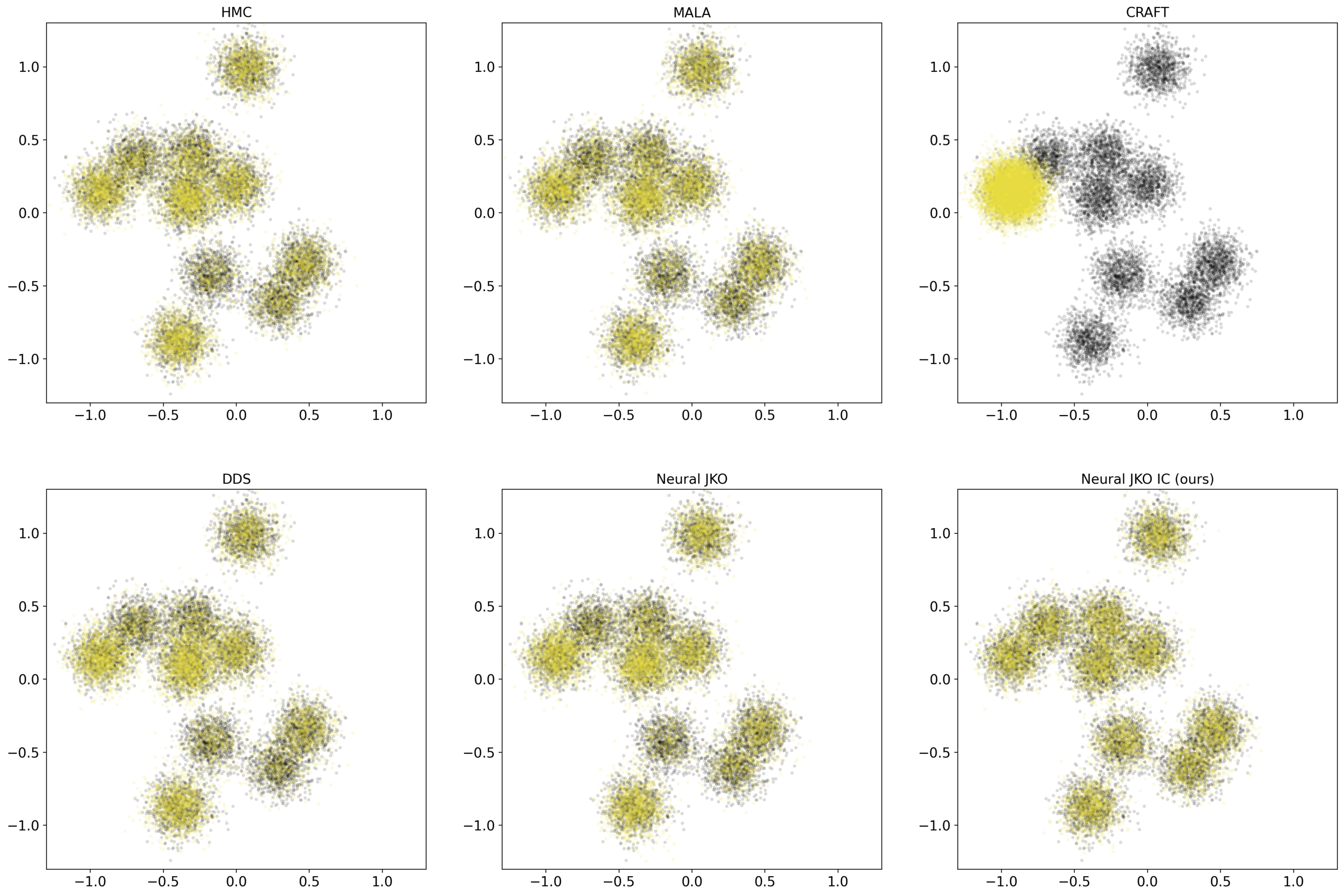}
\caption{
Marginalized sample generation with various methods for the GMM-$200$ distribution with \textit{ground truth} samples and \textcolor[RGB]{240, 228, 66}{\textit{generated}} samples for each associated method. We observe that CRAFT mode collapses and that only the importance corrected neural JKO model distributes the mass correctly onto the modes.
\label{fig:samples_mixtures200}
}
\end{center}
\end{figure}
\end{figure}

\begin{figure}
\begin{center}
    Mustache
    
    \includegraphics[width = .78\linewidth]{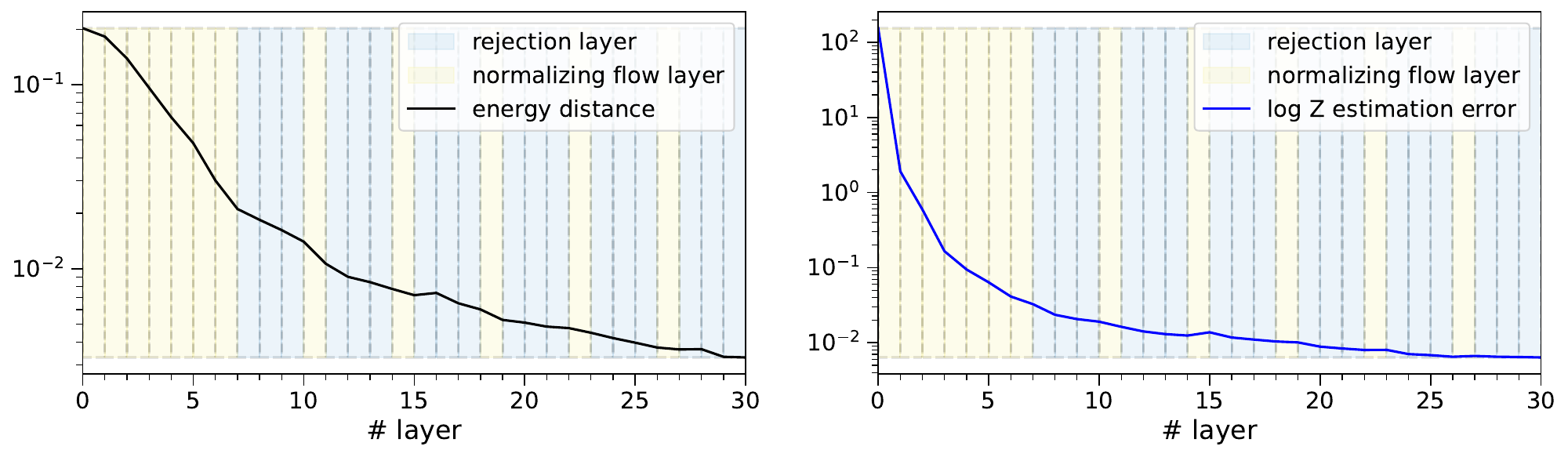}  
    
    shifted $8$ Modes
    
    \includegraphics[width = .78\linewidth]{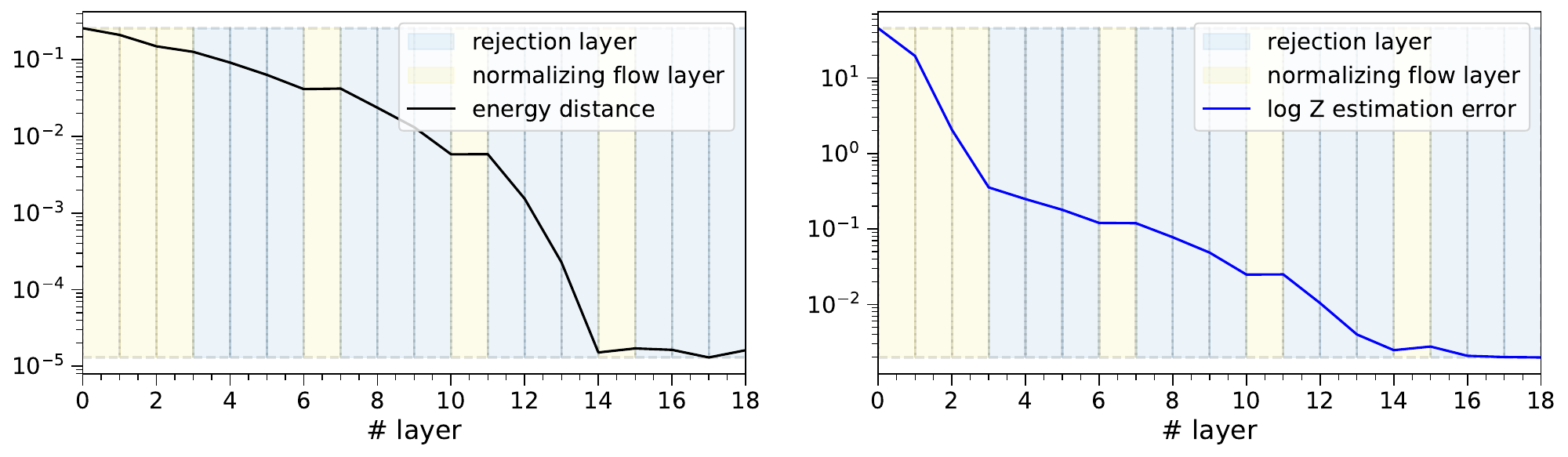}
    
    shifted $8$ Peaky
    
    \includegraphics[width = .78\linewidth]{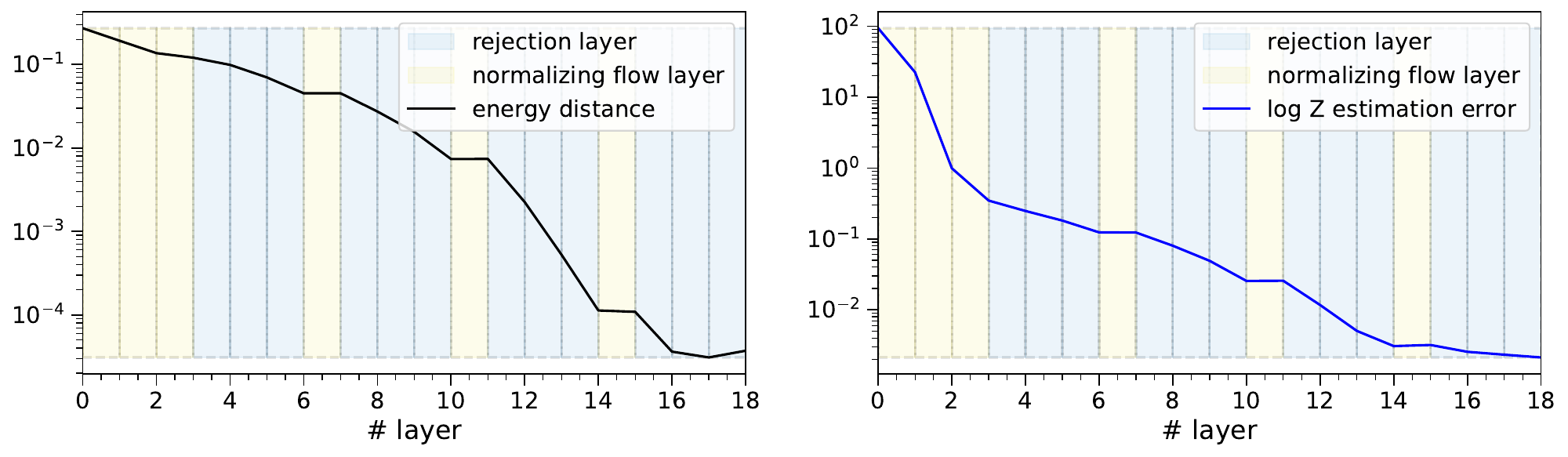}
    
    Funnel
    
    \includegraphics[width = .78\linewidth]{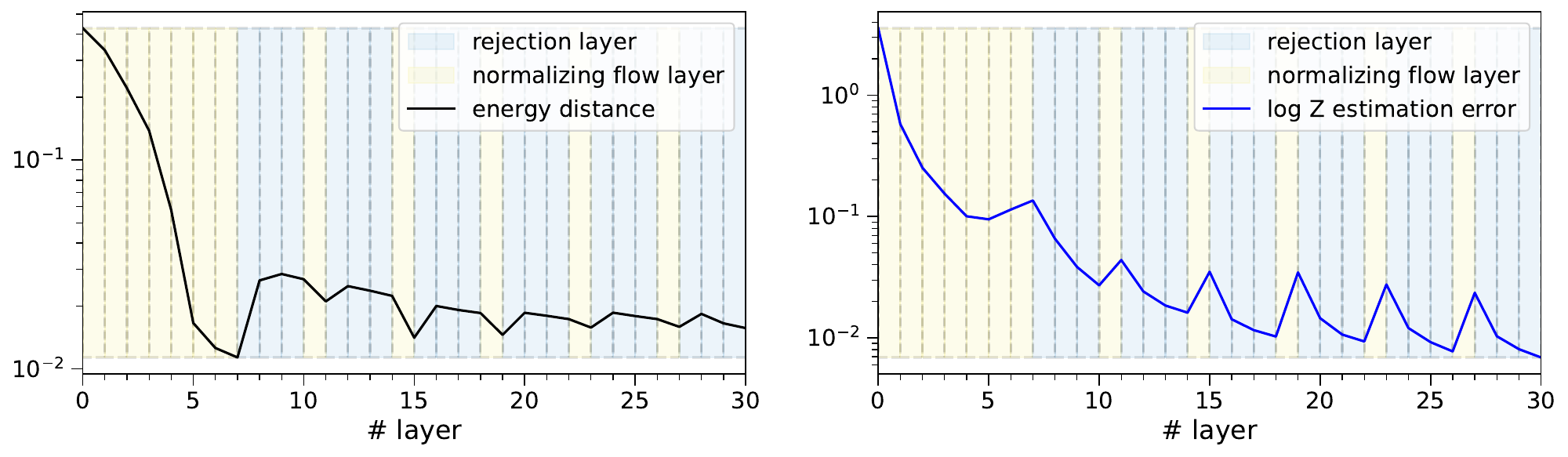}
    
    GMM-$10$
    
    \includegraphics[width = .8\linewidth]{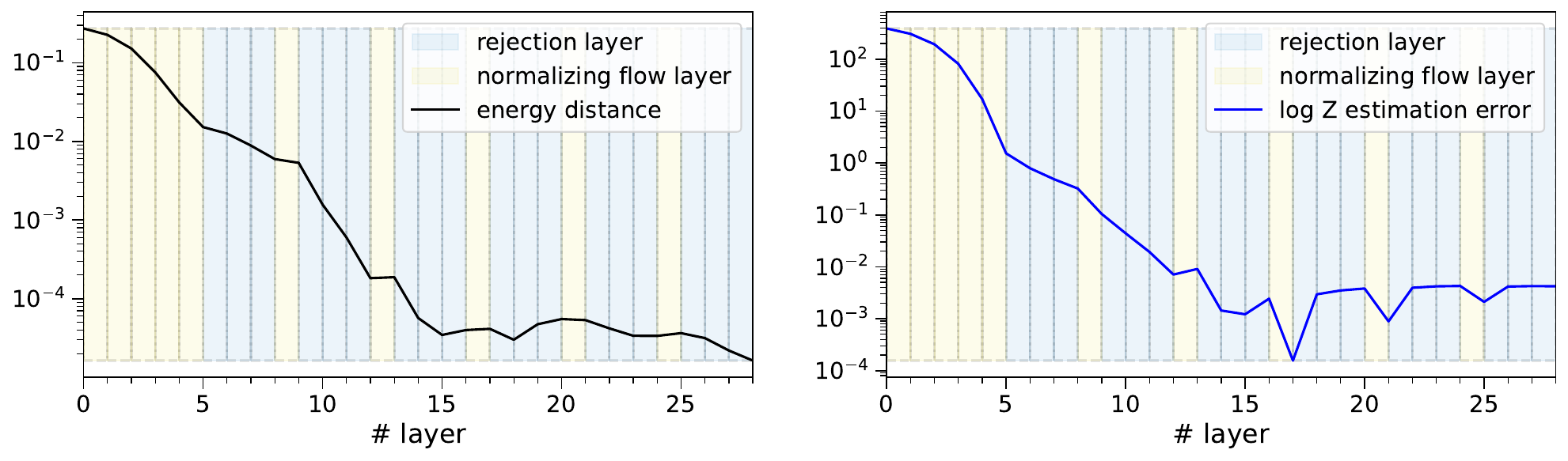}
\caption{We plot the energy distance (\textbf{left}) and $\log(Z)$ estimate (\textbf{right}) over the steps of our importance corrected neural JKO method for different examples. The error measures decrease in the beginning and then saturate at some values.
\label{fig:qoi}
}
\end{center}
\end{figure}

\begin{figure}
\begin{figure}[H]
    \centering
    Non-convex loss function $\mathcal L(\theta)=\mathrm{KL}({\mathcal T_\theta}_\#\mu_0,\nu)$
    
    \begin{subfigure}[t]{.4\textwidth}
        \centering
        
        \includegraphics[width=0.65\textwidth]{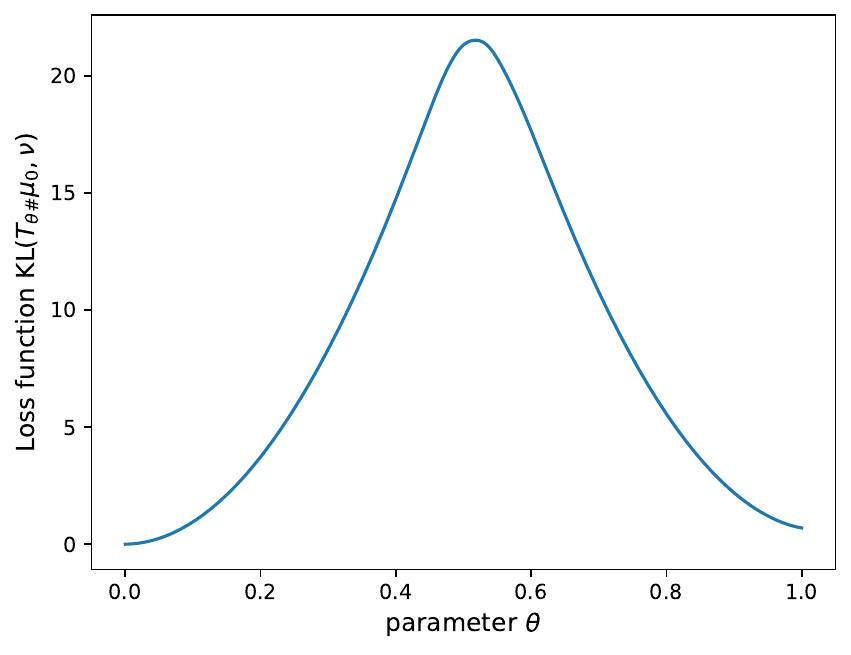}
    \end{subfigure}
    \vspace{.5cm}

    Density of ${\mathcal T_\theta}_\#\mu_0$ for different values of $\theta$
    
    \begin{subfigure}[t]{.23\textwidth}
        \includegraphics[width=0.9\textwidth]{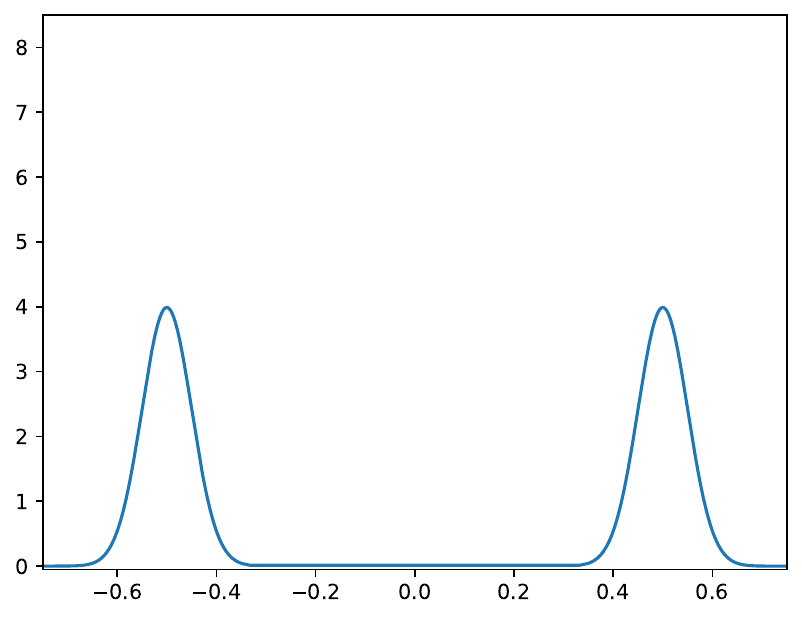}
        \caption*{$\theta=0$}
    \end{subfigure}
        \begin{subfigure}[t]{.23\textwidth}
        \includegraphics[width=0.9\textwidth]{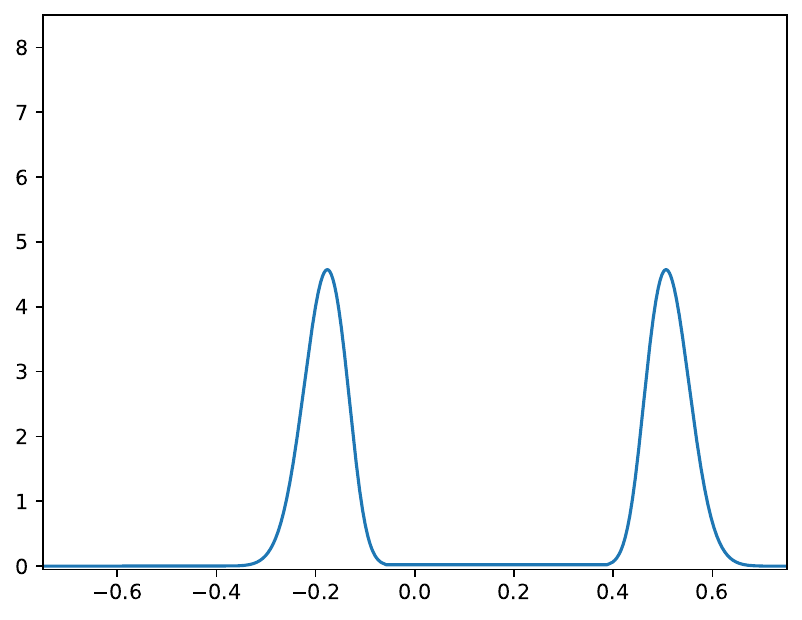}
        \caption*{$\theta=\frac13$}
    \end{subfigure}
        \begin{subfigure}[t]{.23\textwidth}
        \includegraphics[width=0.9\textwidth]{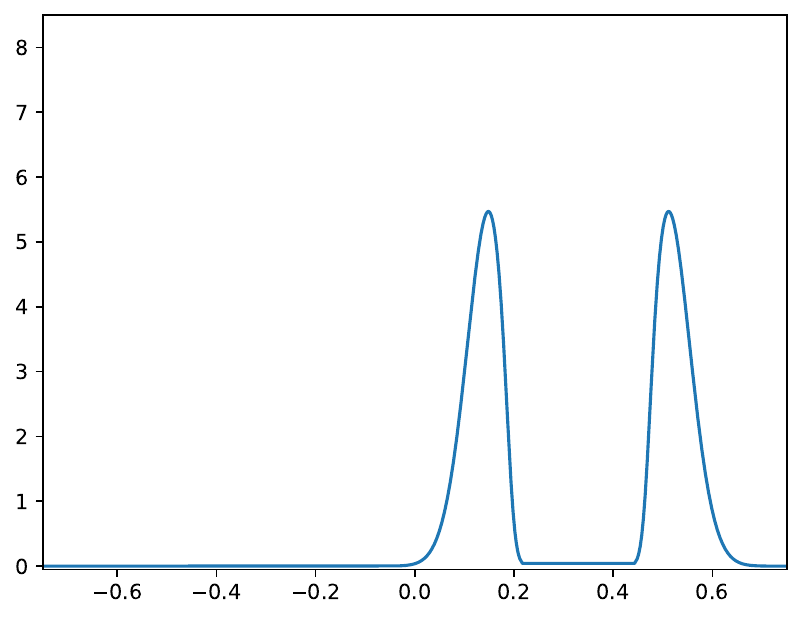}
        \caption*{$\theta=\frac23$}
    \end{subfigure}
        \begin{subfigure}[t]{.23\textwidth}
        \includegraphics[width=0.9\textwidth]{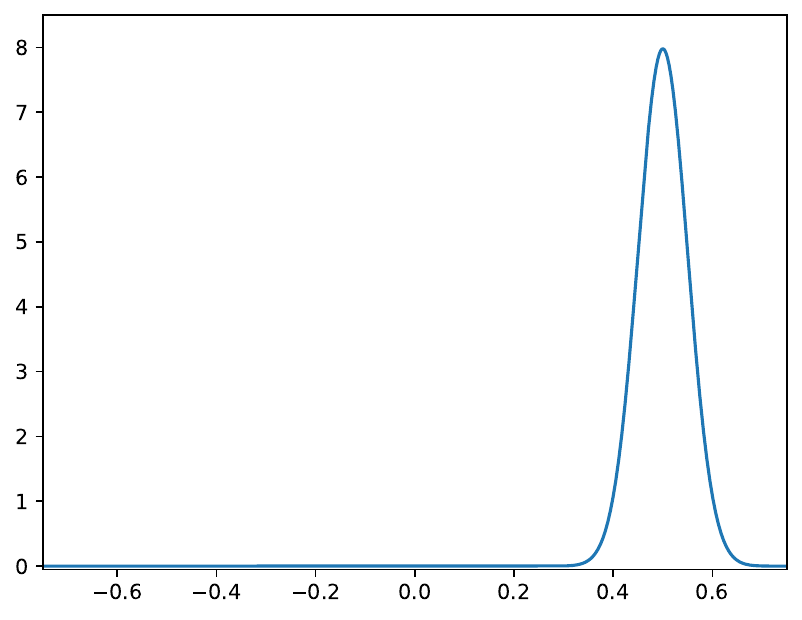}
        \caption*{$\theta=1$}
    \end{subfigure}
    \caption{
    Illustration of the loss function $[0,1]\ni \theta \mapsto \mathcal L(\theta)=\mathrm{KL}({\mathcal T_\theta}_\#\mu_0,\nu)$ and the densities of the generated distributions ${\mathcal T_\theta}_\#\mu_0$ from Example~\ref{ex:nonconvexity}. Both values $\theta=0$ and $\theta = 1$ correspond to local minima (note that $\mathcal{L}(1)>0=\mathcal{L}(0)$). In particular $\theta=1$ corresponds to the case of mode collapse.
    }
    \label{fig:nonconvexity}
\end{figure}

\begin{figure}[H]
    \centering
    \begin{subfigure}{.21\textwidth}
        \includegraphics[width=0.9\textwidth]{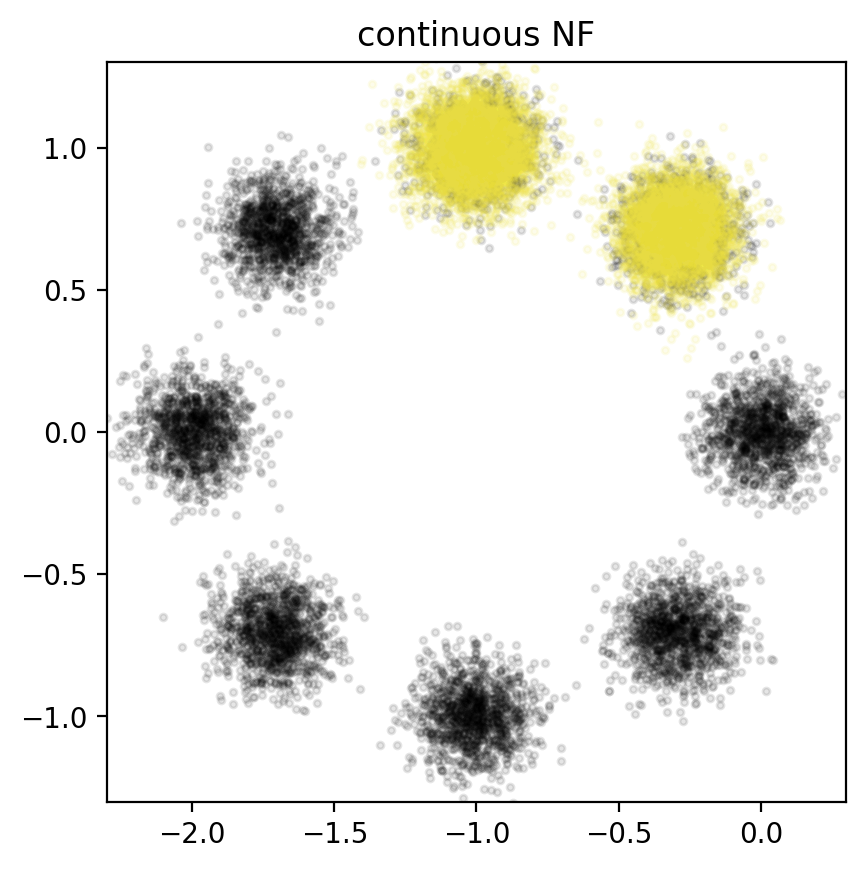}
        \includegraphics[width=0.9\textwidth]{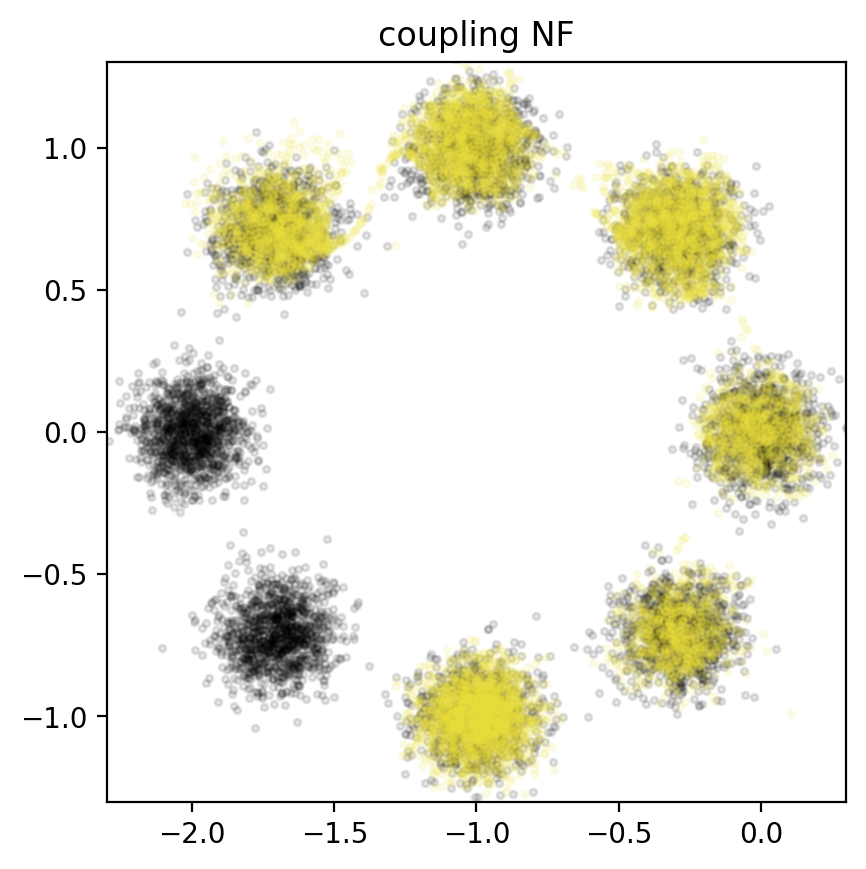}
        \includegraphics[width=0.9\textwidth]{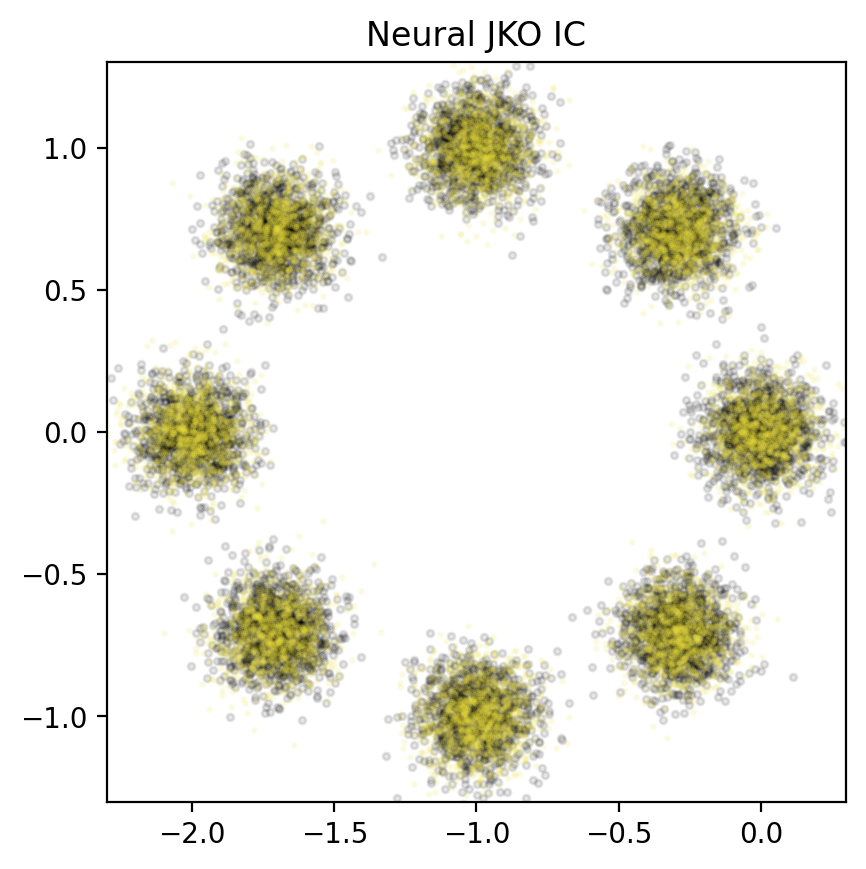}
    \caption*{8 Modes}
    \end{subfigure}
    \begin{subfigure}{.21\textwidth}
        \includegraphics[width=0.9\textwidth]{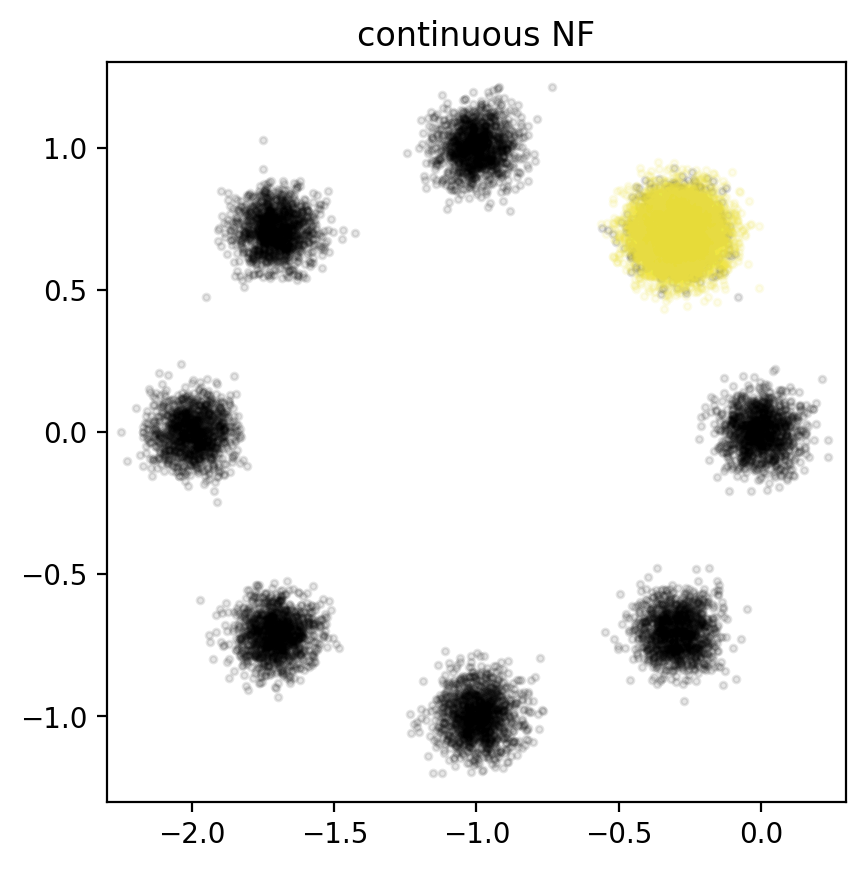}
        \includegraphics[width=0.9\textwidth]{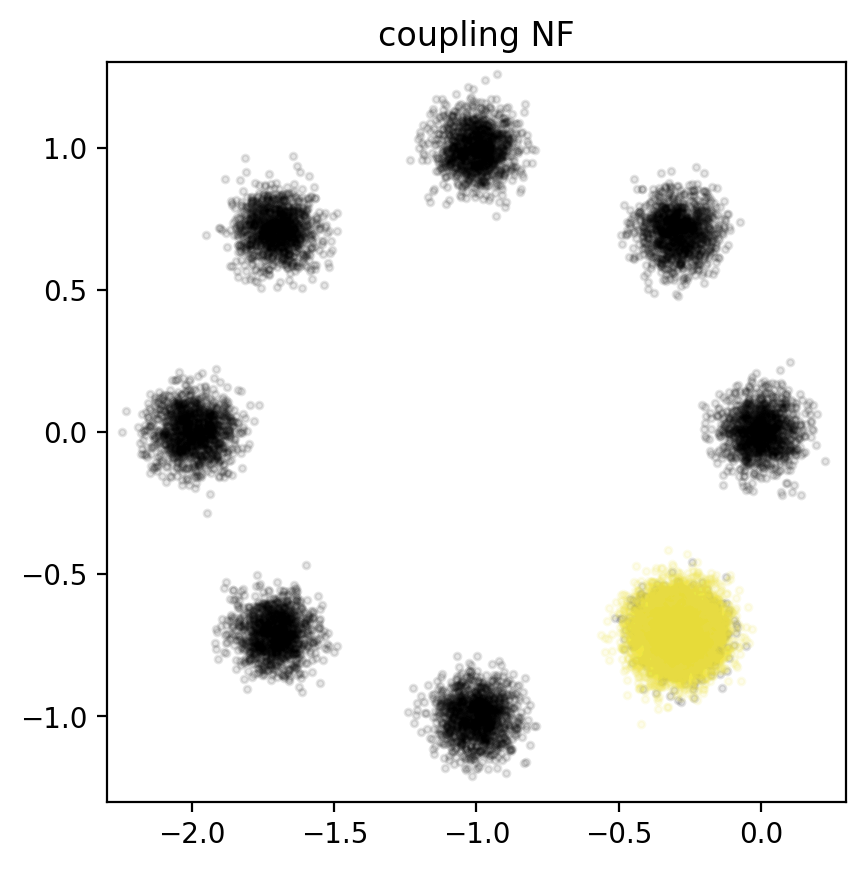}
        \includegraphics[width=0.9\textwidth]{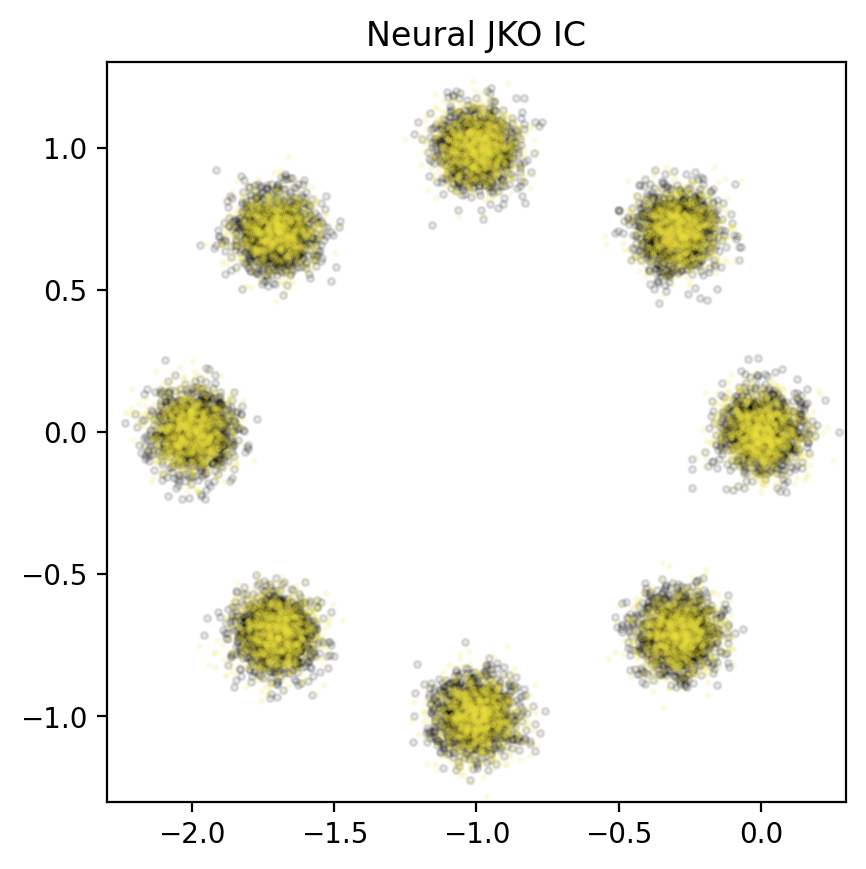}
    \caption*{8 Peaky}
    \end{subfigure}
    \begin{subfigure}{.21\textwidth}
        \includegraphics[width=0.9\textwidth]{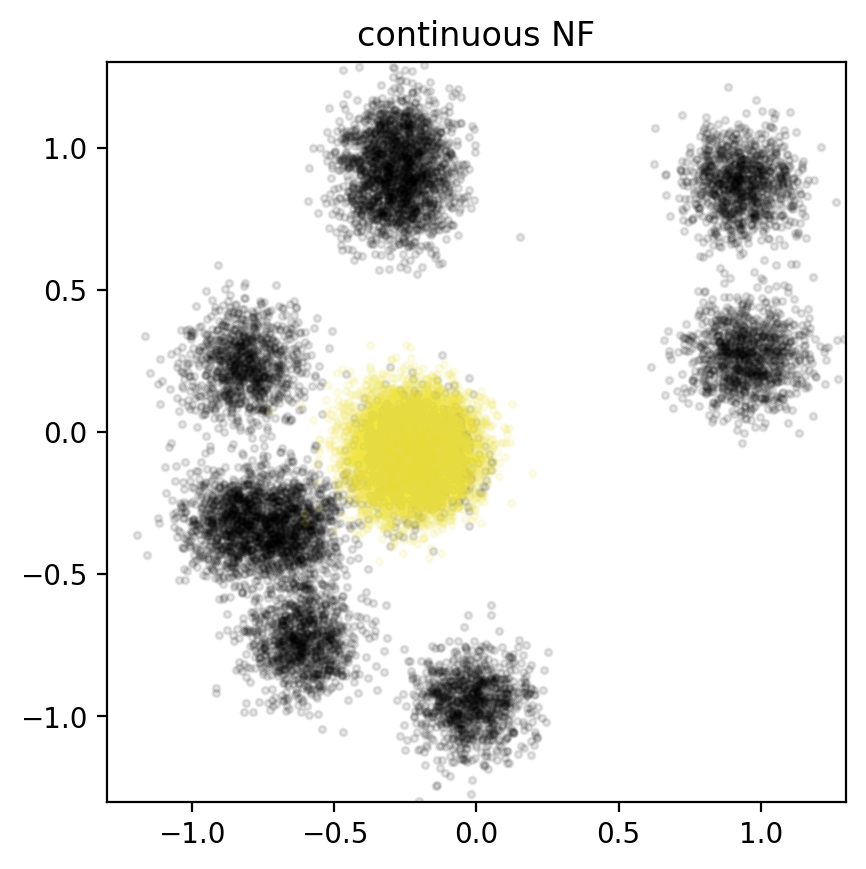}
        \includegraphics[width=0.9\textwidth]{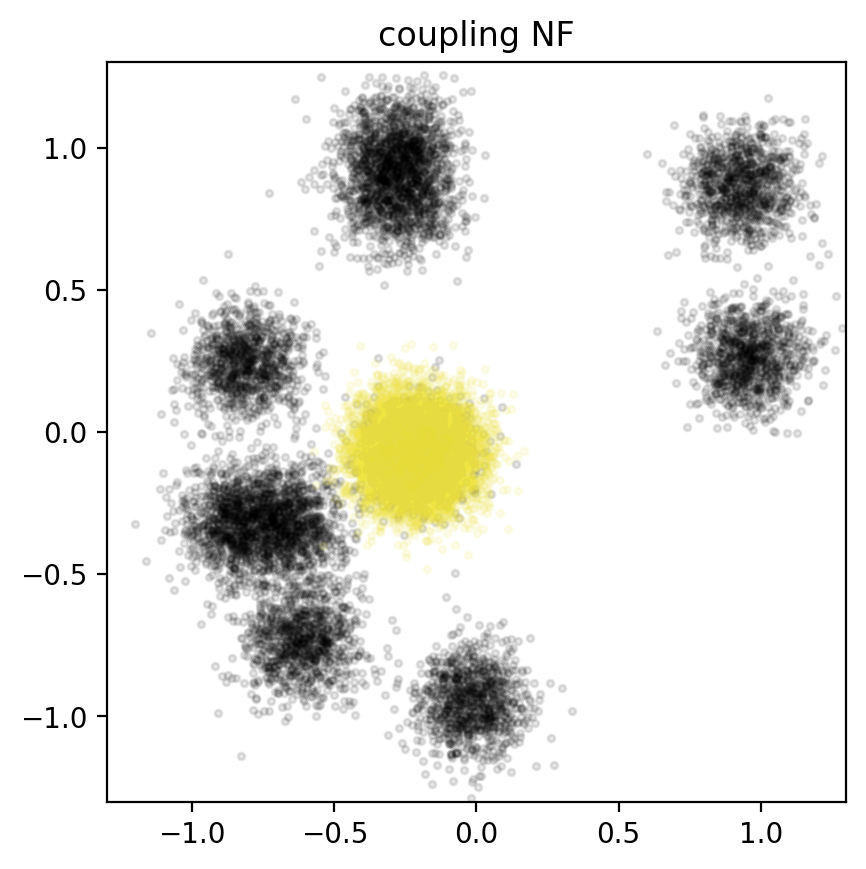}
        \includegraphics[width=0.9\textwidth]{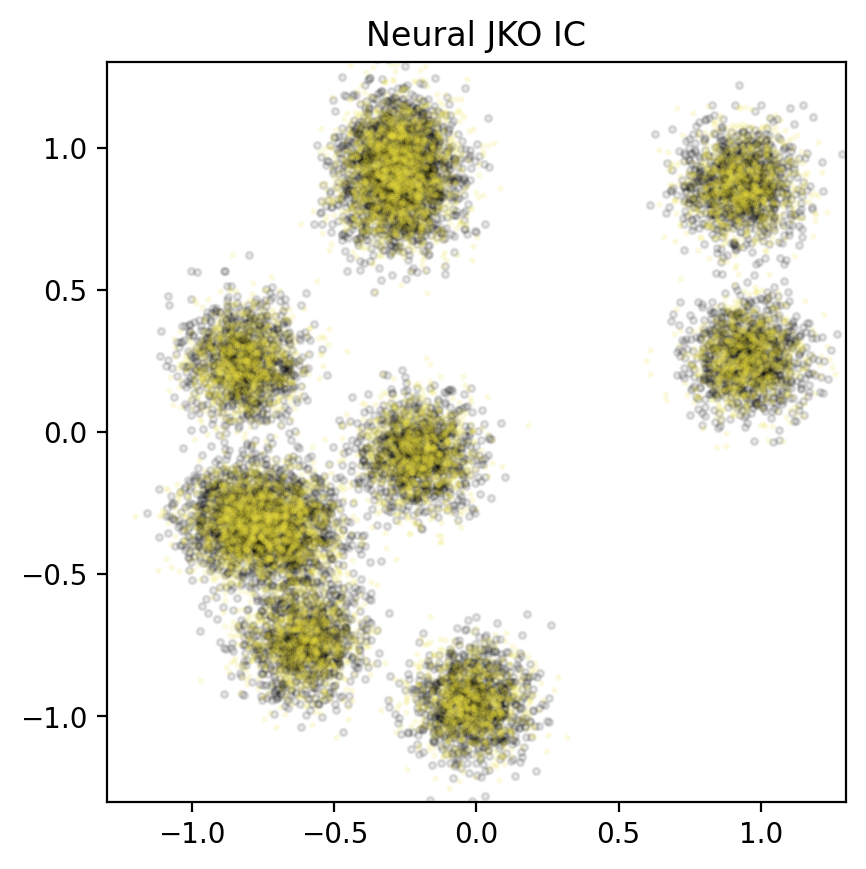}
    \caption*{GMM-$10$}
    \end{subfigure}    
    \begin{subfigure}{.21\textwidth}
        \includegraphics[width=0.9\textwidth]{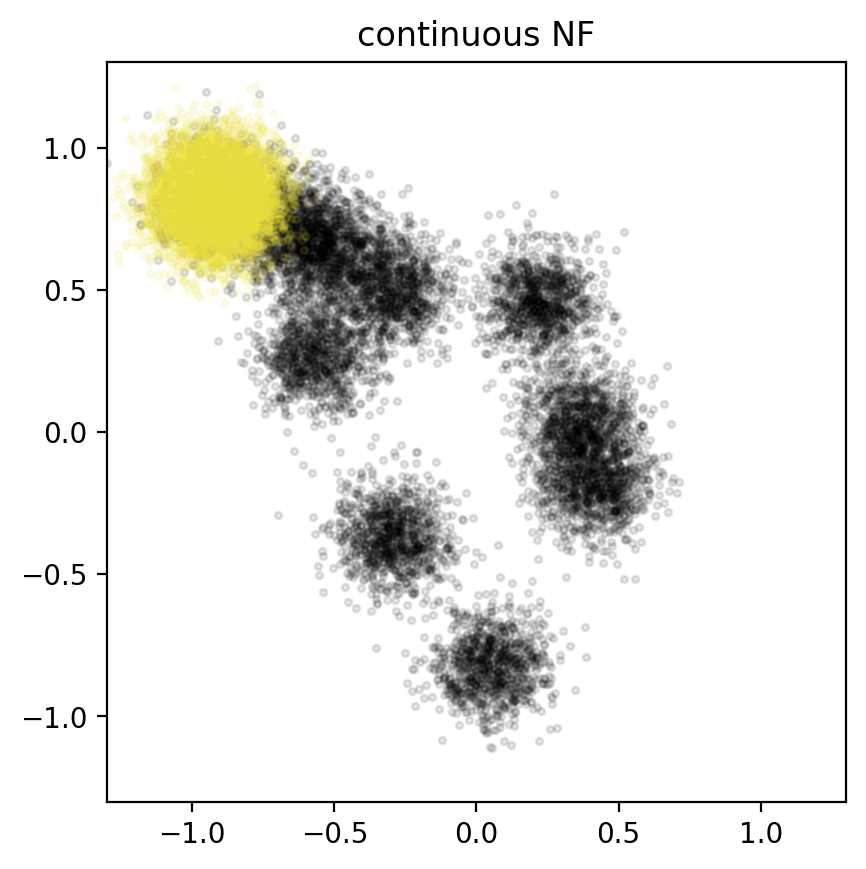}
        \includegraphics[width=0.9\textwidth]{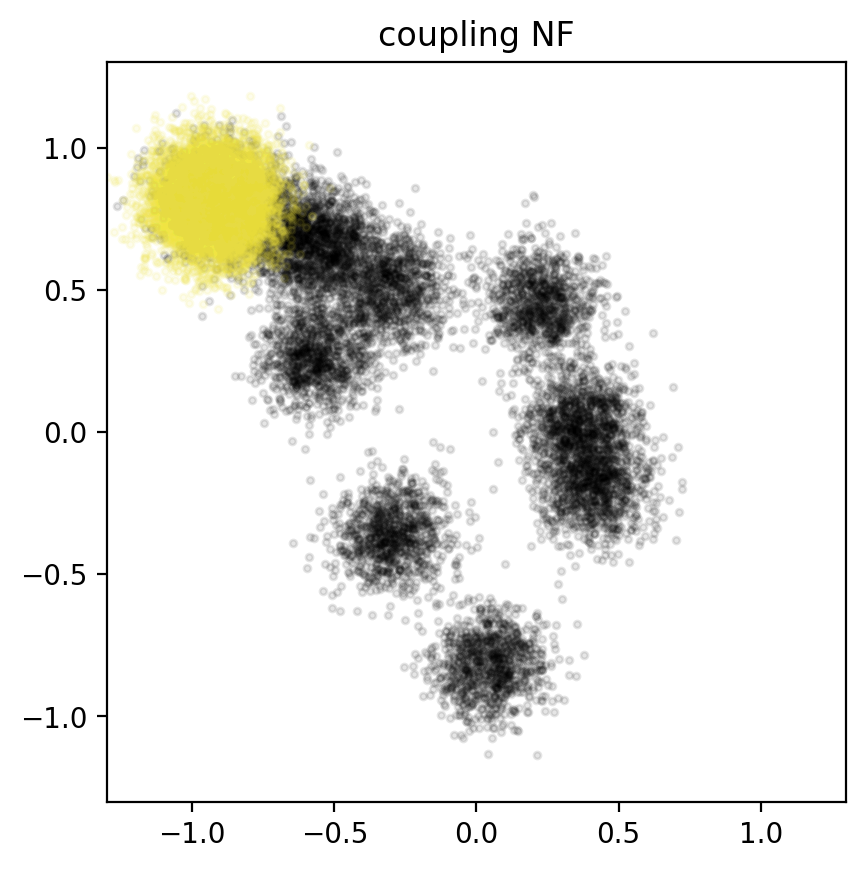}
        \includegraphics[width=0.9\textwidth]{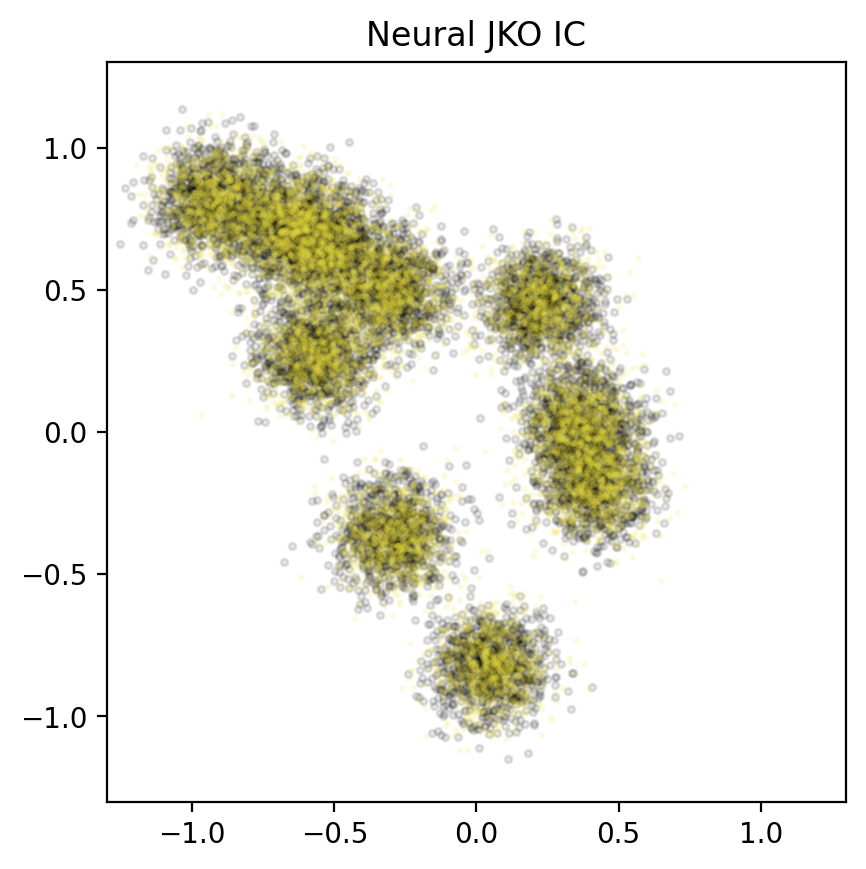}
    \caption*{GMM-$20$}
    \end{subfigure}    
\caption{
Marginalized sample generation for a single normalizing flow compared with neural JKO IC for different example distributions with \textit{ground truth} samples and \textcolor[RGB]{240, 228, 66}{\textit{generated}} samples for each associated method. We observe that the standard normalizing flow architectures always collapse to one or few modes while neural JKO IC recovers all modes correctly.
\label{fig:NF_comparison}
}
\end{figure}
\end{figure}

\end{document}